\documentclass[10pt,journal,compsoc]{IEEEtran}
\usepackage{array}
\usepackage{amsmath}
\usepackage[amsthm]{ntheorem}
\usepackage{amsfonts}
\usepackage{amssymb}
\usepackage{times}
\usepackage{epsfig}
\usepackage{graphicx}
\usepackage{subfigure}
\usepackage{url}
\usepackage{xcolor}
\usepackage{enumerate}
\usepackage{bm}
\usepackage{multirow}
\usepackage{diagbox}

\usepackage{colortbl}

\usepackage[colorlinks,linkcolor=red]{hyperref}
\newtheorem{theorem}{Theorem}
\newtheorem{prop}{Proposition}

\newcommand\QEDclosed{\ensuremath{\square}}

\usepackage{dsfont}

\newcommand{\vect}[1]{\boldsymbol{\mathbf{#1}}}
\newcommand{\mat}{\boldsymbol}
\def\x{\vect{x}}
\def\M{\mat{M}}

\def\u{\vect{\mu}}
\def\f{\vect{f}}

\renewcommand{\mathbf}{\boldsymbol}
\renewcommand{\vect}{\boldsymbol}

\newtheorem{definition}{Definition}

\ifCLASSOPTIONcompsoc
  \usepackage[nocompress]{cite}
\else
  \usepackage{cite}
\fi
\ifCLASSINFOpdf
\else
\fi
\begin{document}
\newcolumntype{L}[1]{>{\raggedright\arraybackslash}p{#1}}
\newcolumntype{C}[1]{>{\centering\arraybackslash}p{#1}}
\newcolumntype{R}[1]{>{\raggedleft\arraybackslash}p{#1}}
%
% paper title
% Titles are generally capitalized except for words such as a, an, and, as,
% at, but, by, for, in, nor, of, on, or, the, to and up, which are usually
% not capitalized unless they are the first or last word of the title.
% Linebreaks \\ can be used within to get better formatting as desired.
% Do not put math or special symbols in the title.
\title{Unsupervised Multi-Class Domain Adaptation: Theory, Algorithms, and Practice}
%
%
% author names and IEEE memberships
% note positions of commas and nonbreaking spaces ( ~ ) LaTeX will not break
% a structure at a ~ so this keeps an author's name from being broken across
% two lines.
% use \thanks{} to gain access to the first footnote area
% a separate \thanks must be used for each paragraph as LaTeX2e's \thanks
% was not built to handle multiple paragraphs
%
%
%\IEEEcompsocitemizethanks is a special \thanks that produces the bulleted
% lists the Computer Society journals use for "first footnote" author
% affiliations. Use \IEEEcompsocthanksitem which works much like \item
% for each affiliation group. When not in compsoc mode,
% \IEEEcompsocitemizethanks becomes like \thanks and
% \IEEEcompsocthanksitem becomes a line break with idention. This
% facilitates dual compilation, although admittedly the differences in the
% desired content of \author between the different types of papers makes a
% one-size-fits-all approach a daunting prospect. For instance, compsoc
% journal papers have the author affiliations above the "Manuscript
% received ..."  text while in non-compsoc journals this is reversed. Sigh.

\author{Yabin~Zhang$^*$,
	Bin~Deng$^*$,
	Hui~Tang,
	Lei~Zhang,
	and~Kui~Jia% <-this % stops a space
	\IEEEcompsocitemizethanks{\IEEEcompsocthanksitem Y. Zhang, B. Deng, H. Tang, and K. Jia are with the School of Electronic and Information Engineering, South China University of Technology, Guangzhou, China, and also with Pazhou Lab, Guangzhou, China. E-mails: zhang.yabin@mail.scut.edu.cn, bindeng.scut@gmail.com, eehuitang@mail.scut.edu.cn, kuijia@scut.edu.cn. $^*$These two authors contribute equally. Correspondence to: K. Jia \IEEEcompsocthanksitem L. Zhang is with the Department of Computing, The Hong Kong Polytechnic University, HongKong, and also with DAMO Academy, Alibaba Group. E-mail: cslzhang@comp.polyu.edu.hk \protect
		% note need leading \protect in front of \\ to get a newline within \thanks as
		% \\ is fragile and will error, could use \hfil\break instead.
}}
%\IEEEcompsocthanksitem J. Doe and J. Doe are with Anonymous University.}% <-this % stops an unwanted space
%\thanks{Manuscript received April 19, 2005; revised August 26, 2015.}}

% note the % following the last \IEEEmembership and also \thanks -
% these prevent an unwanted space from occurring between the last author name
% and the end of the author line. i.e., if you had this:
%
% \author{....lastname \thanks{...} \thanks{...} }
%                     ^------------^------------^----Do not want these spaces!
%
% a space would be appended to the last name and could cause every name on that
% line to be shifted left slightly. This is one of those "LaTeX things". For
% instance, "\textbf{A} \textbf{B}" will typeset as "A B" not "AB". To get
% "AB" then you have to do: "\textbf{A}\textbf{B}"
% \thanks is no different in this regard, so shield the last } of each \thanks
% that ends a line with a % and do not let a space in before the next \thanks.
% Spaces after \IEEEmembership other than the last one are OK (and needed) as
% you are supposed to have spaces between the names. For what it is worth,
% this is a minor point as most people would not even notice if the said evil
% space somehow managed to creep in.

% The paper headers
\markboth{Journal of \LaTeX\ Class Files,~Vol.~14, No.~8, February~2020}%
{Shell \MakeLowercase{\textit{et al.}}: Bare Demo of IEEEtran.cls for Computer Society Journals}
% The only time the second header will appear is for the odd numbered pages
% after the title page when using the twoside option.
%
% *** Note that you probably will NOT want to include the author's ***
% *** name in the headers of peer review papers.                   ***
% You can use \ifCLASSOPTIONpeerreview for conditional compilation here if
% you desire.

% The publisher's ID mark at the bottom of the page is less important with
% Computer Society journal papers as those publications place the marks
% outside of the main text columns and, therefore, unlike regular IEEE
% journals, the available text space is not reduced by their presence.
% If you want to put a publisher's ID mark on the page you can do it like
% this:
%\IEEEpubid{0000--0000/00\$00.00~\copyright~2015 IEEE}
% or like this to get the Computer Society new two part style.
%\IEEEpubid{\makebox[\columnwidth]{\hfill 0000--0000/00/\$00.00~\copyright~2015 IEEE}%
%\hspace{\columnsep}\makebox[\columnwidth]{Published by the IEEE Computer Society\hfill}}
% Remember, if you use this you must call \IEEEpubidadjcol in the second
% column for its text to clear the IEEEpubid mark (Computer Society jorunal
% papers don't need this extra clearance.)

% use for special paper notices
%\IEEEspecialpapernotice{(Invited Paper)}

% for Computer Society papers, we must declare the abstract and index terms
% PRIOR to the title within the \IEEEtitleabstractindextext IEEEtran
% command as these need to go into the title area created by \maketitle.
% As a general rule, do not put math, special symbols or citations
% in the abstract or keywords.
\IEEEtitleabstractindextext{%
\begin{abstract}
	In this paper, we study the formalism of unsupervised multi-class domain adaptation (multi-class UDA), which underlies a few recent algorithms whose learning objectives are only motivated empirically. Multi-Class Scoring Disagreement (MCSD) divergence is presented by aggregating the absolute margin violations in multi-class classification, and this proposed MCSD is able to fully characterize the relations between any pair of multi-class scoring hypotheses. By using MCSD as a measure of domain distance, we develop a new domain adaptation bound for multi-class UDA; its data-dependent, probably approximately correct bound is also developed that naturally suggests adversarial learning objectives to align conditional feature distributions across source and target domains. Consequently, an algorithmic framework of Multi-class Domain-adversarial learning Networks (McDalNets) is developed, and its different instantiations via surrogate learning objectives either coincide with or resemble a few recently popular methods, thus (partially) underscoring their practical effectiveness. Based on our identical theory for multi-class UDA, we also introduce a new algorithm of Domain-Symmetric Networks (SymmNets), which is featured by a novel adversarial strategy of domain confusion and discrimination. SymmNets affords simple extensions that work equally well under the problem settings of either closed set, partial, or open set UDA. We conduct careful empirical studies to compare different algorithms of McDalNets and our newly introduced SymmNets. Experiments verify our theoretical analysis and show the efficacy of our proposed SymmNets. In addition, we have made our implementation code publicly available.

\end{abstract}

% Note that keywords are not normally used for peerreview papers.
}

% make the title area
\maketitle

% To allow for easy dual compilation without having to reenter the
% abstract/keywords data, the \IEEEtitleabstractindextext text will
% not be used in maketitle, but will appear (i.e., to be "transported")
% here as \IEEEdisplaynontitleabstractindextext when the compsoc
% or transmag modes are not selected <OR> if conference mode is selected
% - because all conference papers position the abstract like regular
% papers do.
\IEEEdisplaynontitleabstractindextext
% \IEEEdisplaynontitleabstractindextext has no effect when using
% compsoc or transmag under a non-conference mode.

% For peer review papers, you can put extra information on the cover
% page as needed:
% \ifCLASSOPTIONpeerreview
% \begin{center} \bfseries EDICS Category: 3-BBND \end{center}
% \fi
%
% For peerreview papers, this IEEEtran command inserts a page break and
% creates the second title. It will be ignored for other modes.
\IEEEpeerreviewmaketitle

\IEEEraisesectionheading{\section{Introduction}\label{Sec:intro}}

\IEEEPARstart{S}{tandard} machine learning assumes that training and test data are drawn from the same underlying distribution. As such, uniform convergence bounds guarantee the generalization of models learned on training data for the use of testing \cite{MLFoudatation2014}.  Although standard machine learning has achieved great success in various tasks \cite{alexnet, girshick2014rich,jia2019orthogonal}, even with few training data \cite{vinyals2016matching,zhang2019part} or training data of multiple modalities \cite{jia2019deep}, in many practical scenarios, one may encounter situations where annotated training data can only be collected easily from one or several distributions that are related to the testing distribution. In other words, the target data of interest follow a distribution differing from the training source data.
%In many practical scenarios, however, one may encounter the situation that annotated training data can only be collected easily from one or several distributions that are related to the testing one; in order words, the target data of interest follow a distribution differing from the training source ones.
A typical example in deep learning-based image analysis is that one may annotate as many synthetic images as possible, but often fails to annotate even a single real image. Thus it is expected to adapt the models learned from synthetic images for testing on real images. This problem setting falls in the realm of transfer learning or domain adaptation \cite{transfer_survey}. In this work, we focus particularly on unsupervised domain adaptation (UDA), in which target data are completely unlabeled.

In the literature, theoretical studies on domain adaptation characterize the conditions under which classifiers trained on labeled source data can be adapted for use on the target domain \cite{ben2007analysis,ben2010theory,mansour2009domain,courty2016optimal}. For example, Ben-David \emph{et al.} \cite{ben2010theory} propose the notion of distribution divergence induced by the hypothesis space of binary classifiers, based on which a bound of the expected error on the target domain is thus developed. Mansour \emph{et al.} \cite{mansour2009domain} extend the zero-one loss used in \cite{ben2010theory} to arbitrary loss functions of binary classification. These theoretical results motivate many of existing UDA algorithms, including the recently popular ones based on the domain-adversarial training of deep networks \cite{dann,adda,cada,mcd,symnets}. A common motivation of these algorithms is to design adversarial objectives concerned with minimax optimization, in order to reduce the hypothesis-induced domain divergence via the learning of domain-invariant feature representations.
While theoretical adaptation conditions are strictly derived under the setting of binary classification with analysis-amenable loss functions, practical algorithms easier to be optimized are often expected to be applied to the cases of multiple classes. In other words, the learning objectives in many of the recent algorithms are only inspired by, rather than strictly derived from the domain adaptation bounds in \cite{ben2010theory,mansour2009domain}. This gap between theories and algorithms is recently studied in \cite{mdd}, where the notion of margin disparity discrepancy (MDD) induced by pairs of multi-class scoring hypotheses is introduced to measure the divergence between domain distributions. This thus extends theories in \cite{ben2010theory,mansour2009domain} and connects with the multi-class setting of practical algorithms.

The MDD introduced in \cite{mdd} is constructed using a scalar-valued function of \emph{relative margin}. It characterizes a disagreement between any pair of multi-class scoring hypotheses. This disagreement, however, does not take relationships among all of the multiple classes into account. As a result, the theory developed in \cite{mdd} cannot properly explain the effectiveness of a series of recent UDA algorithms \cite{mcd,symnets,adr,swd,Cicek_2019_ICCV}. In this work, we are motivated to follow \cite{mdd} and develop a theory for unsupervised multi-class domain adaptation (multi-class UDA) that connects more closely with recent algorithms. Inspired by the MDD of \cite{mdd} and the multi-class classification framework of Dogan \emph{et al.} \cite{dogan2016unified}, which aggregates violations of class-wise \emph{absolute margins} as a single loss, we technically propose a notion of matrix-formed, \emph{Multi-Class Scoring Disagreement (MCSD)}, which takes a full account of the element-wise disagreements between any pair of multi-class scoring hypotheses.
MCSDs defined over domain distributions induce a novel \emph{MCSD divergence}, measuring distribution distance between the source and target domains. Based on MCSD divergence, we develop a new adaptation bound for multi-class UDA. A data-dependent, probably approximately correct (PAC) bound is also developed using the notion of Rademacher complexity. We connect our results with existing theories of either binary \cite{ben2010theory} or multi-class UDA \cite{mdd} by introducing their absolute margin-based equivalent or variant of domain divergence, as well as the corresponding domain adaptation bounds. We show the advantages of MCSD divergence over these absolute margin-based equivalent/variant (and also their corresponding ones in \cite{ben2010theory} and \cite{mdd}).

The bounds derived in our theory of multi-class UDA based on either MCSD divergence or
the absolute margin-based versions of \cite{ben2010theory} and \cite{mdd}
%its degenerate versions
naturally suggest adversarial objectives of minimax optimization, which promote the learning of feature distributions invariant across source and target domains. We term such an algorithmic framework as \emph{Multi-class Domain-adversarial learning Networks (McDalNets)}, as illustrated in Figure \ref{FigMcDalNets}. While it is difficult to optimize the objectives of McDalNets directly, we show that a few optimization-friendly surrogate objectives instantiate the recently popular methods \cite{mcd,mdd}, thus (partially) explaining the underlying mechanisms of their effectiveness. In addition to McDalNets, we introduce a new algorithm of \emph{Domain-Symmetric Networks (SymmNets)}, which is motivated from our same theory of multi-class UDA. Figure \ref{FigSymmNet} is an illustration of this. The proposed SymmNets is featured by a domain confusion and discrimination strategy that ideally achieves the same theoretically derived learning objective.

While most of the theories and algorithms presented in the paper are concerned with \emph{closed set UDA}, where the two domains share the same label space, one might also be interested in other variant settings, such as \emph{partial} \cite{san,mada,pada,iwanpda,transfer_example_partial} or \emph{open set} \cite{open_set_math,open_set_bp} UDA. In this work, we present simple extensions of SymmNets that are able to achieve partial or open set UDA as well. We conduct careful ablation studies to compare different algorithms of McDalNets,
including those based on the absolute margin-based versions of \cite{ben2010theory} and \cite{mdd}, as well as our newly introduced SymmNets. As shown in Table \ref{Tab:different_implementation}, experiments on six commonly used benchmarks show that algorithms of McDalNets based on MCSD divergence consistently improve over those based on the absolute margin-based versions of \cite{ben2010theory} and \cite{mdd},
%its degenerate versions, and also our newly introduced SymmNets. As shown in Table \ref{Tab:different_implementation}, experiments on five commonly used benchmarks show that algorithms of McDalNets consistently improve over its degenerate versions,
certifying the usefulness of fully characterizing disagreements between pairs of scoring hypotheses in multi-class UDA. Experiments under the settings of the closed set, partial, and open set UDA also empirically verify the effectiveness of our proposed SymmNets.

\subsection{Relations with Existing Works}
	
\subsubsection{Domain Adaptation Theories}

In the literature, these exist theoretical domain adaptation results concerning mostly with the classification problem and also with regression \cite{cortes2014domain,cortes2015adaptation,mansour2009domain}. For classification, these results consider either a setting where target data are partially labeled \cite{mohri2012new,zhang2012generalization}, or the standard unsupervised setting from the perspectives of optimal transportation \cite{courty2016optimal,courty2017joint} or hypothesis-induced domain divergence \cite{ben2007analysis,ben2010theory,mansour2009domain,kuroki2019unsupervised,mdd}. We focus on the latter line of theories, which are closely related to the one we contribute.
	
The seminal domain adaptation theories \cite{ben2007analysis,ben2010theory,mansour2009domain} bound the expected target error for binary classification with terms characterizing the expected source error, the domain distance under certain metrics of distribution divergence, and constant ones that depend on the capacity of the hypothesis space; the term of domain distance differentiates these theoretical bounds. For example, Ben-David \emph{et al.} \cite{ben2007analysis,ben2010theory} propose for binary classification the zero-one loss-based ${\mathcal{H}\Delta\mathcal{H}}$-divergence by characterizing the disagreement between any pair of labeling hypotheses; Mansour \emph{et al.} \cite{mansour2009domain} introduce a notion of discrepancy distance by extending the zero-one loss used in \cite{ben2007analysis} to general loss functions of binary classification; by fixing one hypothesis of \cite{mansour2009domain} to the ideal source minimizer, Kuroki \emph{et al.} \cite{kuroki2019unsupervised} propose a more tractable source-guided discrepancy. Although many of the recent algorithms \cite{dann, adda, mcd, symnets} are motivated from seminal theories \cite{ben2007analysis,ben2010theory}, the gap between theories of binary classification and practical algorithms of multi-class classification remains.
To reduce this gap, Zhang \emph{et al.} \cite{mdd} make a first attempt to extend the theories of \cite{mansour2009domain,ben2010theory} to the case of multiple classes by introducing a novel notion of margin disparity discrepancy (MDD); MDD is a measure of domain distance built upon a scalar-valued function of margin disparity (MD), which can to some extent characterize the difference of multi-class scoring hypotheses.
	
While both our MCSD and those of \cite{ben2010theory,mansour2009domain,mdd} are based on the characterization of disagreements between any pair of labeling/scoring hypotheses, our MCSD is capable of characterizing them at a finer level, especially in the multi-class setting (cf. Figure \ref{Fig:f-f-disp}). Technically, our MCSD characterizes element-wise disagreements of multi-class scoring hypotheses by aggregating violations of class-wise absolute margins. By contrast, the zero-one loss-based counterpart of \cite{ben2010theory} only characterizes the labeling disagreement, and the margin disparity (MD) of \cite{mdd} improves over \cite{ben2010theory} with a scoring disagreement that is based on a scalar-valued, relative margin.
Consequently, the domain divergence induced by our MCSD can better explain the effectiveness of a series of recent UDA algorithms \cite{mcd,symnets,adr,swd,Cicek_2019_ICCV}, whose designs take the relations of scores of all the multiple classes into account.
%\textcolor{red}{Furthermore, to connect our MCSD divergence with these of \cite{ben2010theory,mdd}, we present the absolute margin-based equivalent of ${\mathcal{H}\Delta\mathcal{H}}$-divergence \cite{ben2010theory} and variant of MDD \cite{mdd}, and clarify the advantages of our MCSD in characterizing finer details of the scoring disagreement.}
	
\subsubsection{Algorithms of Multi-Class Domain Adaptation}

Existing algorithms of multi-class UDA are mainly motivated by learning domain-invariant feature representations \cite{dan,jan,dann,adda,mcd,symnets,adr,swd,Cicek_2019_ICCV,zhu2017unpaired,rozantsev2018beyond,dwt,dada}, or by minimizing the domain discrepancy in the image space via image generation \cite{pixel_level,isola2017image}. We briefly review the former line of algorithms, focusing on those based on the strategy of adversarial training.
	
Motivated to minimize the domain divergence measured by ${\mathcal{H}\Delta\mathcal{H}}$-divergence of \cite{ben2010theory}, Ganin \emph{et al.} \cite{dann} introduce the first strategy of the domain-adversarial training of neural networks (DANN), where a binary classifier is adopted as the domain discriminator, and the domain distance is minimized by learning features of the two domains in a manner adversarial to the domain discriminator. Tzeng \emph{et al.} \cite{adda} summarize three implementation manners of adversarial objective, including minimax \cite{dann}, confusion \cite{domain_confusion}, and GAN \cite{gan}. The domain discriminator of the binary classifier enables the learning of the alignment of marginal feature distributions across domains, but it is ineffective for the alignment of conditional feature distributions, which is necessary for practical UDA problems in a multi-class setting. Recent methods \cite{mcd,symnets,mdd,adr,swd,Cicek_2019_ICCV} strive to overcome this limitation by playing adversarial games between two classifiers.
%Although minimizing the distance of marginal feature distributions across domains achieves certain success, taking the conditional distribution into account brings clear advantage. Adopting the binary domain discriminator again, Long \emph{et al.} \cite{cada} minimize the multilinear conditioned feature distributions across domains. Instead of the binary domain discriminator, two task classifiers are recently used in \cite{mcd,mdd,adr} to align the conditional feature distributions.
More specifically, Saito \emph{et al.} \cite{mcd} adopt the maximum $L_1$ distance of output probabilities of two symmetric classifiers as a surrogate domain discrepancy; Lee \emph{et al.} \cite{swd} replace the $L_1$ distance in \cite{mcd} with the Wasserstein distance \cite{villani2008optimal}, taking advantage of its geometrical characterization; in \cite{mdd}, two classifiers are used asymmetrically to estimate conditional feature distributions with margin loss; in \cite{adr}, two task classifiers are introduced implicitly by applying two random dropouts to the same task classifier; a classifier concatenated by two task classifiers is adopted to implement the adversarial training objective in \cite{symnets,Cicek_2019_ICCV}.
	
Motivated by the domain adaptation bounds to be presented in Section \ref{Sec:theory_motivation}, we propose an algorithmic framework of McDalNets, whose optimization-friendly surrogate objectives instantiate these recently popular methods \cite{dann,mdd,mcd} (cf. Section \ref{SecMcDalNetAlgms}), thus (partially) explaining the underlying mechanisms of their effectiveness. We also introduce a new algorithm of SymmNets, whose learning objective aligns with our developed theoretical bound as well (cf. Section \ref{SecSymmNets}).

\subsubsection{Variants of Problem Settings}
The theories and algorithms discussed so far apply to the problem setting of closed set UDA, where a shared label space across domains is assumed. There exist other variant settings, e.g., partial \cite{san} or open set \cite{open_set_math} UDA. We discuss these settings and the corresponding methods as follows.
	
The setting of partial UDA assumes that classes of the target domain constitute an unknown subset of those of the source domain. To address the challenge brought by partial class coverage, a typical strategy is to weight source instances using the collective prediction evidence of target instances \cite{san,pada,iwanpda,transfer_example_partial}. Simply extending our SymmNets with a weighting scheme gives excellent results.
	
The setting of open set UDA assumes that both the source and target domains contain certain classes that are exclusive to each other, where for simplicity all the unshared classes in each domain are aggregated as a single (super-) unknown class. A key issue to extend methods of closed set UDA for the use in the open set setting is to design appropriate criteria that reject the target instances of unshared classes. To this end, Busto \emph{et al.} \cite{open_set_math} adopt a predefined distance threshold, and Saito \emph{et al.} \cite{open_set_bp} learn rejection automatically via the adversarial training. Our algorithm of SymmNets is flexible enough to be applied to open set UDA simply by adding an additional output neuron to the task classifier that is responsible for the aggregated super-class, while keeping other algorithmic ingredients fixed.
	
\subsection{Contributions}

Many recent algorithms for multi-class UDA \cite{mcd,adr,swd,Cicek_2019_ICCV}, including our preliminary work of SymmNets \cite{symnets}, rely on an adversarial strategy that learns to align conditional feature distributions across domains via a full account of the relationships among the hypotheses of classifiers. While these algorithms are inspired by classical domain adaptation theories \cite{ben2010theory,mansour2009domain,ben2007analysis}, their learning objectives are largely designed empirically; as such, the connections between theories and algorithms remain loose. The present paper aims to improve over the recent theory of multi-class UDA \cite{mdd}, and to connect with these algorithms more closely by formalizing a new theory of multi-class UDA, which underlies these algorithms with a framework that also inspires new algorithms. We summarize our technical contributions as follows.
\begin{itemize}
	\item We propose to aggregate violations of absolute margin functions to define a notion of matrix-formed, \emph{Multi-Class Scoring Disagreement (MCSD)}, which enables a full characterization of the relations between any pair of scoring hypotheses. Based on the induced \emph{MCSD divergence} as a measure of domain distance, we develop a new adaptation bound for multi-class UDA; a data-dependent PAC bound is also developed using the notion of Rademacher complexity. We connect our results with existing theories of either binary or multi-class UDA, by introducing
	their absolute margin-based equivalent or variant of domain divergence, and the corresponding adaptation bounds. 
	\item Our developed theories naturally suggest adversarial objectives to learn aligned conditional feature distributions across domains; we term such an algorithmic framework based on deep networks as \emph{Multi-class Domain-adversarial learning Networks (McDalNets)}. We show that different instantiations of McDalNets via surrogate learning objectives either coincide with or resemble a few recently popular methods, thus (partially) underscoring their practical efficacy. We also introduce a new algorithm of \emph{Domain-Symmetric Networks (SymmNets-V2)} based on our same theory of multi-class UDA, which improves over \emph{SymmNets-V1} proposed in our preliminary work.
	\item While theories and algorithms presented in the paper are mostly concerned with the problem setting of closed set UDA, we also present simple extensions of SymmNets that work equally well under the settings of partial or open set UDA. We conduct careful ablation studies to compare different algorithms of McDalNets,
	including those based on the absolute margin-based versions of \cite{ben2010theory} and \cite{mdd},
	%its degenerate versions,
	as well as our newly introduced SymmNets. Experiments on commonly used benchmarks show the advantages of McDalNets and SymmNets, certifying the effectiveness of fully characterizing disagreements between pairs of scoring hypotheses in multi-class UDA. We have made our code available at \url{https://github.com/YBZh/MultiClassDA}.
\end{itemize}

\color {black}
\section{A Theory of Unsupervised Multi-Class Domain Adaptation}
\label{Sec:theory_motivation}

We present in this section a theory of unsupervised multi-class domain adaptation (multi-class UDA). Our theoretical derivations follow \cite{ben2010theory,mansour2009domain,mdd}, but with a key novelty of measuring the distance between domain distributions using a divergence that fully characterizes the relations between different hypotheses of multi-class classification. We also present variants of the proposed divergence to connect with theoretical results developed in the literature. We start with a learning setup of multi-class UDA. Table \ref{tab-notations} gives a summary of our used math notations. All proofs are given in the appendices.

\begin{table}[t]
  \centering
  \caption{A summary of the used math notations.}\label{tab-notations}
  %\begin{tabular}{>{\columncolor[rgb]{0.7,0.8,0.9}}l>{\columncolor[rgb]{0.7,0.8,0.9}}p{5cm}}
  \begin{tabular}{lp{5cm}}
  \hline
   Notations & Meaning \\ \hline
   $K$ & Number of classes \\ %\hline
   $\mathbb{R}$ & The space of real numbers \\ %\hline
   $\mathbb{R}^K$ & The space of $K$-dimensional real vectors \\ %\hline
   %$\mathbb{R}_+$ & The set of non-negative real numbers \\ %\hline
   $\mathcal{X}$ & Instance space \\ %\hline
   $\mathcal{Y}$ & Label space \\ %\hline
   $\mathcal{H}, \mathcal{F}$ & Hypothesis spaces of labeling or scoring functions \\ %\hline
   \textcolor{black}{$\x,\x^s_i, \x^t_i$} & \textcolor{black}{A general instance, an $i^{th}$ source instance, or an $i^{th}$ target instance in the space $\mathcal{X}$}  \\ %\hline
   %$y$ & A label index in $\mathcal{Y}$ \\ %\hline
   \textcolor{black}{$D$} & \textcolor{black}{A distribution over a domain (e.g., $\mathcal{X}\times \mathcal{Y}$)} \\ %\hline
   \textcolor{black}{$D_x$} & \textcolor{black}{A marginal distribution over $\mathcal{X}$ when $D$ is over $\mathcal{X}\times \mathcal{Y}$} \\
   \textcolor{black}{$P, Q$} & \textcolor{black}{Source or target distributions over $\mathcal{X}\times \mathcal{Y}$} \\ %\hline
   \textcolor{black}{$P_x, Q_x$} & \textcolor{black}{Source or target marginal distributions over $\mathcal{X}$} \\ %\hline
   \textcolor{black}{$h: \mathcal{X}\rightarrow \mathcal{Y}$} & \textcolor{black}{A function in a hypothesis space $\mathcal{H}$} \\ %\hline
   \textcolor{black}{$\f, \f', \f'': \mathcal{X}\rightarrow \mathbb{R}^K$} & \textcolor{black}{Functions in a scoring space $\mathcal{F}$} \\ %\hline
   \textcolor{black}{$f_k, f'_k, f''_k: \mathcal{X}\rightarrow \mathbb{R}$} & \textcolor{black}{The $k^{th}$ components of $\f, \f'$, or $\f''$} \\ %\hline
   \textcolor{black}{$\u: \mathbb{R}^K \times \mathcal{Y} \rightarrow \mathbb{R}^K$} & \textcolor{black}{Absolute margin function of Definition 1} \\ %\hline
   \textcolor{black}{$\mu_k: \mathbb{R}^K \times \mathcal{Y} \rightarrow \mathbb{R}$} & \textcolor{black}{The $k^{th}$ component of $\u$} \\ %\hline
   \textcolor{black}{$\Phi_\rho: \mathbb{R}\rightarrow [0,1]$} & \textcolor{black}{Ramp loss (5) with margin $\rho$} \\
   \textcolor{black}{$\mathbb{E}$} & \textcolor{black}{Expectation of a random variable} \\ %\hline
   \textcolor{black}{$\mathbb{I}$ [Boolean expression]} & \textcolor{black}{Indicator function, which returns 1 when the expression is true, and 0 otherwise } \\ %Indicator function (equals 1 when the expression is true, and 0 otherwise) 
   \hline
  \end{tabular}
\end{table}

\subsection{Learning Setup}

%For a standard learning setting of multi-class classification, learners receive samples from a distribution $D$ over $\mathcal{X\times Y}$, where $\mathcal{X}$ is the instance space and $\mathcal{Y} = \{1,...,K\}$ is the label space. We also write $D_x$ as the corresponding marginal distribution over $\mathcal{X}$.
Multi-class UDA assumes two different but related distributions over $\mathcal{X\times Y}$, namely the source one $P$ and target one $Q$. Learners receive $n_s$ labeled examples $\{(\x_i^s,y_i^s)\}_{i=1}^{n_s}$ drawn i.i.d. from $P$ and $n_t$ unlabeled examples $\{\x_j^t\}_{j=1}^{n_t}$ drawn i.i.d. from $Q_x$. The goal of multi-class UDA is to identify a labeling hypothesis $h: \mathcal{X}\rightarrow \mathcal{Y}$ from a space $\mathcal{H}$ such that the following \emph{expected error} over the target distribution is minimized
\begin{equation}\label{EqnExpectTargetErr}
\mathcal{E}_Q(h) := \mathbb{E}_{(\x,y)\sim Q} L(h(\x), y) ,
\end{equation}
where $L$ is a properly defined loss function. For ease of theoretical analysis, Ben-David \emph{et al.} \cite{ben2007analysis,ben2010theory} assume $L$ as a zero-one loss of the form $\mathbb{I}[h(\x)\neq y]$, where $\mathbb{I}$ is the indicator function, which is extended in \cite{mansour2009domain} as general loss functions of binary classification. Domain adaptation theories \cite{ben2010theory,mansour2009domain,mdd} typically bound the expected target error (\ref{EqnExpectTargetErr}) using derived meaningful terms.

Consider a space $\mathcal{F}$ that contains the scoring function $\f: \mathcal{X} \rightarrow \mathbb{R}^{|\mathcal{Y}|} = \mathbb{R}^K$, which induces a labeling function $h_{\f}(\x) = \arg \max_{k\in \mathcal{Y}} f_k(\x)$, where $f_k$ denotes the $k^{th}$ component of the vector-valued function $\f$. Adding the same function $g: \mathcal{X}\rightarrow \mathbb{R}$ to all components $f_k$ of $\f$ does not change the classification decision, since $\arg\max_{k\in \mathcal{Y}}f_k(\x) = \arg\max_{k\in \mathcal{Y}}(f_k(\x) + g(\x))$; this could be problematic for obtaining unique solutions of scoring functions. Similar to \cite{dogan2016unified}, we fix this issue by enforcing the sum-to-zero constraint $\sum_{k=1}^K f_k(\x) = 0$ to the scoring functions.

\subsection{Domain Distribution Divergence and Adaptation Bounds based on Multi-Class Scoring Disagreement}\label{subsec:multi-class-loss}

Unsupervised domain adaptation is made possible by assuming the closeness between the distributions $P$ and $Q$; otherwise classifiers learned from the labeled source data would be less relevant for the classification of target data. The measures of distribution distances thus become crucial factors in developing either UDA theories or the corresponding algorithms.

{\color{black}
\subsubsection{Existing Measures of Domain Divergence}

In the seminal work \cite{ben2010theory}, a key innovation is the introduction of a distribution distance induced by a hypothesis space $\mathcal{H}^{\{0, 1\}}$ of binary classification
\begin{equation}\label{equ:HDeltaHDivergence}
d_{0\text{\rm -}1} (P_x, Q_x) := \sup\limits_{h,h' \in \mathcal{H}^{\{0, 1\}}} \left| \mathbb{E}_{Q_x} \mathbb{I}[h\neq h'] - \mathbb{E}_{P_x} \mathbb{I}[h\neq h'] \right| ,
\end{equation}
\textcolor{black}{where $\mathbb{E}_{Q_x} \mathbb{I}[h\neq h'] = \mathbb{E}_{\x\sim Q_x} \mathbb{I}[h(\x)\neq h'(\x)]$, i.e., the zero-one loss-based expectation of the hypothesis disagreement, which we term as \emph{hypothesis disagreement (HD)} to facilitate the subsequent discussion, and similarly for $\mathbb{E}_{P_x} \mathbb{I}[h\neq h']$; the disagreement between $h$ and $h'$ in fact specifies a measurable subset $\{\x \in \mathcal{X} | h(\x)\neq h'(\x) \}$, and the distribution distance (termed $\mathcal{H}\Delta\mathcal{H}$-divergence in \cite{ben2010theory}) between $P_x$ and $Q_x$ is measured on the subsets by taking the supremum over all pairs of $h, h' \in \mathcal{H}^{\{0, 1\}}$.}
% is termed as the (expected) \emph{0-1} distance between $h$ and $h'$ with respect to the distribution $D$, and the distribution distance (termed $\mathcal{H}\Delta\mathcal{H}$-divergence in \cite{ben2010theory}) between $P_x$ and $Q_x$ is measured on the \emph{0-1} distance by taking the supremum over all pairs of $h, h' \in \mathcal{H}^{\{0, 1\}}$.
Compared with the simple $\ell_1$ distribution divergence, the distance (\ref{equ:HDeltaHDivergence}) is more relevant to the problem of domain adaptation and can be estimated from finite samples for an $\mathcal{H}^{\{0, 1\}}$ of fixed VC dimension \cite{ben2010theory}. Based on the same idea of characterizing the hypothesis disagreement, Mansour \emph{et al.} \cite{mansour2009domain} extend the zero-one loss-based distance (\ref{equ:HDeltaHDivergence}) to general loss functions $L$, giving rise to the distance (termed \emph{discrepancy distance} in \cite{mansour2009domain})
\begin{equation}\label{equ:DiscDivergence}
d_L (P_x, Q_x) := \sup\limits_{h,h' \in \mathcal{H}} \left| \mathbb{E}_{Q_x} L(h, h') - \mathbb{E}_{P_x} L(h, h') \right| .
\end{equation}
Note that (\ref{equ:DiscDivergence}) is symmetric and satisfies triangle inequality, but it does not strictly define a distance since it is possible that $d_L (P_x, Q_x) = 0$ for $P_x \neq Q_x$.

{\color{black}
In spite of being more general, the distance (\ref{equ:DiscDivergence}) applies only to UDA problems of binary classification. To develop multi-class UDA, disagreement of multi-class hypotheses should be taken into account. The key issue here is to extend binary loss functions $L$, especially \emph{margin-based ones}, to the case of multiple classes \cite{VapnikBook}. In literature, there exists no a canonical formulation of multi-class classification; various formulation variants have been proposed depending on different notions of multi-class margins and margin-based losses \cite{koltchinskii2002empirical,lee2004multicategory,liu2011reinforced,szedmak2005learning}, where margins are usually defined either by comparing components $\{f_k\}_{k=1}^K$ of a $K$-class scoring function $\f$ (i.e., \emph{relative margins}), or directly on the components $\{f_k\}_{k=1}^K$ themselves (i.e., \emph{absolute margins}). Based on this idea, Zhang \emph{et al.} \cite{mdd} first investigate multi-class UDA by measuring the disagreement of multi-class hypotheses with a relative margin function \cite{koltchinskii2002empirical}. Given a fixed $\f$, a \emph{margin disparity discrepancy (MDD)} is proposed in \cite{mdd} that defines the distribution divergence as
\begin{equation}\label{equ:MarginDisparityDiscrepancy}
d_{MD}^{(\rho)} (P_x, Q_x) := \sup\limits_{\f' \in \mathcal{F}} [\mathbb{E}_{Q_x} \Phi_\rho(\rho_{\f'}(\cdot, h_{\f})) - \mathbb{E}_{P_x} \Phi_\rho(\rho_{\f'}(\cdot, h_{\f}))],
\end{equation}
where $\mathbb{E}_{Q_x}\Phi_\rho(\rho_{\f'}(\cdot, h_{\f})) = \mathbb{E}_{\x\sim Q_x}\Phi_\rho(\rho_{\f'}(\x, h_{\f}{(\x)}))$ is termed as the \emph{margin disparity (MD)} in \cite{mdd}, which is induced by $\f$ and $\f'$ w.r.t. the distribution $Q_x$, the MD $\mathbb{E}_{P_x}\Phi_\rho(\rho_{\f'}(\cdot, h_{\f}))$ is similarly defined, $\Phi_\rho$ is a ramp loss defined as
	\begin{equation} \label{equ-ramp_loss}
	\Phi_\rho(x) :=
	\begin{cases}
	0, & \rho \leq x \\
	1 - x/{\rho}, & 0 < x < \rho \\
	1, & x \leq 0
	\end{cases},
	\end{equation}
and $\rho_{\f}(\x, y)  = \frac{1}{2} \left(f_y(\x) - \max_{y' \neq y} f_{y'}(\x) \right)$ is the relative margin function. The MDD (\ref{equ:MarginDisparityDiscrepancy}) improves over (\ref{equ:HDeltaHDivergence}) and (\ref{equ:DiscDivergence}) by using the hypothesis $h_{\f}$ of a given $\f$ to induce a relative margin of the scoring function $\f'$, thus successfully measuring a disagreement between the multi-class $\f$ and $\f'$. We note that the induced relative margin depends only on the component of $\f$ that has the maximum value (i.e., the hypothesis $h_{\f}$); it does not fully characterize the disagreements between $\f$ and $\f'$. It is in fact this deficiency of MDD that motivates the present paper. By proposing a new divergence that can fully characterize the disagreements between pairs of multi-class scoring functions, we expect to develop the corresponding theory of multi-class UDA that helps underscore the effectiveness of a series of recent UDA algorithms \cite{mcd,symnets,adr,swd,Cicek_2019_ICCV}.
}

\subsubsection{The Proposed Domain Divergence and Adaptation Bound based on Multi-Class Scoring Disagreement}

The multi-class classification framework of Dogan \emph{et al.} \cite{dogan2016unified} decomposes a multi-class loss function into class-wise margins and margin violations (i.e., large-margin losses), and then aggregates these violations as a single loss value. Inspired by this framework, we propose in this paper a matrix-formed, \emph{multi-class scoring disagreement (MCSD)} to fully characterize the difference between any pair of scoring functions $\f', \f'' \in \mathcal{F}$, which is later used to define a distribution distance tailored to multi-class UDA. We first present the necessary definition of the absolute margin function.

\begin{definition}[Absolute margin function]\label{AbsMarginFuncDefinition}
The absolute margin function $\u: \mathbb{R}^K \times \mathcal{Y} \rightarrow \mathbb{R}^K$, with $\u = [\mu_1, \dots, \mu_K]^{\top}$, is defined on a multi-class scoring function $\f(\x) \in  \mathbb{R}^K$ and a label $y \in \mathcal{Y}$ as
	\begin{equation}
	\mu_k(\f(\x),y) =
	\begin{cases}
	+f_k(\x), & k = y \\
	-f_k(\x), & k \in \mathcal{Y} \setminus \{y\}
	\end{cases}.
	\end{equation}
\end{definition}
Given the sum-to-zero constraint $\sum_{k=1}^K f_k(\x) = 0$, the defined margin function enjoys the following properties \cite{dogan2016unified}.
\begin{itemize}
	\item $\mu_y(\f(\x),y)$ is non-decreasing w.r.t. $f_y(\x)$ ,
	\item $\mu_k(\f(\x),y)$ is non-increasing w.r.t. $f_k(\x)$ $\forall \ k \in \mathcal{Y} \setminus \{y\}$
	\item When $\mu_k(\f(\x),y) \geq 0 \ \forall \ k\in \mathcal{Y}$ and $\exists k\in\mathcal{Y}$ such that $\mu_k(\f(\x),y) > 0$, we have $\arg\max_{k\in\mathcal{Y}}f_k(\x) = y$ .
\end{itemize}
The third property characterizes the correct classification by checking non-negativeness/positiveness of absolute margins. To develop MCSD, we consider the ramp loss (\ref{equ-ramp_loss}) to penalize margin violations.
For $\rho > 0$ and a distribution $D$ over $\mathbb{R}$, ramp loss has the nice property of $\mathbb{E}_{x\sim D}\Phi_\rho(x) \geq \mathbb{E}_{x\sim D}\mathbb{I}[x\leq 0]$,  which is important to bound the target error $\mathcal{E}_Q(h_{\f})$ using margin-based loss functions defined over the scoring function $\f$.
}
\begin{definition}[Multi-class scoring disagreement]
	For a pair of scoring functions $\f', \f'' \in \mathcal{F}$, the multi-class scoring disagreement (MCSD) is defined with respect to a distribution $D$ over the domain $\mathcal{X}$ as
	\begin{equation}\label{EqnMCSD}
	\text{\rm MCSD}_D^{(\rho)}(\f', \f'') := \frac{1}{K}\mathbb{E}_{\x\sim D} \|\M^{(\rho)}(\f'(\x)) - \M^{(\rho)}(\f''(\x))\|_1,
	\end{equation}
	where $\|\cdot\|_1$ is the $L_1$ norm and $\M^{(\rho)}(\f(\x))\in [0,1]^{K\times K}$ is the matrix of absolute margin violations defined as
	\begin{equation} \label{equ:Martix-M}
	\M^{(\rho)}_{i,j}(\f(\x)) = \Phi_\rho(\mu_i(\f(\x),j)).
	\end{equation}
\end{definition}
Each column $\M_{:, k}^{(\rho)}$ of the matrix $\M^{(\rho)}$ computes violations of the absolute margin function $\u(\f(\cdot), k)$ w.r.t. a class $k\in \mathcal{Y}$, and the corresponding $\|\M_{:, k}^{(\rho)}(\f') - \M_{:, k}^{(\rho)}(\f'')\|_1$ measures the difference of margin violations between the scoring functions $\f'$ and $\f''$. The proposed MCSD (\ref{EqnMCSD}) is based on the absolute value aggregation of these disagreements.
To have an intuitive understanding of the behaviors of $\f'$, $\f''$, and the $\text{\rm MCSD}_D^{(\rho)}(\f', \f'')$, we plot in Figure \ref{Fig:f-f-disp-MCSD} the value of $\text{\rm MCSD}_D^{(\rho)}(\f', \f'')$ (firing on a single instance $\x$) in the case of $K = 3$ and $\rho = 5$, by fixing either $\f'(\x)$ or $\f''(\x)$ and using the other as the argument.

%\begin{table*}[!ht]
%	\caption{\textcolor{red}{A summary of measuring quantities that characterize varying degrees of multi-class scoring disagreement.  }} \label{Tab:disagreement_summary}
%	\centering
%	\begin{tabular}{l|C{40.3mm}C{40.3mm}C{40.3mm}}
%		\hline
%		\diagbox{Manner of loss construction}{Degree of characterization}                  & Labeling disagreement & An element of scoring disagreement & Aggregation of element-wise scoring disagreements \\  %& Architecture
%		\hline
%		Zero-one loss                              &  (Binary) hypothesis disagreement \cite{ben2010theory} & -- & -- \\
%		Relative margin-based loss   & --                                     & MD \cite{mdd} & -- \\
%		Absolute margin-based loss & $\widehat{\text{\rm MCSD}}$ (\ref{EqnMCSDDegen2HDH})  & $\widetilde{\text{\rm MCSD}}$                 (\ref{EqnMCSDDegen2MDD})           & MCSD (\ref{EqnMCSD}) \\
%		\hline
%	\end{tabular}
%\end{table*}

\begin{figure*}[t]
	\subfigure[]{
		\hspace{0.8cm}
		\begin{minipage}[t]{0.3\linewidth}
			\centering
			\includegraphics[width=0.95\linewidth] {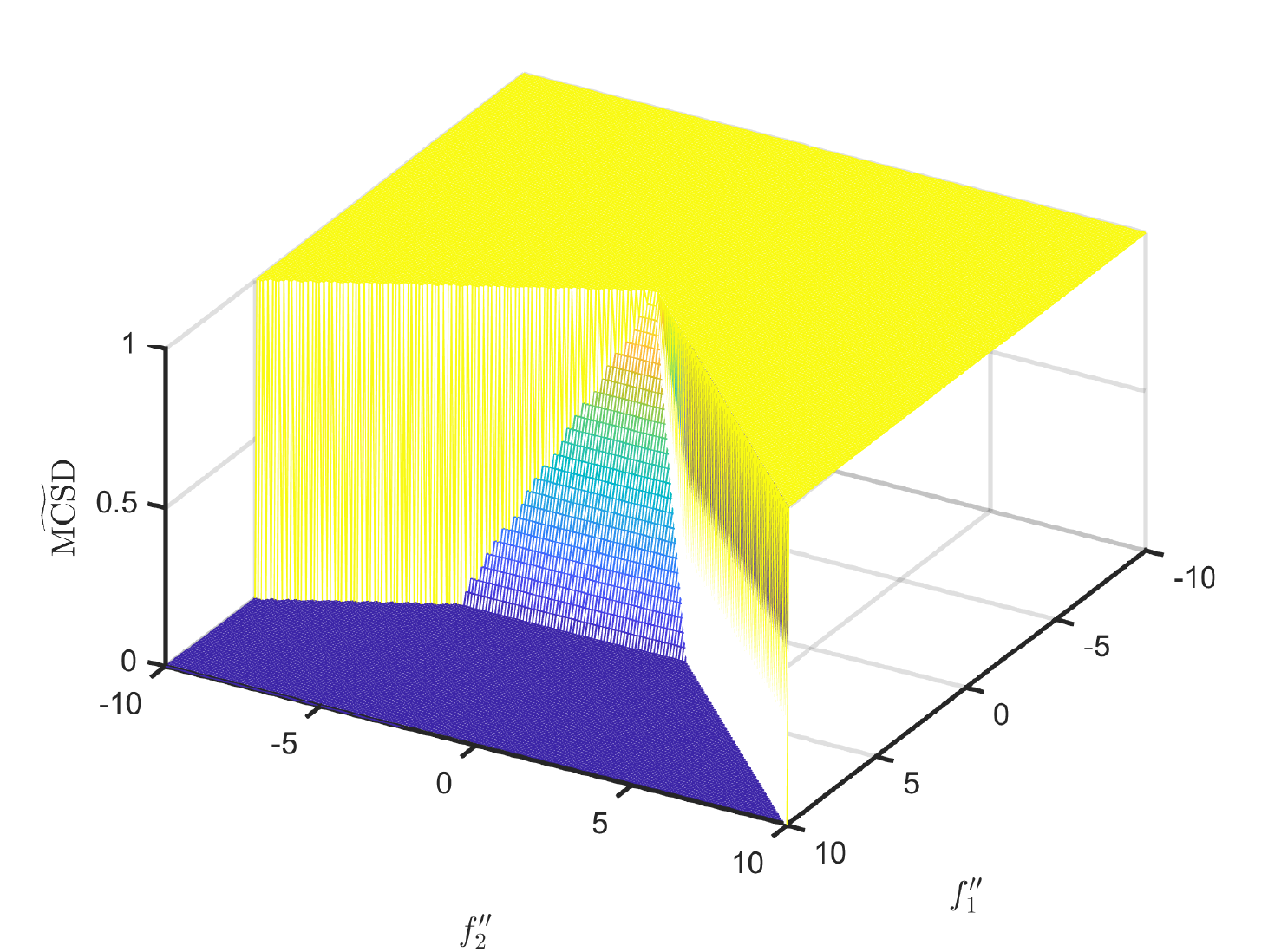}\hfill
			\includegraphics[width=0.95\linewidth] {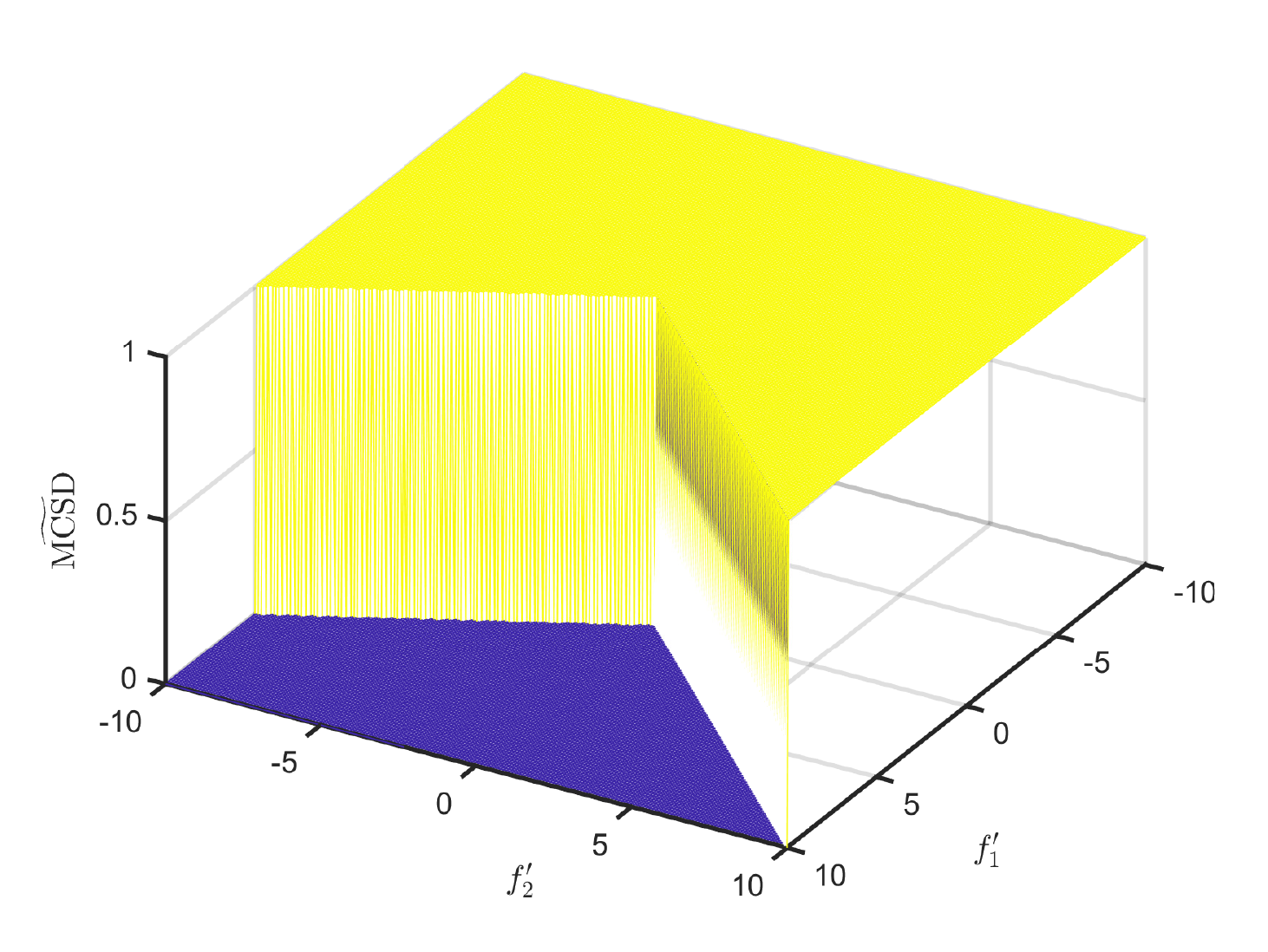}
	\end{minipage}}
	\hfill
	\subfigure[]{
		\begin{minipage}[t]{0.3\linewidth}
			\centering
			\includegraphics[width=0.95\linewidth] {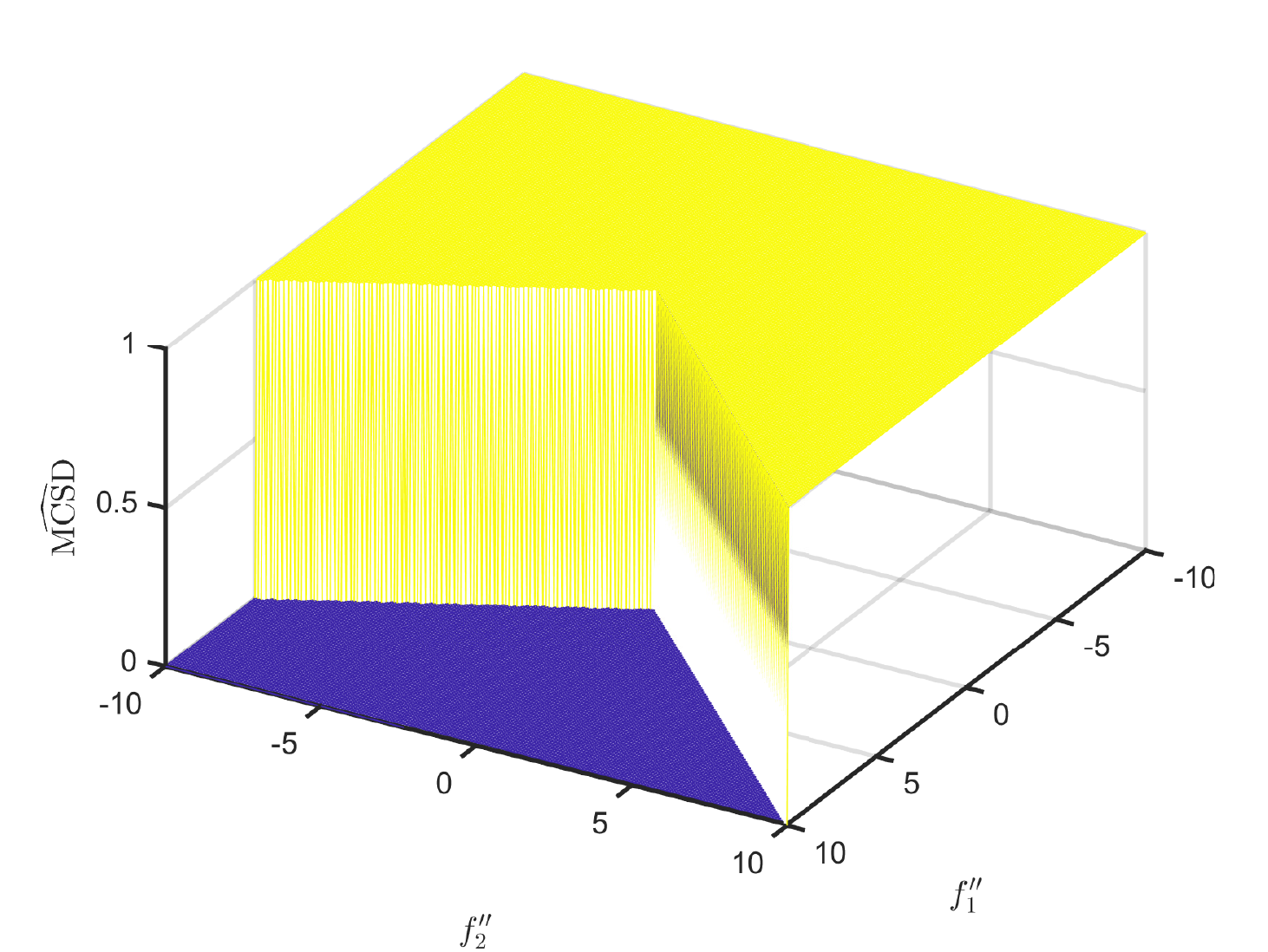}\hfill
			\includegraphics[width=0.95\linewidth] {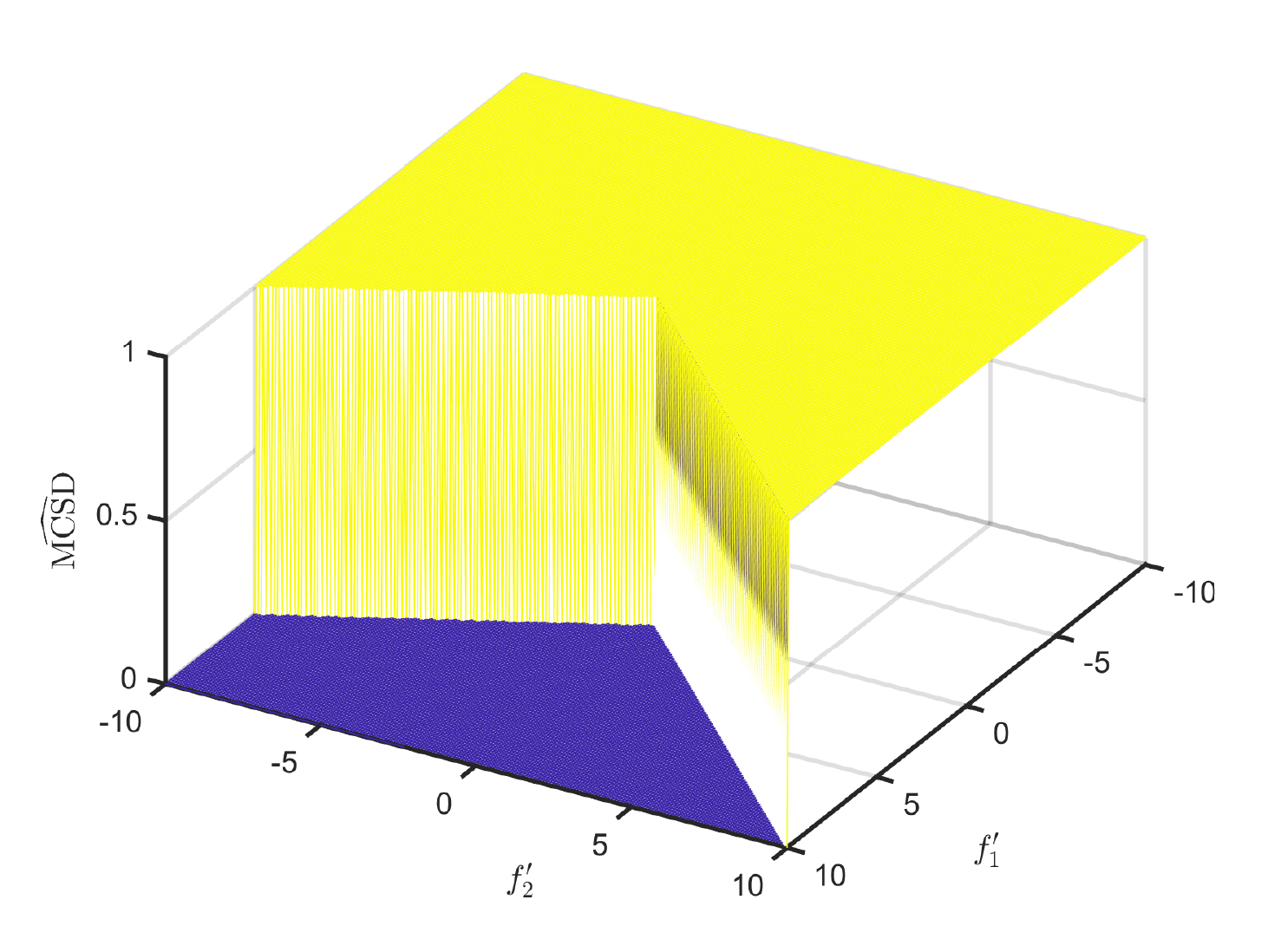}
	\end{minipage}}
	\hfill
	\subfigure[]{\label{Fig:f-f-disp-MCSD}
		\hspace{0.3cm}
		\begin{minipage}[t]{0.3\linewidth}
			\centering
			\includegraphics[width=0.95\linewidth] {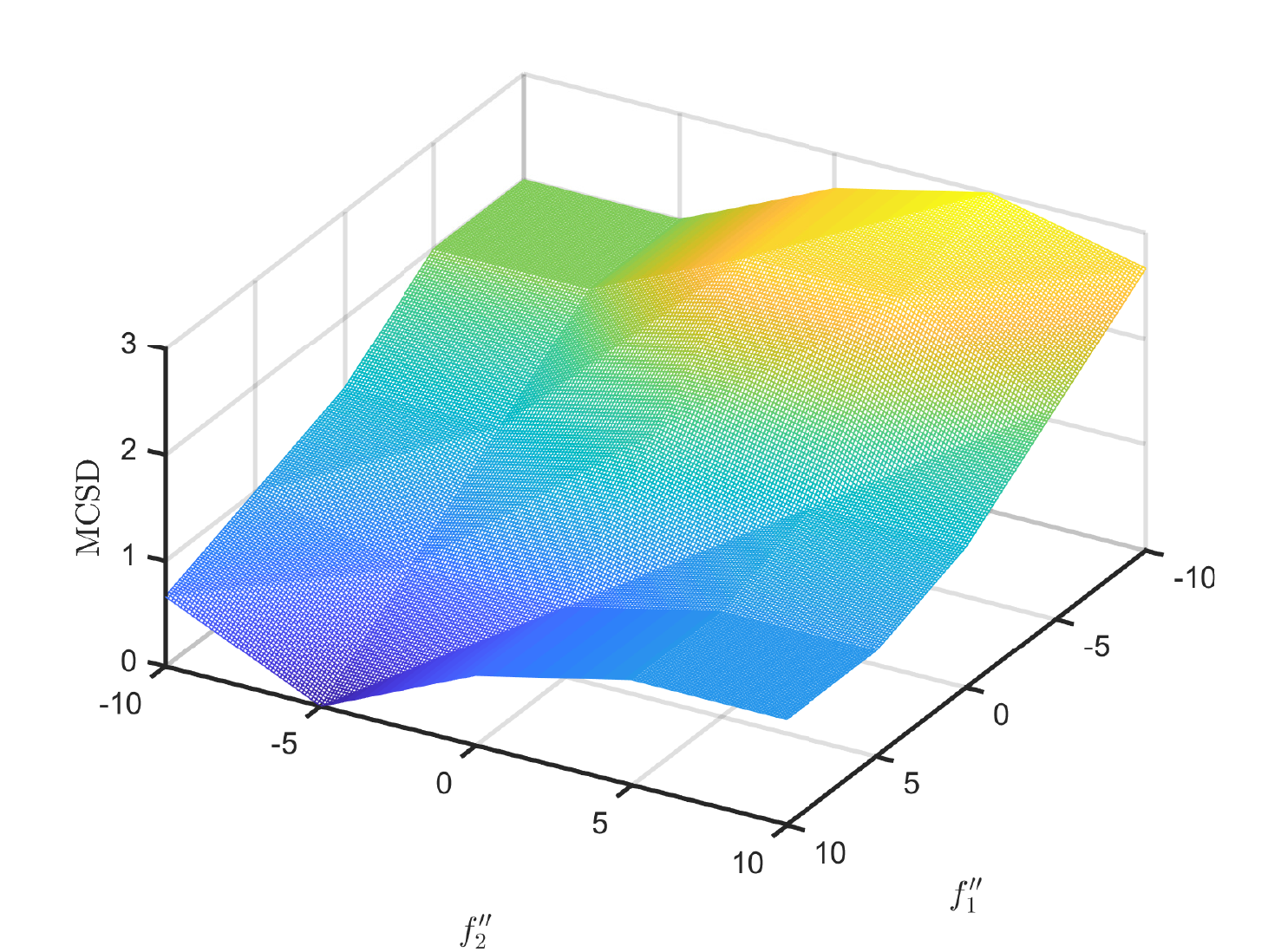}\hfill
			\includegraphics[width=0.95\linewidth] {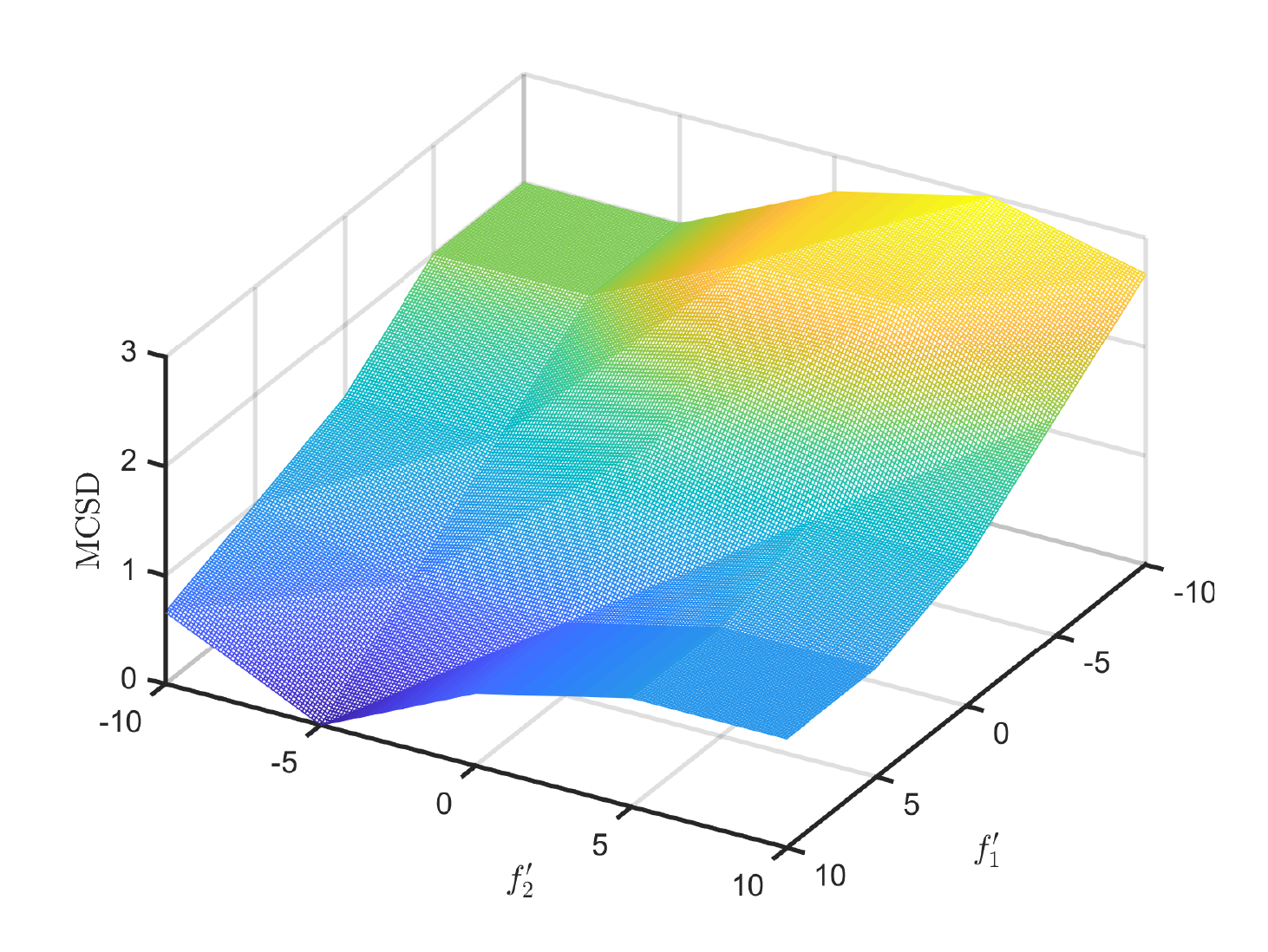}
	\end{minipage}}
	
	\caption{Plots of various disagreements between two scoring functions $\f'$ and $\f''$ firing on a single instance $\x$ in a case of $K=3$ and $\rho=5$, where the scoring functions satisfy the sum-to-zero constraint. Top row: fix $\f'(\x)$ to be $[10;-5;-5]$ and use $\f''(\x) = [f''_1;f''_2;-(f''_1 + f''_2)]$ as the argument; Bottom row: fix $\f''(\x)$ to be $[10;-5;-5]$ and use $\f'(\x) = [f'_1;f'_2;-(f'_1 + f'_2)]$ as the argument. \textcolor{black}{(a) The $\widetilde{\text{\rm MCSD}}$ (\ref{EqnMCSDDegen2MDD}), which can be considered as an absolute margin-based variant of the margin disparity (MD) \cite{mdd} (cf. terms in (\ref{equ:MarginDisparityDiscrepancy})); (b) the $\widehat{\text{\rm MCSD}}$ (\ref{EqnMCSDDegen2HDH}), which is an absolute margin-based equivalent of the hypothesis disagreement (HD) \cite{ben2010theory} (cf. terms in (\ref{equ:HDeltaHDivergence})); (c) our proposed $\text{\rm MCSD}$ (\ref{EqnMCSD}). } }\label{Fig:f-f-disp}
\end{figure*}

We have the following definition of distribution distance based on the proposed MCSD.
\begin{definition}[MCSD divergence] \label{DefMCSDDist}
	Given the definition of MCSD, we define the divergence between distributions $P_x$ and $Q_x$ over the domain $\mathcal{X}$ with respect to the space $\mathcal{F}$ as
	\begin{multline}\label{EqnMCSDDist}
	d^{(\rho)}_{MCSD}(P_x, Q_x) := \\ \sup\limits_{\f',\f''\in \mathcal{F}}[\text{\rm MCSD}_{Q_x}^{(\rho)}(\f', \f'') - \text{\rm MCSD}_{P_x}^{(\rho)}(\f', \f'')] .
	\end{multline}
\end{definition}
The proposed MCSD divergence (\ref{EqnMCSDDist}) satisfies the properties of non-negativity and triangle inequality, but it is not symmetric w.r.t. $P_x$ and $Q_x$. Nevertheless, we show its usefulness for multi-class UDA by developing the following bound.
\begin{theorem}\label{TheoremMCSDUDAGenBound}
	Fix $\rho>0$. For any scoring function $\f \in \mathcal{F}$, the following holds over the source and target distributions $P$ and $Q$,
	\begin{equation}\label{EqnMCSDUDAGenBound}
	\mathcal{E}_Q(h_{\f}) \leq \mathcal{E}^{(\rho)}_P(\f) + d^{(\rho)}_{MCSD}(P_x, Q_x) + \lambda,
	\end{equation}
	where the constant $\lambda = \mathcal{E}^{(\rho)}_P(\f^*) + \mathcal{E}^{(\rho)}_Q(\f^*)$ with $\f^* = \arg\min\limits_{\f\in \mathcal{F}}\mathcal{E}^{(\rho)}_P(\f) + \mathcal{E}^{(\rho)}_Q(\f)$, and
	\begin{equation}
	\mathcal{E}_Q(h_{\f}) := \mathbb{E}_{(\x,y)\sim Q} \mathbb{I}[h_{\f}(\x)\neq y] ,
	\end{equation}
	\begin{equation}
	\mathcal{E}^{(\rho)}_P(\f) := \mathbb{E}_{(\x,y)\sim P} \sum_{k=1}^K \Phi_{\rho}( \mu_k(\f(\x),y) )  .
	\end{equation}
\end{theorem}
\textcolor{black}{Theorem \ref{TheoremMCSDUDAGenBound} has a form similar to the domain adaptation bounds proposed by Ben-David \emph{et al.} \cite{ben2010theory} and Zhang \emph{et al.} \cite{mdd}; differently, it relies on the absolute margin-based loss function and MCSD divergence to achieve a full characterization of the difference between scoring functions of multi-class UDA.} As the bound (\ref{EqnMCSDUDAGenBound}) suggests, given the fixed $\lambda$, the expected target error $\mathcal{E}_Q(h_{\f})$ is determined by the distance $d^{(\rho)}_{MCSD}(P_x, Q_x)$ (and the expected loss $\mathcal{E}^{(\rho)}_P(\f)$ over the source domain); smaller $d^{(\rho)}_{MCSD}(P_x, Q_x)$ indicates better adaptation of multi-class UDA.
%To connect with domain adaptation bounds developed in literature, notably those proposed in \cite{mdd,ben2010theory}, we first present the following two \emph{scalar-valued, degenerate} versions of MCSD as
\textcolor{black}{To connect with domain adaptation bounds developed in the literature, notably those proposed in \cite{mdd,ben2010theory}, we first present the following absolute margin-based variant of MD \cite{mdd} (cf. terms in (\ref{equ:MarginDisparityDiscrepancy})) and the absolute margin-based equivalent of HD \cite{ben2010theory} (cf. terms in (\ref{equ:HDeltaHDivergence})), for a pair of scoring functions $\f', \f''\in \mathcal{F}$ w.r.t. a distribution $D$}
\begin{equation} \label{EqnMCSDDegen2MDD}
\widetilde{\text{\rm MCSD}}_D^{(\rho)}(\f', \f'') := \mathbb{E}_{\x\sim D}\Phi_{\rho/2}[\mu_{h_{\f''}(\x)}(\f''(\x),h_{\f'}(\x))] ,
\end{equation}
\begin{equation}\label{EqnMCSDDegen2HDH}
\widehat{\text{\rm MCSD}}_D^{(\rho)}(\f', \f'') := \mathbb{E}_{\x\sim D}\mathbb{I}[\Phi_{\rho}[\mu_{h_{\f''}(\x)}(\f''(\x),h_{\f'}(\x))] = 1] .
\end{equation}
\textcolor{black}{The terms (\ref{EqnMCSDDegen2MDD}) and (\ref{EqnMCSDDegen2HDH}) also measure the multi-class scoring disagreements to some extent, and give the corresponding distribution divergence $d^{(\rho)}_{\widetilde{MCSD}}$ as the absolute margin-based variant of MDD \cite{mdd}, and $d^{(\rho)}_{\widehat{MCSD}}$ as the absolute margin-based equivalent of $\frac{1}{2}\mathcal{H}\Delta\mathcal{H}$-divergence \cite{ben2010theory}, respectively. We have the following propositions for $\widetilde{\text{\rm MCSD}}$ and $\widehat{\text{\rm MCSD}}$. }

\begin{prop}
	\label{TheoremMCSDDegen2MDDUDAGenBound}
	Fix $\rho>0$. For any scoring function $\f \in \mathcal{F}$,
	\begin{equation}\label{EqnMCSDDegen2MDDUDAGenBound}
	\mathcal{E}_Q(h_{\f}) \leq \mathcal{E}^{(\rho)}_P(\f) + d^{(\rho)}_{\widetilde{MCSD}}(P_x, Q_x) + \lambda,
	\end{equation}
	%and minimizer of the distribution distance MDD proposed in \cite{XXX} (Definition 3.2 in \cite{XXX}) also minimizes the degenerate divergence $d^{(\rho)}_{\widetilde{MCSD}}(P_x, Q_x)$. {\color{red} Please check (1) whether the bound (\ref{EqnMCSDDegen2MDDUDAGenBound}) is correct, and (2) what does the minimizer of the MDD/MCSD distance exactly mean?}
\end{prop}

\begin{prop}
	\label{TheoremMCSDDegen2HDHUDAGenBound}
	Fix $\rho>0$. For any scoring function $\f \in \mathcal{F}$,
	\begin{equation}\label{EqnMCSDDegen2HDHUDAGenBound}
	\mathcal{E}_Q(h_{\f}) \leq \mathcal{E}^{(\rho)}_P(\f) + d^{(\rho)}_{\widehat{MCSD}}(P_x, Q_x) + \lambda.
	\end{equation}
	%and the $\frac{1}{2}\mathcal{H}\Delta\mathcal{H}$ distribution divergence proposed in \cite{ben2010theory} (Lemma 3 in \cite{ben2010theory}) is equivalent to the degenerate divergence $d^{(\rho)}_{\widehat{MCSD}}(P_x, Q_x)$.
\end{prop}
\textcolor{black}{Note that $\mathcal{E}_Q(h_{\f})$, $\mathcal{E}^{(\rho)}_P(\f)$, and $\lambda$ are defined as the same as these in Theorem \ref{TheoremMCSDUDAGenBound}, and $d^{(\rho)}_{\widetilde{MCSD}}$ and $d^{(\rho)}_{\widehat{MCSD}}$ are defined following the definition of $d^{(\rho)}_{MCSD}$ (cf. Definition \ref{DefMCSDDist}) by replacing the term $\text{\rm MCSD}$ (\ref{EqnMCSD}) with $\widetilde{\text{\rm MCSD}}$ (\ref{EqnMCSDDegen2MDD}) and $\widehat{\text{\rm MCSD}}$ (\ref{EqnMCSDDegen2HDH}) respectively.}
%\textcolor{red}{We summarize the various measuring quantities of multi-class scoring disagreement in Table \ref{Tab:disagreement_summary}. Specifically, the HD \cite{ben2010theory} (of binary classification) and its absolute margin-based equivalent (\ref{EqnMCSDDegen2HDH}) (of multi-class classification) solely characterize the labeling disagreement; the scalar-valued, relative margin-based MD of  \cite{mdd} and its absolute margin-based variant (\ref{EqnMCSDDegen2MDD}) improve by using multi-class margins to characterize the scoring disagreement; our proposed, matrix-formed MSCD (\ref{EqnMCSD}) aggregates the absolute margin violations and measures the element-wise scoring disagreements.}
%\textcolor{black}{Note that $d^{(\rho)}_{\widetilde{MCSD}}(P_x, Q_x)$ and $d^{(\rho)}_{\widehat{MCSD}}(P_x, Q_x)$ are defined by following the Definition \ref{DefMCSDDist} with replacing terms of $\widetilde{MCSD}$ and $\widehat{MCSD}$ respectively.}
Compared with the scalar-valued, absolute margin-based versions (\ref{EqnMCSDDegen2MDD}) and (\ref{EqnMCSDDegen2HDH}) (and also their corresponding ones in \cite{mdd} and \cite{ben2010theory}), our matrix-formed MCSD (\ref{EqnMCSD}) is able to characterize finer details of the scoring disagreements, as illustrated in Figure \ref{Fig:f-f-disp}. Consequently, the domain adaptation bound developed on the induced MCSD divergence would be beneficial to characterizing multi-class UDA in a finer manner, which possibly inspires better UDA algorithms.

\subsection{A Data-Dependent Multi-class Domain Adaptation Bound}

In this section, we extend the multi-class UDA bound in Theorem \ref{TheoremMCSDUDAGenBound} to a PAC bound, by showing that both terms of $\mathcal{E}^{(\rho)}_P(\f)$ and $d^{(\rho)}_{MCSD}(P_x, Q_x)$ can be estimated from finite samples. Our extension is based on the following notion of Rademacher complexity.
\begin{definition}[Rademacher complexity]
	Let $\mathcal{G}$ be a space of functions mapping from $\mathcal{Z}$ to $[a, b]$ and $\mathcal{S} = \left\{z_1,...,z_m\right\}$ be a fixed sample of size $m$ draw from the distribution $D$ over $\mathcal{Z}$. Then, the empirical Rademacher complexity of $\mathcal{G}$ with respect to the sample $\mathcal{S}$ is defined as
	\begin{equation}
	\widehat{\mathfrak{R}}_{\mathcal{S}}(\mathcal{G}) := \frac{1}{m} \mathbb{E}_\sigma \sup\limits_{g\in \mathcal{G}}\sum_{i=1}^{m}\sigma_i g(z_i),
	\end{equation}
	where $\{\sigma_i\}_{i=1}^m$ are independent uniform random variables taking values in $\left\{-1, +1\right\}$. The Rademacher complexity of $\mathcal{G}$ is defined as the expectation of $\widehat{\mathfrak{R}}_{\mathcal{S}}(\mathcal{G})$ over all samples of size $m$
	\begin{equation}
	\mathfrak{R}_{m,D}(\mathcal{G}) := \mathbb{E}_{\mathcal{S}\sim D^m} \widehat{\mathfrak{R}}_{\mathcal{S}}(\mathcal{G}).
	\end{equation}
\end{definition}
The Rademacher complexity captures the richness of a function space by measuring the degree to which it can fit random noise. The empirical version has the additional advantage that it is data-dependent and can be estimated from finite samples. We have the following definition from \cite{mohri2012foundations}, before introducing our Rademacher complexity-based adaptation bound.
\begin{definition}
	For a space $\mathcal{F}$ of scoring functions mapping from $\mathcal{X}$ to $\mathbb{R}^{|\mathcal{Y}|}$, we define
	\begin{equation}
	\Pi_1(\mathcal{F}) := \{ \x \rightarrow f_k(\x) | k \in \mathcal{Y}, \f\in \mathcal{F}\}.
	\end{equation}
\end{definition}
The defined space $\Pi_1(\mathcal{F})$ can be seen as the union of projections of $\mathcal{F}$ onto each output dimension.

\begin{theorem}\label{theorem-overall}
	Let $\mathcal{F}$ be the space of scoring functions mapping from $\mathcal{X}$ to $\mathbb{R}^K$. Let $P$ and $Q$ be the source and target distributions over $\mathcal{X\times Y}$, and $P_x$ and $Q_x$ be the corresponding marginal distributions over $\mathcal{X}$. Let $\widehat{P}$ and $\widehat{Q}_x$ denote the corresponding empirical distributions for a sample $\mathcal{S} = \{(\x_i^s,y_i^s)\}_{i=1}^{n_s}$ and a sample $\mathcal{T} = \{\x_j^t\}_{j=1}^{n_t}$. Fix $\rho>0$. Then, for any $\delta > 0$, with probability at least $1 - 3\delta$, the following holds for all $\f\in \mathcal{F}$
	\begin{equation}\label{equ-theorem1}
	\begin{aligned}
	\mathcal{E}_Q(h_{\f}) \leq & \mathcal{E}^{(\rho)}_{\widehat{P}}(\f) +  d^{(\rho)}_{MCSD}(\widehat{P}_x, \widehat{Q}_x) \\
	& + (\frac{2K^2}{\rho} + \frac{4K}{\rho})\widehat{\mathfrak{R}}_\mathcal{S}(\Pi_1(\mathcal{F}))
	+ \frac{4K}{\rho}\widehat{\mathfrak{R}}_\mathcal{T}(\Pi_1(\mathcal{F})) \\
	& + 6K\sqrt{\frac{\log\frac{4}{\delta}}{2n_s}} + 3K\sqrt{\frac{\log\frac{4}{\delta}}{2n_t}} + \lambda,
	\end{aligned}
	\end{equation}
	where the constant $\lambda = \min\limits_{\f\in \mathcal{F}}\mathcal{E}^{(\rho)}_P(\f) + \mathcal{E}^{(\rho)}_Q(\f)$, and
	\begin{equation}
	\mathcal{E}^{(\rho)}_{\widehat{P}}(\f) := \frac{1}{n_s}\sum_{i=1}^{n_s} \sum_{k=1}^K \Phi_{\rho}( \mu_k(\f(\x^s_i),y^s_i) )  .
	\end{equation}
\end{theorem}

\section{Connecting Theory with Algorithms}
\label{Sec:SymNet}

In the derived bound (\ref{equ-theorem1}) of multi-class UDA, the constant $\lambda$ and complexity terms are assumed to be fixed given the hypothesis space $\mathcal{F}$. To minimize the expected target error $\mathcal{E}_Q(h_{\f})$, one is tempted to minimize the first two terms of $\mathcal{E}^{(\rho)}_{\widehat{P}}(\f)$ and $d^{(\rho)}_{MCSD}(\widehat{P}_x, \widehat{Q}_x)$. In practice, a function $\psi$ of the feature extractor is typically used to lift the input space $\mathcal{X}$ to a feature space $\mathcal{X}^{\psi} = \{ \psi(\x) | x \in \mathcal{X} \}$, where with a slight abuse of notation, the space $\mathcal{F}$ of the scoring function $\f: \mathcal{X}^{\psi}\rightarrow \mathbb{R}^{|\mathcal{Y}|} = \mathbb{R}^K$ and the induced labeling function $h_{\f}: \psi(\x)\rightarrow \arg\max_{k\in \mathcal{Y}}f_k(\psi(\x))$ are again well defined. We correspondingly write as $P^{\psi}$ and $Q^{\psi}$ for the source and target distributions over the lifted domain $\mathcal{X}^{\psi}\times \mathcal{Y}$, and their empirical (or marginal) versions as $\widehat{P}^{\psi}$ and $\widehat{Q}^{\psi}$ (or $P_x^{\psi}$ and $Q_x^{\psi}$). The function $\psi$ is typically implemented as a learnable deep network.

Given the learnable $\psi$, minimizing the right hand side of the bound (\ref{equ-theorem1}) can be achieved by identifying $\psi^{*}$ that minimizes $d^{(\rho)}_{MCSD}(\widehat{P}_x^{\psi}, \widehat{Q}_x^{\psi})$, and additionally identifying $\f^{*}$ with $\psi^{*}$ that minimizes $\mathcal{E}^{(\rho)}_{ \widehat{P}^{\psi} }(\f)$. Recall that the MCSD divergence (\ref{EqnMCSDDist}) is defined by taking the supremum over all pairs of $\f', \f'' \in \mathcal{F}$. Spelling $d^{(\rho)}_{MCSD}(\widehat{P}_x^{\psi}, \widehat{Q}_x^{\psi})$ out gives the following general objective of minimax optimization for multi-class UDA
%\begin{equation}
%\min\limits_{\f,\psi} \mathcal{E}^{(\rho)}_{ \widehat{P}^{\psi} }(\f) +  d^{(\rho)}_{MCSD}(\widehat{P}_x^{\psi}, \widehat{Q}_x^{\psi}) .
%\end{equation}
\begin{equation}
\begin{gathered} \label{equ-opt1}
\min\limits_{\f,\psi}\  \mathcal{E}^{(\rho)}_{\widehat{P}^{\psi}}(\f) +  [\text{\rm MCSD}_{\widehat{Q}_x^{\psi}}^{(\rho)}(\f', \f'') - \text{\rm MCSD}_{\widehat{P}_x^{\psi}}^{(\rho)}(\f', \f'') ], \\
\max\limits_{\f',\f''}\  [\text{\rm MCSD}_{\widehat{Q}_x^{\psi}}^{(\rho)}(\f', \f'') - \text{\rm MCSD}_{\widehat{P}_x^{\psi}}^{(\rho)}(\f', \f'')] ,
\end{gathered}
\end{equation}
which suggests an adversarial learning strategy to promote domain-invariant conditional feature distributions via the learned $\psi$, thus extending \cite{dann} to account for multi-class UDA. We term the general algorithm (\ref{equ-opt1}) via the adversarial learning strategy as \emph{Multi-class Domain-adversarial learning Networks (McDalNets)}. Figure \ref{FigMcDalNets} gives an architectural illustration, where the scoring function $\f$ is for the multi-class classification task of interest, and $\f'$ and $\f''$ are auxiliary functions for the learning of $\psi$. Since $\f$, $\f'$, and $\f''$ contain all the parameters of classifiers, we also use them to respectively refer to the task and auxiliary classifiers.

\begin{figure}[htb]
	\includegraphics[width=1.0\linewidth] {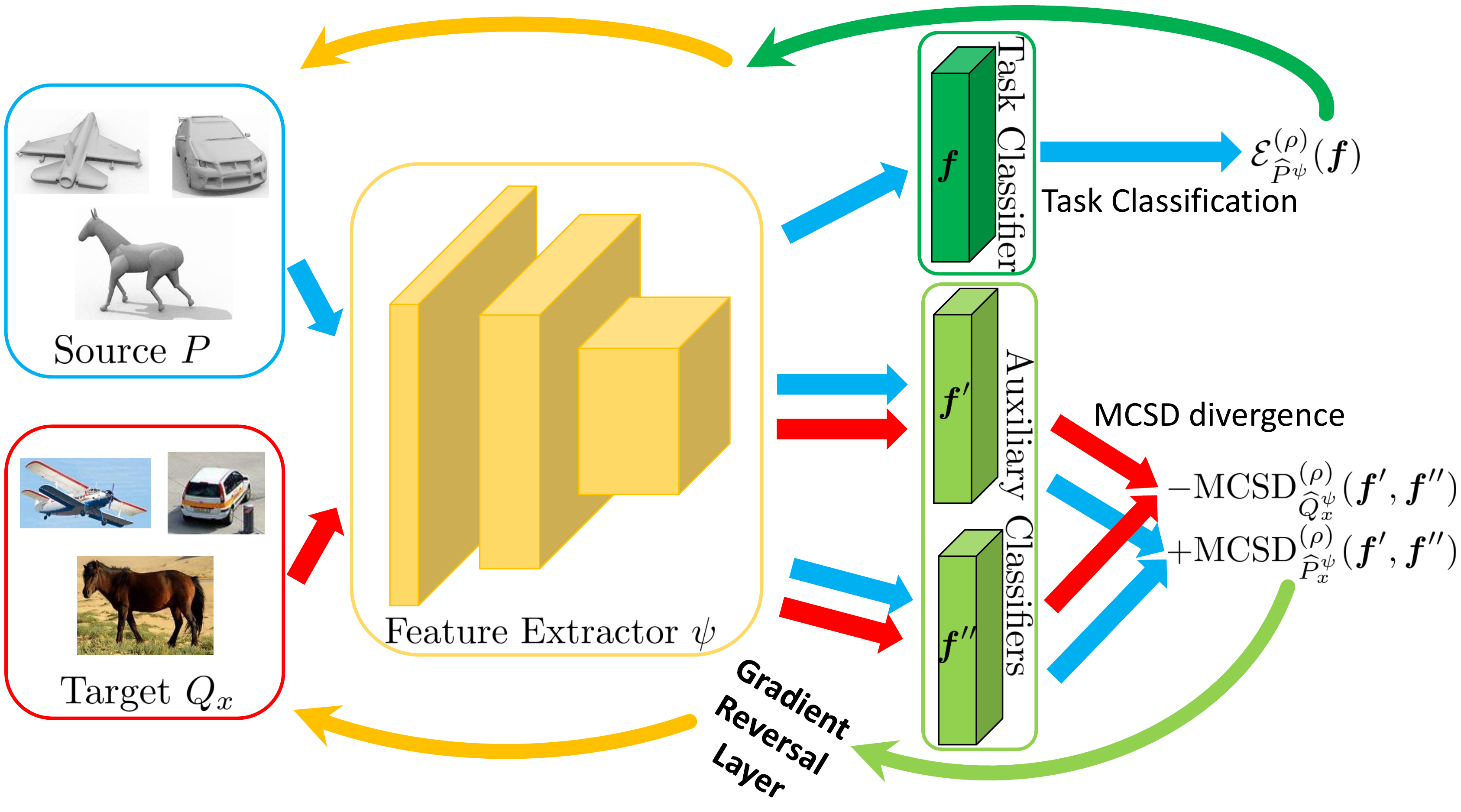}
	\caption{An architectural illustration of Multi-class Domain-adversarial learning Networks (McDalNets), which is motivated from the theoretically derived objective (\ref{equ-opt1}). The gradient reversal layer is adopted here to implement the adversarial objective; we note that other implementations (e.g., those discussed in \cite{adda}) would apply as well.}
	\label{FigMcDalNets}
\end{figure}

\subsection{Different Algorithms of Multi-class Domain-adversarial learning Networks}
\label{SecMcDalNetAlgms}

The proposed MCSD divergence is amenable to the theoretical analysis of multi-class UDA. However, it is difficult to directly optimize the MCSD based problem (\ref{equ-opt1}) via stochastic gradient descent (SGD), due to the use of ramp loss $\Phi_\rho$ in MCSD (\ref{EqnMCSD}) that causes an issue of vanishing gradients. \footnote{We have tried to train the McDalNets (illustrated in Figure \ref{FigMcDalNets}) with the exact objective (\ref{equ-opt1}). However, it turns out that the optimization stagnates after a few iterations, since absolute values of the outputs of scoring functions increase over the predefined $\rho$, and the gradients thus vanish.} To develop specific algorithms of McDalNets that are optimization-friendly, we consider surrogate functions of MCSD (\ref{EqnMCSD}), which are easier to be trained by SGD and also able to characterize the disagreements of all $K$ pairs of the corresponding elements in scoring functions $\f', \f'' \in \mathcal{F}$. These surrogates give the following objectives of specific algorithms
\begin{equation}
\begin{gathered} \label{equ-opt2}
\min\limits_{\f,\psi}\  \mathcal{L}_{\widehat{P}^{\psi}}(\f) +  [\text{\rm SurMCSD}_{\widehat{Q}_x^{\psi}}(\f', \f'') - \text{\rm SurMCSD}_{\widehat{P}_x^{\psi}}(\f', \f'') ], \\
\max\limits_{\f',\f''}\  [\text{\rm SurMCSD}_{\widehat{Q}_x^{\psi}}(\f', \f'') - \text{\rm SurMCSD}_{\widehat{P}_x^{\psi}}(\f', \f'')] ,
\end{gathered}
\end{equation}
\noindent \emph{respectively} with $\text{\rm SurMCSD}_{D^{\psi}}(\f', \f'')$ over a distribution $D^{\psi}$ as
\begin{equation}\label{equ-mcd}
\begin{aligned}
%\text{\rm SurMCSD}_{D^{\psi}}(\f', \f'') = \mathbb{E}_{\x\sim D}\frac{1}{K}\|\phi(\f'(\psi(\x)))-\phi(\f''(\psi(\x)))\|_1,
& (L_1/\text{\rm MCD \cite{mcd}} ): \mathbb{E}_{\x\sim D}\frac{1}{K}\|\phi(\f'(\psi(\x)))-\phi(\f''(\psi(\x)))\|_1,
\end{aligned}
\end{equation}
\begin{equation}\label{equ-kl}
\begin{aligned}
%\text{\rm SurMCSD}_{D^{\psi}}(\f', \f'') = & \mathbb{E}_{\x\sim D} \frac{1}{2}[\text{\rm KL}(\phi(\f'(\psi(\x))), \phi(\f''(\psi(\x)))) \\
%& + \text{\rm KL}(\phi(\f''(\psi(\x))), \phi(\f'(\psi(\x))))],
& (\text{\rm KL}): \mathbb{E}_{\x\sim D} \frac{1}{2}[\text{\rm KL}(\phi(\f'(\psi(\x))), \phi(\f''(\psi(\x)))) \\
& \qquad\qquad + \text{\rm KL}(\phi(\f''(\psi(\x))), \phi(\f'(\psi(\x))))],
\end{aligned}
\end{equation}
\begin{equation}\label{equ-cross-entropy}
\begin{aligned}
%\text{\rm SurMCSD}_{D^{\psi}}(\f', \f'') = & \mathbb{E}_{\x\sim D} \frac{1}{2}[\text{\rm CE}(\phi(\f'(\psi(\x))), \phi(\f''(\psi(\x)))) \\
%& + \text{\rm CE}(\phi(\f''(\psi(\x))), \phi(\f'(\psi(\x))))],
& (\text{\rm CE}): \mathbb{E}_{\x\sim D} \frac{1}{2}[\text{\rm CE}(\phi(\f'(\psi(\x))), \phi(\f''(\psi(\x)))) \\
& \qquad\qquad + \text{\rm CE}(\phi(\f''(\psi(\x))), \phi(\f'(\psi(\x))))],
\end{aligned}
\end{equation}
where $\phi(\cdot)$ is the softmax operator, $\text{\rm KL}(\cdot,\cdot)$ is the Kullback-Leibler divergence, and $\text{\rm CE}(\cdot,\cdot)$ is the cross-entropy function, and due to the same issue from the ramp loss, we have used a standard log loss
\begin{equation}\label{logloss}
\mathcal{L}_{\widehat{P}^{\psi}}(\f) = \mathbb{E}_{(\x,y)\sim \widehat{P}} -\log[\phi_{y}(\f(\psi(x)))],
\end{equation}
to replace the term $\mathcal{E}^{(\rho)}_{\widehat{P}^{\psi}}(\f)$ of empirical source error in (\ref{equ-opt1}). While MCSD (\ref{EqnMCSD}) takes a matrix-formed difference, the optimization-friendly surrogates (\ref{equ-mcd}), (\ref{equ-kl}), and (\ref{equ-cross-entropy}) generally take vector forms that characterize scoring disagreements between $K$ entry pairs of $\f'$ and $\f''$. In fact, we have the following proposition to show the equivalance of the matrix-formed MCSD to an aggregation of $K$ disagreements between any entry pair of $\f'$ and $\f''$.

\begin{prop}\label{PropMCDRelateMCSD}
Given the ramp loss $\Phi_\rho$ defined as (\ref{equ-ramp_loss}), there exists a distance measure $\varphi: \mathbb{R}\times \mathbb{R} \rightarrow \mathbb{R}_+$ defined as
\begin{equation}
\varphi(a,b) =  (K-1)|\Phi_\rho(-a)-\Phi_\rho(-b)|+|\Phi_\rho(a)-\Phi_\rho(b)|, \nonumber	
\end{equation}
%\begin{equation}
%\textrm{--- Please specify the function} \ \varphi \ \textrm{here ---} , \nonumber	
%\end{equation}
such that the matrix-formed $\| \M^{(\rho)}(\f'(\x)) - \M^{(\rho)}(\f''(\x)) \|_1$ in MCSD (\ref{EqnMCSD}) can be calculated as the sum of $\varphi$-distance values of $K$ entry pairs between $f'_k(\x)$ and $f''_k(\x)$, i.e.,
\begin{equation}
\| \M^{(\rho)}(\f'(\x)) - \M^{(\rho)}(\f''(\x)) \|_1 = \sum_{k=1}^K \varphi(f'_k(\x), f''_k(\x)). \nonumber
\end{equation}
\end{prop}

We also consider an algorithm that replaces the MCSD terms of (\ref{equ-opt1}) with
%surrogate function of the scalar-valued, degenerate MCSD version of (\ref{EqnMCSDDegen2MDD}), giving rise to
\textcolor{black}{a surrogate function of the scalar-valued, absolute margin-based version (\ref{EqnMCSDDegen2MDD}), giving rise to}
\begin{equation}
\begin{gathered} \label{equ-opt2-degen}
\min\limits_{\f,\psi}\  \mathcal{L}_{\widehat{P}^{\psi}}(\f) +  [\widetilde{\text{\rm SurMCSD}}_{\widehat{Q}_x^{\psi}}(\f', \f'') - \widetilde{\text{\rm SurMCSD}}_{\widehat{P}_x^{\psi}}(\f', \f'') ], \\
\max\limits_{\f',\f''}\  [\widetilde{\text{\rm SurMCSD}}_{\widehat{Q}_x^{\psi}}(\f', \f'') - \widetilde{\text{\rm SurMCSD}}_{\widehat{P}_x^{\psi}}(\f', \f'')] ,
\end{gathered}
\end{equation}
\noindent with $\widetilde{\text{\rm SurMCSD}}_{D^{\psi}}(\f', \f'')$ over a distribution $D^{\psi}$ as \footnote{For better optimization, we follow \cite{gan,mdd} and practically implement the surrogate disagreement terms in (\ref{equ-opt2-degen}) as
	\begin{equation} \label{equ-mdd-form-long}
	\begin{gathered}
	\widetilde{\text{\rm SurMCSD}}_{\widehat{Q}_x^{\psi}}(\f', \f'') = \mathbb{E}_{\x\sim \widehat{Q}}\log[1-\phi_{h_{\f'}(\psi(\x))}(\f''(\psi(\x)))], \\
	\widetilde{\text{\rm SurMCSD}}_{\widehat{P}_x^{\psi}}(\f', \f'') = \mathbb{E}_{\x\sim \widehat{P}}-\log[\phi_{h_{\f'}(\psi(\x))}(\f''(\psi(\x)))].
	\end{gathered}
	\end{equation}
}
\begin{equation}\label{equ-mdd-form-ori}
\begin{aligned}
%\widetilde{\text{\rm SurMCSD}}_{D^{\psi}}(\f', \f'') = \mathbb{E}_{\x\sim D}-\log[\phi_{h_{\f'}(\psi(\x))}(\f''(\psi(\x)))] .
& (\text{\rm MDD \cite{mdd} variant}): \mathbb{E}_{\x\sim D}-\log[\phi_{h_{\f'}(\psi(\x))}(\f''(\psi(\x)))] .
\end{aligned}
\end{equation}
\textcolor{black}{Similarly, an algorithm based on the scalar-valued, absolute margin-based version (\ref{EqnMCSDDegen2HDH}) can be considered as an equivalent of the DANN algorithm \cite{dann} }
%Similarly, an algorithm based on the degenerate (\ref{EqnMCSDDegen2HDH}) can be considered as an equivalent of the DANN algorithm \cite{dann}
\noindent with $\widehat{\text{\rm SurMCSD}}_{D^{\psi}}(\f', \f'')$ over a distribution $D^{\psi}$ as \footnote{For better optimization, we follow \cite{dann} and practically implement the surrogate disagreement terms in (\ref{equ-dann-form-ori}) as
	\begin{equation} \label{equ-dann-form-dann}
	\begin{gathered}
	\widehat{\text{\rm SurMCSD}}_{\widehat{Q}_x^{\psi}}(\f', \f'') = \mathbb{E}_{\x\sim \widehat{Q}}\log[1- {\rm sigmoid}(d(\psi(\x)))], \\
	\widehat{\text{\rm SurMCSD}}_{\widehat{P}_x^{\psi}}(\f', \f'') = \mathbb{E}_{\x\sim \widehat{P}}-\log[{\rm sigmoid}(d(\psi(\x)))].
	\end{gathered}
	\end{equation}
}
\begin{equation}\label{equ-dann-form-ori}
\begin{aligned}
& (\text{\rm DANN \cite{dann}}): \mathbb{E}_{\x\sim D}-\log[{\rm sigmoid}(d(\psi(\x)))],
\end{aligned}
\end{equation}
where $d: D^{\psi} \to \mathbb{R}$ is a mapping function, and $\mathbb{I}[d(\psi(\x)) > 0] = \mathbb{I}[h_{\f'}(\psi(\x)) = h_{\f''}(\psi(\x))]$.

%We also note that a similar algorithm based on the degenerate (\ref{EqnMCSDDegen2HDH}) can be considered as an equivalent of the DANN algorithm \cite{dann}.
%{\color{red} How about the algorithm based on the degenerate (\ref{EqnMCSDDegen2HDH})? Can we specify its objective together into the above algorithm based on (\ref{EqnMCSDDegen2MDD})?}

We note that algorithms discussed above resemble some recently proposed ones in the literature of UDA. For example, the objective (\ref{equ-opt2}) with the surrogate (\ref{equ-mcd}) is equivalent to the MCD algorithm \cite{mcd};
%the objective (\ref{equ-opt2-degen}) with the degenerate surrogate (\ref{equ-mdd-form-ori}) can be considered as a variant of MDD \cite{mdd}.
\textcolor{black}{the objective (\ref{equ-opt2-degen}) with the surrogate (\ref{equ-mdd-form-ori}) can be considered as a variant of MDD \cite{mdd}. }
In Section \ref{SecExps}, we conduct ablation studies to investigate the efficacy of these algorithms, and compare them with a new one to be presented shortly, which is motivated from the same theoretically derived objective (\ref{equ-opt1}).

\subsection{A New Algorithm of Domain-Symmetric Networks}\label{SecSymmNets}

\begin{figure*}
	\begin{center}
		\includegraphics[width=0.8\linewidth] {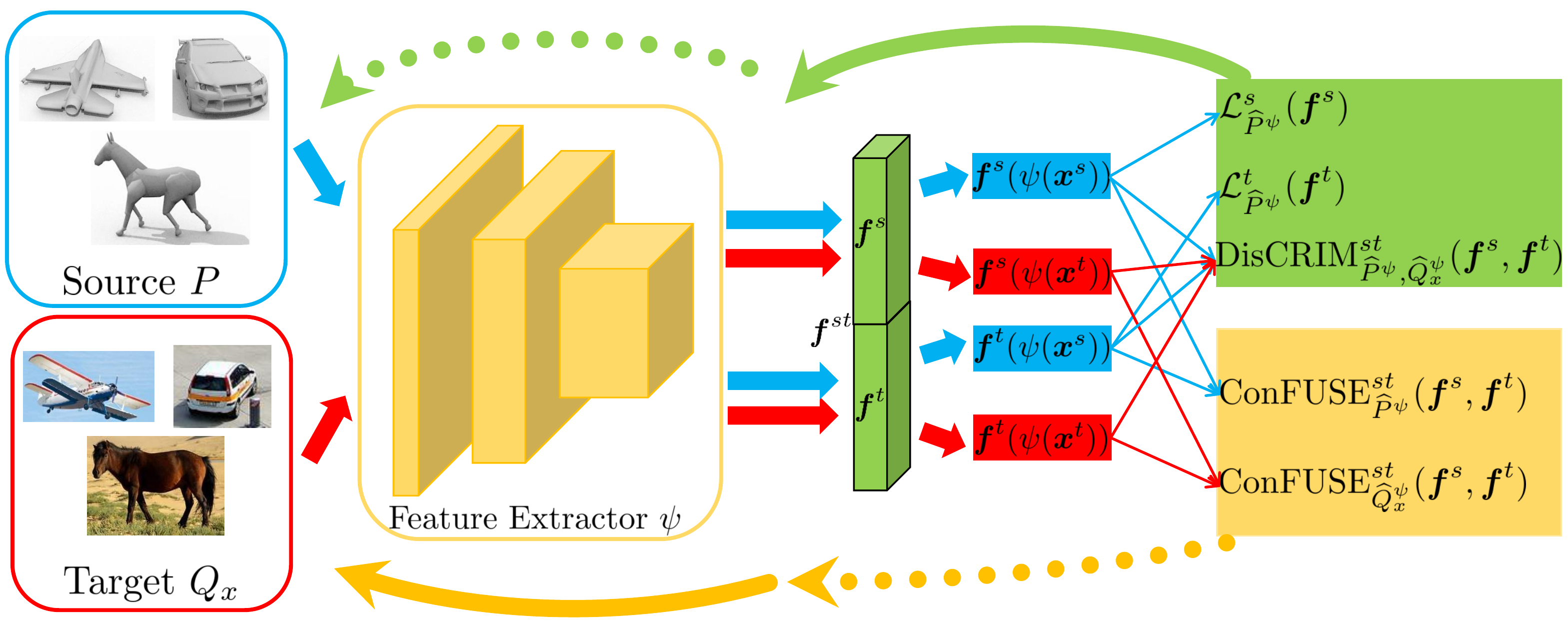}
	\end{center}
	\caption{The architecture of our proposed SymmNets, which includes a feature extractor $\psi$ and three classifiers of $\mathbf{f}^s, \mathbf{f}^t$, and $\mathbf{f}^{st}$. Note that the classifier $\mathbf{f}^{st}$ shares its layer neurons with those of $\mathbf{f}^s$ and $\mathbf{f}^t$. Parameters of the classifiers (i.e., $\mathbf{f}^s, \mathbf{f}^t$, and $\mathbf{f}^{st}$) and those of the feature extractor $\psi$ are respectively updated by gradients from loss terms in green and yellow boxes. Please refer to the main text for how the objectives are defined.}
	\label{FigSymmNet}
\end{figure*}

Apart from the task classifier $\f$, algorithms of McDalNets presented above use two auxiliary classifiers $\f'$ and $\f''$ only for learning $\psi$, which is less efficient in the use of parameters. To improve the efficiency, we propose an integrated scheme that concatenates $\f'$ and $\f''$ as $[\f'; \f''] \in \mathbb{R}^{2K}$, and lets them be respectively responsible for the classification of the source and target instances, as shown in Figure \ref{FigSymmNet}. We correspondingly use the notations of $\f^s$ and $\f^t$ to replace $\f'$ and $\f''$, and denote the concatenated classifier as $\f^{st}$, which \emph{shares parameters with} $\f^s$ and $\f^t$. We term such a network as Domain-Symmetric Networks (SymmNets) due to the symmetry of class-wise neuron distributions in $\f^s$ and $\f^t$.

To achieve the theoretically motivated learning objective (\ref{equ-opt1}), we have the following two designs to train SymmNets.
\begin{itemize}
	\item Since target data $\{ \x_j^t \}_{j=1}^{n_t}$ are unlabeled, to enforce symmetric predictions between the respective $K$ neurons of $\f^s$ and $\f^t$, we use a \emph{cross-domain training scheme} that trains the target classifier $\f^t$ using labeled source data $\{ (\x_i^s, y_i^s) \}_{i=1}^{n_s}$.
	\item While different algorithms presented in Section \ref{SecMcDalNetAlgms} take the adversarial training strategy (e.g., a manner of reverse gradients \cite{dann}) to learn domain-invariant conditional feature distributions, for SymmNets, we instead use a \emph{domain confusion (and discrimination) training scheme} on the concatenated classifier $\f^{st}$ to achieve the same goal.
\end{itemize}
We introduce the following notations before presenting the algorithm of SymmNets. For an input $\x$, $\f^s(\psi(\x)) \in \mathbb{R}^K$ and $\f^t(\psi(\x)) \in \mathbb{R}^K$ are the output vectors before the softmax operator $\phi$, and we denote $\vect{p}^s(\x) = \phi(\f^s(\psi(\x))) \in [0, 1]^K$ and $\vect{p}^t(\x) = \phi(\f^t(\psi(\x))) \in [0, 1]^K$. We also apply softmax to the output of the concatenated classifier $\f^{st}$, resulting in $\vect{p}^{st}(\x) = \phi([\f^s(\psi(\x));\f^t(\psi(\x))]) \in [0, 1]^{2K}$. For ease of subsequent notations, we also write $p^s_k(\x)$ (\emph{resp.} $p^t_k(\x)$ or $p^{st}_k(\x)$), $k \in \{1, \dots, K\}$, for the $k^{th}$ element of the probability vector $\vect{p}^s(\x)$ (\emph{resp.} $\vect{p}^t(\x)$ or $\vect{p}^{st}(\x)$) predicted by $\f^s$ (\emph{resp.} $\f^t$ or $\f^{st}$).

\vspace{0.1cm}
\noindent\textbf{Learning of the Source and Target Task Classifiers} We train the task classifier $\f^s$ using a standard log loss over the labeled source data as follows
\begin{equation} \label{Loss_task_sc}
\min_{\f^s} {\cal{L}}_{\widehat{P}^{\psi}}^s (\f^s) = - \frac{1}{n_s} \sum_{i=1}^{n_s}  \omega_{y_i^s} \mathrm{log}(p^s_{y_i^s}(\x_i^s)),
\end{equation}
where $\omega_{y_i^s} \in [0,1]$ is fixed as the value of $1$ for closed set and open set UDA, and will be turned active in Section \ref{SecExt2PartialOpenSet} for the extension of SymmNets to the setting of partial UDA.

To account for element-wise disagreements between predictions of $\f^s$ and $\f^t$, it is necessary to establish neuron-wise correspondence between them. To this end, we propose a cross-domain training scheme that trains the target classifier $\f^t$ again using the labeled source data
\begin{equation} \label{Loss_task_tc}
\min_{\f^t} {\cal{L}}_{\widehat{P}^{\psi}}^t (\f^t) = - \frac{1}{n_s} \sum_{i=1}^{n_s} \omega_{y_i^s} \mathrm{log}(p^t_{y_i^s}(\x_i^s)) .
\end{equation}
At a first glance, it seems that training $\f^t$ on $\{ (\x_i^s, y_i^s) \}_{i=1}^{n_s}$ only makes it a duplicate classifier of $\f^s$. However, its effect on establishing neuron-wise correspondence between $\f^s$ and $\f^t$  is very essential to achieve learning of domain-invariant features via the objectives of domain confusion and discrimination, as presented shortly. We also present ablation studies in Section \ref{SecExps} that verify the efficacy of the scheme (\ref{Loss_task_tc}).

\vspace{0.1cm}
\noindent\textbf{Adversarial Feature Learning via Domain Confusion and Discrimination} Algorithms in Section \ref{SecMcDalNetAlgms} use surrogate MCSD functions and minimize the induced MCSD divergence to learn $\psi$, in order to align conditional feature distributions across source and target domains. Instead of using surrogate MCSD functions in SymmNets, we propose domain confusion objectives to directly reduce the domain divergence, by learning $\psi$ such that it produces features whose scoring disagreements between $\f^s$ and $\f^t$ (\emph{via their parameter-sharing $\f^{st}$}) on both the source and target domains are \emph{equally small (and ideally null)}. Our confusion objectives are as follows
\begin{align}
\notag\min_{\psi} \text{\rm ConFUSE}_{\widehat{P}^{\psi}}^{st} (\f^{s},\f^t) = & - \frac{1}{2n_s} \sum_{i=1}^{n_s} \omega_{y_i^s} \mathrm{log}(p^{st}_{y_i^s}(\x_i^s)) \\ & - \frac{1}{2n_s} \sum_{i=1}^{n_s} \omega_{y_i^s} \mathrm{log}(p^{st}_{y_i^s+K}(\x_i^s)) , \label{Loss:cate_conf_s} \\
\notag\min_{\psi} \text{\rm ConFUSE}_{\widehat{Q}_x^{\psi}}^{st} (\f^{s},\f^t) = & - \frac{1}{2n_t} \sum_{j=1}^{n_t} \sum_{k=1}^{K} p^{st}_{k+K}(\x_j^t) \mathrm{log}(p^{st}_k(\x_j^t)) \\ & - \frac{1}{2n_t} \sum_{j=1}^{n_t} \sum_{k=1}^{K} p^{st}_k(\x_j^t) \mathrm{log}( p^{st}_{k+K}(\x_j^t)) , \label{Loss:do_conf_t}
\end{align}
where for a source example $(\x^s, y^s)$ with the label $y^s = k$, we identify its corresponding pair of the $k^{th}$ and $(k+K)^{th}$ neurons in $\f^{st}$, and use a cross-entropy between the (two-way) uniform distribution and probabilities on this neuron pair; for a target example $\x^t$, we simply use a cross-entropy between probabilities respectively on the first and second half sets of neurons in $\f^{st}$. We again fix $\omega_{y_i^s} = 1$ for closed set and open set UDA.

To provide an adversarial objective to the confusion ones (\ref{Loss:cate_conf_s}) and (\ref{Loss:do_conf_t}), we use the following domain discrimination loss
\begin{align} \label{Loss_domain_stc}
\notag \min_{\f^{s},\f^t} \text{\rm DisCRIM}_{\widehat{P}^{\psi}, \widehat{Q}_x^{\psi}}^{st} (\f^{s},\f^t) = & - \frac{1}{n_s} \sum_{i=1}^{n_s} \omega_{y_i^s} \mathrm{log}(p^{st}_{y_i^s}(\x_i^s)) \\ & - \frac{1}{n_t} \sum_{j=1}^{n_t}  \mathrm{log}(\sum_{k=1}^{K} p^{st}_{k+K}(\x_j^t)) ,
\end{align}
where $\omega_{y_i^s} = 1$ for closed set and open set UDA, and $p^{st}_k(\x)$ and $p^{st}_{k+K}(\x)$ can be viewed as the probabilities of classifying an example $\x$ of class $k$ as the source and target domains respectively.

\vspace{0.1cm}
\noindent\textbf{Overall Learning Objective} Combining (\ref{Loss:cate_conf_s}) and (\ref{Loss:do_conf_t}), and (\ref{Loss_task_sc}), (\ref{Loss_task_tc}), and (\ref{Loss_domain_stc}) gives the following objective to train SymmNets
\begin{align} \label{Loss:overall}
\notag & \min_{\psi} \text{\rm ConFUSE}_{\widehat{P}^{\psi}}^{st} (\f^{s},\f^t) + \lambda \text{\rm ConFUSE}_{\widehat{Q}_x^{\psi}}^{st} (\f^{s},\f^t), \\
& \min_{\f^s, \f^t} {\cal{L}}_{\widehat{P}^{\psi}}^s (\f^s) + {\cal{L}}_{\widehat{P}^{\psi}}^t (\f^t) + \text{\rm DisCRIM}_{\widehat{P}^{\psi}, \widehat{Q}_x^{\psi}}^{st} (\f^{s},\f^t) ,
\end{align}
where $\lambda \in [0,1]$ is a trade-off parameter to suppress less stable signals of $\text{\rm ConFUSE}_{\widehat{Q}_x^{\psi}}^{st} (\f^{s},\f^t)$ at early stages of training, since signals of $\text{\rm ConFUSE}_{\widehat{P}^{\psi}}^{st} (\f^{s},\f^t)$ from labeled source data are authentic and thus more stable. We note that the objective (\ref{Loss:overall}) of SymmNets is different from that in our preliminary work \cite{symnets}: in (\ref{Loss:overall}), the scoring disagreements between $\f^{s}$ and $\f^{t}$ are minimized \textit{explicitly} on target data, the entropy objective is achieved \textit{implicitly} in the target confusion objective (\ref{Loss:do_conf_t}), and both the class and domain supervision of source data is adopted in the domain discrimination objective (\ref{Loss_domain_stc}); in our preliminary work \cite{symnets}, the scoring disagreements between $\f^{s}$ and $\f^{t}$ are minimized \textit{implicitly} on target data, the entropy objective is adopted \textit{explicitly}, and only the domain supervision of source data is adopted in the domain discrimination objective.
%the scoring disagreements between $\f^s$ and $\f^t$ are minimized \textit{explicitly} and the entropy minimization objective is achieved \textit{inexplicitly} in the target confusion objective (\ref{Loss:do_conf_t}), and the supervision of source class is adopted in the domain discrimination objective (\ref{Loss_domain_stc}).}
we use \textbf{SymmNets-V1} and \textbf{SymmNets-V2} to report the results respectively from these two versions of our algorithms.

% {\color{red}
% 	One may be tempted to apply the objective (\ref{equ-opt2}) to the SymmNet architecture of Figure \ref{FigSymmNet}, based on the surrogate MCSD functions (\ref{equ-mcd}), (\ref{equ-kl}), or (\ref{equ-cross-entropy}), which enjoy the same advantages of efficient parametrization with the concatenated classifier $f^{st}$ and competition of probabilities in $\vect{p}^{st}$.  Ablation studies in Section \ref{SecExps} show that results of these alternatives are inferior to that of the proposed objective (\ref{Loss:overall}). cf. Table \ref{Tab:mcsd_with_concatenated_cla}. --- I suggest that we do not present these results here, since the the concatenated classifier is harmful for the McDalNets with surrogate MCSD functions (\ref{equ-mcd}), (\ref{equ-kl}), or (\ref{equ-cross-entropy}), as empirically illustrated in Table \ref{Tab:mcsd_with_concatenated_cla} and Table \ref{Tab:different_implementation_off31}.
% }

\vspace{0.1cm}
\noindent\textbf{Theoretical Connection} We discuss the conditions on which the objective (\ref{Loss:overall}) of SymmNets connects with the theoretically derived objective (\ref{equ-opt1}). We first show with the following proposition that the objective (\ref{Loss:overall}) minimizes the term in (\ref{equ-opt1}) of empirical source error defined on both the $\f^s$ and $\f^t$.
\begin{prop}
	\label{PropSymmNetConnectMcDalNet-SrcErr}
	Let $\mathcal{F}$ be a rich enough space of continuous and bounded scoring functions, with the sum-to-zero constraint $\sum_{k=1}^K f_k =0$. For $\f^s, \f^t \in \mathcal{F}$ and a fixed function $\psi$ that satisfies $\psi(\x_1)\neq \psi(\x_2)$ when $y_1\neq y_2$, $\exists \ \rho > 0$ such that, the minimizer $\f^{s*}$ of ${\cal{L}}_{\widehat{P}^{\psi}}^s (\f^s)$ in (\ref{Loss:overall}) also minimizes the term $\mathcal{E}^{(\rho)}_{\widehat{P}^{\psi}}(\f^s)$ in (\ref{equ-opt1}) of empirical source error defined on $\f^s$, and the minimizer $\f^{t*}$ of ${\cal{L}}_{\widehat{P}^{\psi}}^t (\f^t)$ in (\ref{Loss:overall}) also minimizes the term $\mathcal{E}^{(\rho)}_{\widehat{P}^{\psi}}(\f^t)$ in (\ref{equ-opt1}) of empirical source error defined on $\f^t$.
\end{prop}
We note that the assumption of continuous and bounded scoring functions in Proposition \ref{PropSymmNetConnectMcDalNet-SrcErr} could be \emph{practically} met with the function implementation of the fully-connected network layer; the assumption of $\psi(\x_1)\neq \psi(\x_2)$ when $y_1\neq y_2$ is also reasonable with properly initialized and learned $\psi$. The objective (\ref{equ-opt1}) promotes the alignment of conditional feature distributions across the two domains, by learning $\psi$ that reduces MCSD divergence. We show with the following proposition that the objective (\ref{Loss:overall}) has the same effect.
\begin{prop}
	\label{PropSymmNetConnectMcDalNet-FeaAligh}
	For $\psi$ of a function space of enough capacity and fixed functions $\f^s$ and $\f^t$ with the same range, the minimizer $\psi^{*}$ of $\text{\rm ConFUSE}_{\widehat{P}^{\psi}}^{st} (\f^{s},\f^t) + \lambda \text{\rm ConFUSE}_{\widehat{Q}_x^{\psi}}^{st} (\f^{s},\f^t)$ with the parameter $\lambda > 0$ in (\ref{Loss:overall}) zeroizes $\text{\rm MCSD}_{\widehat{Q}_x^{\psi}}^{(\rho)}(\f^s, \f^t) - \text{\rm MCSD}_{\widehat{P}_x^{\psi}}^{(\rho)}(\f^s, \f^t)$ in (\ref{equ-opt1}) of empirical \text{\rm MCSD} divergence defined on $\f^s$ and $\f^t$.
\end{prop}
We finally note that given the fixed $\psi$, minimizing the domain discrimination term $\text{\rm DisCRIM}_{\widehat{P}^{\psi}, \widehat{Q}_x^{\psi}}^{st} (\f^{s},\f^t)$ in (\ref{Loss:overall}) over $\f^s$ and $\f^t$ (together with the minimization of ${\cal{L}}_{\widehat{P}^{\psi}}^s (\f^s)$ and ${\cal{L}}_{\widehat{P}^{\psi}}^t (\f^t)$) will increase the measured divergence between $\widehat{P}_x^{\psi}$ and $\widehat{Q}_x^{\psi}$, thus providing an adversarial feature learning signal similar to the one provided by maximizing the MCSD divergence in (\ref{equ-opt1}).
Specifically, $\text{\rm MCSD}_{\widehat{P}_x^{\psi}}^{(\rho)}(\f^s, \f^t)$ is minimized by minimizing ${\cal{L}}_{\widehat{P}^{\psi}}^s (\f^s) + {\cal{L}}_{\widehat{P}^{\psi}}^t (\f^t)$ based on the Lemma \textcolor{red}{A.2} in the appendices (i.e., $\text{\rm MCSD}_{\widehat{P}_x^{\psi}}^{(\rho)}(\f^s, \f^t) \leq \mathcal{E}^{(\rho)}_{{\widehat{P}}^{\psi}}(\f^s) + \mathcal{E}^{(\rho)}_{\widehat{P}^{\psi}}(\f^t)$) and the Proposition \ref{PropSymmNetConnectMcDalNet-SrcErr}. On the other hand, minimizing $\text{\rm DisCRIM}_{\hat{P}^{\psi}, \hat{Q}_x^{\psi}}^{st} (\f^{s},\f^t)$ maximizes the output diversity of $\f^s$ and $\f^t$, thus resulting in the maximization of $\text{\rm MCSD}_{\hat{Q}_x^{\psi}}^{(\rho)}(\f^s, \f^t)$.

\section{Extensions for Partial and Open Set Domain Adaptation}
\label{SecExt2PartialOpenSet}

The theories and algorithms discussed so far apply to the \emph{closed set} setting of multi-class UDA, where a shared label space between the source and target domains is assumed. In this section, we show that simple extensions of our proposed algorithm of SymmNets can be used for either the \emph{partial} \cite{san,pada,iwanpda,transfer_example_partial} or the \emph{open set} \cite{open_set_math,open_set_bp} multi-class UDA.

\vspace{0.1cm}
\noindent\textbf{Partial Domain Adaptation} The partial setting of multi-class UDA assumes that classes of the target domain constitutes an \emph{unknown subset} of that of the source domain. As the setting suggests, a key challenge here is to identify the source instances that share the same classes with the target domain. To this end, we leverage the class-wise symmetry of neuron predictions between $\f^s$ and $\f^t$ in SymmNets, and propose a soft class weighting scheme that simply weights source instances using collective prediction evidence of target instances from $\f^t$. Specifically, we compute the following class-wise averages of prediction probabilities for target instances, and use these averaged probabilities $\{ \omega_{y_i^s}\}_{i=1}^{n_s}$ as weights for terms in the objectives (\ref{Loss_task_sc}), (\ref{Loss_task_tc}), (\ref{Loss:cate_conf_s}), and (\ref{Loss_domain_stc}) that involve labeled source data $\{(\x_i, y_i)\}_{i=1}^{n_s}$
\begin{align}
\omega_k = \frac{1}{n_t}\sum_{j=1}^{n_t} p_k^t(\vect{x}_j^t) , \ k \in \{1, \dots, K\} .
\end{align}
Such a scheme has the effect that source instances that are potentially of the classes exclusive to the target domain would be weighted down in the instance-reweighting version of the learning objective (\ref{Loss:overall}), thus promoting partial adaptation. In practice, we use more balanced class-wise weights in the early stages of training via
\begin{align}
\omega_k \leftarrow \xi \frac{\omega_k}{\max_{k\in \mathcal{Y}}\omega_k}  + (1-\xi) , \ k \in \{1, \dots, K\} .
\end{align}
where $\xi$ is a parameter set to be smaller in the early stages of training. We note that similar soft weighting schemes are also used in \cite{san,pada}.

\vspace{0.1cm}
\noindent\textbf{Open Set Domain Adaptation} The open set setting of multi-class UDA takes a step further to assume that the target domain contains certain classes exclusive to the source domain as well. Let $K^s$ and $K^t$ respectively denote the numbers of classes in the source and target domains, and $\widetilde{K}$ be the number of classes common to them, which is assumed known in \cite{open_set_math,open_set_bp}. We have $\widetilde{K} \leq K^s$ and $\widetilde{K} \leq K^t$. Extending SymmNets for the open set setting can be simply achieved by adapting its $\f^s$ and $\f^t$ to respectively have $\widetilde{K} + 1$ output neurons, where the final neuron of $\f^s$ is responsible for an aggregated prediction of the domain-specific $K^s - \widetilde{K}$ classes, and the same applies to the adapted $\f^t$. Although domain-specific classes in the source domain are treated as a single, super class, to achieve effective training of the adapted SymmNets via SGD, we still respect their overall population by sampling a $\nu \geq 1$ factor of more source examples from the super class than those from each of the $\widetilde{K}$ shared classes, when constituting training source batches. We investigate different values of $\nu$ in Section \ref{SecExps}; setting $\nu = 6$ consistently gives good results. Since target instances are unlabeled, we simply sample them randomly to constitute training target batches.

\color {black}
\section{Experiments}
\label{SecExps}

In this section, we conduct experiments to investigate the practice of our introduced theory and algorithms. We compare different algorithms or implementations of McDalNets, including these based on the absolute margin-based equivalent of \cite{ben2010theory} and variant of \cite{mdd},
%its degenerate versions,
and our proposed SymmNets-V1 \cite{symnets} and SymmNets-V2 under the closed set setting of multi-class UDA. We also evaluate the efficacy of our SymmNets for partial and open set settings. These experiments are conducted on seven benchmark datasets by implementing algorithms on three backbone networks, which are specified shortly. Additional experiments, results, and analyses are provided in the appendices. %{\color{black} Codes are available at \url{https://github.com/YBZh/MultiClassDA}.}

\begin{table}[htb]
	\begin{center}
		\caption{Summary of datasets. ``C'', ``P'', and ``O'' indicate the respective settings of the closed set, partial, and open set domain adaptation.}
		\label{Tab:datasets}
		\begin{tabular}{l|cccc}
			\hline
			\multirow{2}{*}{Dataset}        & Involved     &      No. of       & No. of   & No. of  \\
			& Tasks        &  Domains       & Classes     &    Samples       \\
			\hline
			ImageCLEF-DA \cite{ImageCLEFDA}  &      C       & 3             & 12          &     1,800          \\
			Office-31 \cite{office_31}       &   C+P+O      & 3             & 31          &     4,110       \\
			Office-Home \cite{office_home}   &    C+P       & 4             & 65          &     15,500        \\
			Digits \cite{mnist,usps,svhn}   &     C         & 3             & 10          &     172.5K    \\
			%Syn2Real \cite{visda,syn2real} &    C+O         & 2             & 12          &     280K      \\
			Syn2Real\cite{syn2real}          &    O         & 2              & 13            &     248K         \\
			VisDA-2017 \cite{visda,syn2real} &    C         & 2             & 12          &     280K \\ % 207.4K      \\
			DomainNet \cite{peng2018moment} & C  & 6              & 345           &    586.6K \\
			\hline
		\end{tabular}
	\end{center}
\end{table}

\begin{table*}[htb]
	\centering
	\caption{\textcolor{black}{Accuracy (\%) of different instantiations of McDalNets on the datasets of Office-31 \cite{office_31}, ImageCLEF-DA \cite{ImageCLEFDA}, Office-Home \cite{office_home}, Digits \cite{mnist,svhn,usps}, VisDA-2017 \cite{visda}, and DomainNet \cite{peng2018moment} under the setting of closed set UDA. Each accuracy reported here is a \textit{result averaged over individual tasks of a specific dataset}. All the results of individual tasks for the respective datasets are given in the appendices.}}
	\label{Tab:different_implementation}
	\begin{tabular}{l|c|c|c|c|c|c}
		\hline
		Methods                                       & {\scriptsize Office-31}  & {\scriptsize  ImageCLEF}  & {\scriptsize  Office-Home}   & {\scriptsize Digits} & {\scriptsize VisDA-2017}  & {\scriptsize \textcolor{black}{DomainNet}} \\
		\hline
		Source Only                                   & 81.8       & 82.7 & 58.9  & 70.5  &  41.8 & \textcolor{black}{24.4} \\
		\hline
		{\scriptsize McDalNets based on the following surrogates of $\widehat{\text{\rm MCSD}}$  (\ref{EqnMCSDDegen2HDH}) and  $\widetilde{\text{\rm MCSD}}$ (\ref{EqnMCSDDegen2MDD})}             &                     &               &                 &                & & \\
		\quad DANN \cite{reverse_grad,dann} (\ref{equ-dann-form-dann})  & 82.8  & 84.2& 60.0 &  72.5   & 58.4  & \textcolor{black}{27.1}  \\
		\quad MDD \cite{mdd} variant (\ref{equ-mdd-form-long}) & 84.5  &86.7 & 61.1 &    not converge         &  not converge & \textcolor{black}{26.5}  \\ %$^*$67.0  \\
		\hline
		{\scriptsize McDalNets based on the following surrogates of MCSD (\ref{EqnMCSD})} &                     &               &                 &                &  \\
		\quad $L_1/\text{\rm MCD}$\cite{mcd} (\ref{equ-mcd})  &84.7 &87.0& 62.0 &  90.6  & 70.4 & \textcolor{black}{27.7} \\
		\quad KL (\ref{equ-kl})                                &84.6 &87.6& 63.3 &  82.9         & 69.0 & \textcolor{black}{27.6}  \\
		\quad CE (\ref{equ-cross-entropy})                     &85.3 & 87.8&64.0 &  94.9  & 70.5 & \textbf{\textcolor{black}{27.9}}  \\
		\hline
		SymmNets-V2 (\ref{Loss:overall})                        & \textbf{89.1} & \textbf{89.7}& \textbf{68.1} &   \textbf{96.0}  & \textbf{71.3} & \textbf{\textcolor{black}{27.9}}  \\
		\hline
	\end{tabular}
\end{table*}

%\subsection{Datasets and Implementations } \label{SubSecDatasetImp}
\vspace{0.1cm}
\noindent\textbf{Datasets} We use the benchmark datasets summarized in Table \ref{Tab:datasets} for our evaluation. In the closed set UDA, we follow standard protocols \cite{reverse_grad, dan} for the datasets of Office-31 \cite{office_31}, Office-Home \cite{office_home}, ImageCLEF-DA \cite{ImageCLEFDA}, and VisDA-2017 \cite{visda}: all labeled source and target samples are used for training; for the Digits datasets of \cite{mnist,svhn,usps}, we follow the protocols in \cite{adr}; we follow the standard split for the DomainNet dataset \cite{peng2018moment}. In partial UDA, all labeled source samples construct the source domain, and the target domain is constructed following the protocols of \cite{san,pada}: for Office-31 \cite{office_31}, the samples of ten classes shared by Office-31 \cite{office_31} and Caltech-256 \cite{caltech} are selected as the target domain; for Office-Home \cite{office_home}, we choose (in alphabetic order) the first 25 classes as target classes and select all samples of these 25 classes as the target domain. In open set UDA, the samples of ten classes shared by Office-31 \cite{office_31} and Caltech-256 \cite{caltech} are selected as shared classes across domains. In alphabetical order, samples of Class 21$\sim$Class 31 and Class 11$\sim$Class 20 are used as unknown samples in the target and source domains, respectively; we follow the standard split for the benchmark dataset of Syn2Real \cite{syn2real}.

\noindent\textbf{Implementations Details} All our methods are implemented using the PyTorch library. For the close set and partial settings of UDA, we adopt a ResNet pre-trained on ImageNet \cite{imagenet}, after removing the last fully connected (FC) layer, as the feature extractor $\psi$. We fine-tune the feature extractor $\psi$ and train a classifier $\mathbf{f}^{st}$ from scratch with the backpropagation algorithm. The learning rate for the newly added layers is set as $10$ times of that of the pre-trained layers. All parameters are updated by SGD with a momentum of $0.9$. We follow \cite{reverse_grad} to employ the annealing strategy of learning rate and the progressive strategy of $\lambda$: the learning rate is adjusted by $\eta_p = \frac{\eta_0}{(1+\alpha p)^{\beta}}$, where $p$ is the progress of training epochs linearly changing from $0$ to $1$, $\eta_0 = 0.01$, $\alpha = 10$, and $\beta = 0.75$, which are optimized to promote convergence and low errors on source samples; $\lambda$ is gradually changed from $0$ to $1$ by $\lambda_p = \frac{2}{1+\mathrm{exp}(-\gamma \cdot p)} - 1$, where $\gamma$ is set to $10$ in all experiments. We empirically set $\xi = \lambda$ in all experiments. Our classification results are obtained from the target task classifier $\mathbf{f}^{t}$ unless otherwise specified, and the comparison between the performance of the source and target task classifiers is illustrated in Figure \ref{fig:convergence}.
For the open set UDA, we follow \cite{syn2real} to replace the very top FC layer of an ImageNet pre-trained ResNet with three FC layers  powered by the batch normalization \cite{bn} and Leaky ReLU activation; the feature extractor $\psi$ is defined by pre-trained layers together with first two of the three added FC layers, and the last FC layer is the classifier $\mathbf{f}^{st}$. We freeze parameters of pre-trained layers and update those of the added FC layers with a learning rate of $0.001$, following \cite{open_set_bp}. We also follow \cite{open_set_math, open_set_bp} to report OS as the accuracy averaged over all classes and OS$^*$ as that averaged over the domain shared classes only. We additionally implement our methods based on the AlexNet \cite{alexnet} and modified LeNet \cite{lecun1998gradient,adda} to testify its generalization to different architectures. Please refer to the appendices for more details.
For a fair comparison, results of other methods are either directly reported from their original papers if available or quoted from \cite{cada}, \cite{pada} and \cite{open_set_bp,syn2real} for the closed set, partial and open set settings of UDA, respectively.

\subsection{Analysis on Different Instantiations of McDalNets}

In this section, we investigate different instantiations of McDalNets that are achieved by using surrogate functions (\ref{equ-mcd}), (\ref{equ-kl}), or (\ref{equ-cross-entropy}) to replace the MCSD terms in the general objective (\ref{equ-opt1}), by comparing with the counterparts based on surrogate functions (\ref{equ-dann-form-dann}) or (\ref{equ-mdd-form-long}) of scalar-valued versions of (\ref{EqnMCSDDegen2HDH}) or (\ref{EqnMCSDDegen2MDD}).
%degenerate MCSD (\ref{EqnMCSDDegen2HDH}) or (\ref{EqnMCSDDegen2MDD}).
These experiments are conducted on the datasets of Office-31 \cite{office_31}, ImageCLEF-DA \cite{ImageCLEFDA}, Office-Home \cite{office_home}, Digits \cite{mnist,svhn,usps}, VisDA-2017 \cite{visda}, and DomainNet \cite{peng2018moment} under the setting of closed set UDA. In practice, we downweight the MCSD divergence in (\ref{equ-opt1}) with respect to the feature extractor $\psi$ at early stages of training, resulting in the following objective
\begin{equation}
\begin{gathered} \label{equ-opt1-parctical}
\min\limits_{\f,\psi}\  \mathcal{E}^{(\rho)}_{\widehat{P}^{\psi}}(\f) + \zeta [\text{\rm MCSD}_{\widehat{Q}_x^{\psi}}^{(\rho)}(\f', \f'') - \text{\rm MCSD}_{\widehat{P}_x^{\psi}}^{(\rho)}(\f', \f'') ], \\
\max\limits_{\f',\f''}\  [\text{\rm MCSD}_{\widehat{Q}_x^{\psi}}^{(\rho)}(\f', \f'') - \text{\rm MCSD}_{\widehat{P}_x^{\psi}}^{(\rho)}(\f', \f'')] ,
\end{gathered}
\end{equation}
where we empirically set $\zeta = \lambda$, which is described in the beginning of Section \ref{SecExps}. The weight $\zeta$ is similarly applied to objectives based on surrogate MCSD functions. We adopt the gradient reversal layer to implement the adversarial objective.
Therefore, the instantiation of McDalNets with the surrogate function (\ref{equ-dann-form-dann}) of the scalar-valued $\widehat{\text{\rm MCSD}}$ (\ref{EqnMCSDDegen2HDH}) coincides with that of DANN \cite{reverse_grad,dann}.
%Therefore, the instantiation of McDalNets with surrogate function (\ref{equ-dann-form-dann}) of degenerate MCSD (\ref{EqnMCSDDegen2HDH}) coincides with that of DANN \cite{reverse_grad,dann}.
The implementation details of other settings are the same as these described in the beginning of Section \ref{SecExps}, except that we train three classifiers $\mathbf{f}$, $\mathbf{f}'$, and $\mathbf{f}''$ from scratch and the classification results are obtained from the task classifier $\mathbf{f}$. For ease of optimization, we also train auxiliary classifiers $\mathbf{f}'$ and $\mathbf{f}''$ using a standard log loss over labeled source data. The ``Source Only'' indeed gives a lower bound, where we fine-tune a model on the source data only.

Results in Table \ref{Tab:different_implementation} show that all instantiations of McDalNets improve over the baseline of ``Source Only'', certifying the efficacy of the MCSD divergence in the domain discrepancy minimization. The McDalNets based on MCSD surrogates (\ref{equ-mcd}), (\ref{equ-kl}), and (\ref{equ-cross-entropy}) generally achieve better results than those based on surrogates (\ref{equ-dann-form-dann}) and (\ref{equ-mdd-form-long}) of
%the degenerate MCSD (\ref{EqnMCSDDegen2HDH}) and (\ref{EqnMCSDDegen2MDD}),
the scalar-valued $\widehat{\text{\rm MCSD}}$ (\ref{EqnMCSDDegen2HDH}) and $\widetilde{\text{\rm MCSD}}$ (\ref{EqnMCSDDegen2MDD}),
testifying the advantages of characterizing finer details of the scoring disagreement in multi-class UDA. McDalNets based on the MCSD surrogate of CE (\ref{equ-cross-entropy}) generally achieves better results than those based on the MCSD surrogates of $L_1/\text{\rm MCD \cite{mcd}}$ (\ref{equ-mcd}) and KL (\ref{equ-kl}), probably due to the mechanism where the CE-based surrogate (\ref{equ-cross-entropy}) also makes predictions of lower entropy; further explanation via illustration is given in the appendices. Among all algorithms, SymmNets-V2 proposed in the present paper achieves the best results across all tasks, confirming its efficacy in multi-class UDA.

\begin{table*}[htb]
	\centering
	\caption{Ablation experiments on components of SymmNets-V2 using the datasets of Office-31 \cite{office_31} and VisDA-2017 \cite{visda} under the setting of closed set UDA. All methods are based on models adapted from a 50-layer ResNet. Please refer to the main text for specifics of these methods.
	}
	\label{Tab:ablation_study}
	\begin{tabular}{lcccc|c}
		\hline
		Methods                                       & A $\to$ W            & A $\to$ D    & D $\to$ A      & W $\to$ A     & Synthetic $\to$ Real  \\
		\hline
		%\hline
		SymmNets-V2 (w/o ${\cal{L}}_{\widehat{P}^{\psi}}^t$) & 71.0$\pm$0.8       & 74.5$\pm$0.9 & 63.3$\pm$0.2    & 62.8$\pm$0.1 & 41.9 \\
		SymmNets-V2 (w/o adversarial training)           & 78.3$\pm$0.3       & 83.3$\pm$0.2 & 64.6$\pm$0.5    & 66.6$\pm$0.1   & 41.6 \\
		SymmNets-V2                                      & \textbf{94.2}$\pm$0.1 & \textbf{93.5}$\pm$0.3  & \textbf{74.4}$\pm$0.1  & \textbf{73.4}$\pm$0.2   & \textbf{71.3}  \\
		\hline
	\end{tabular}
\end{table*}

We also plot convergence curves for different instantiations of McDalNets in Figure \ref{fig:convergence}, where we observe that those based on MCSD surrogates of $L_1/\text{\rm MCD \cite{mcd}}$ (\ref{equ-mcd}), KL (\ref{equ-kl}), and CE (\ref{equ-cross-entropy}) converge generally smoother than those based on surrogates (\ref{equ-dann-form-dann}) and (\ref{equ-mdd-form-long}) of the scalar-valued  $\widehat{\text{\rm MCSD}}$ and $\widetilde{\text{\rm MCSD}}$.
%degenerate MCSD surrogates (\ref{equ-mdd-form-long}) and (\ref{equ-dann-form-dann}).
It could be attributed to the in-built function property of MCSD (\ref{EqnMCSD}), as illustrated in Figure \ref{Fig:f-f-disp}. We particularly note that McDalNets based on the scalar-valued $\widetilde{\text{\rm MCSD}}$ surrogate (\ref{equ-mdd-form-long})
%degenerate MCSD surrogate (\ref{equ-mdd-form-long})
does not converge on the datasets of Digits and VisDA-2017. In comparison, SymmNets-V2 achieves the lowest classification error and the smoothest convergence.

\begin{figure}[h!]
	\begin{center}
		\includegraphics[width=1.0\linewidth] {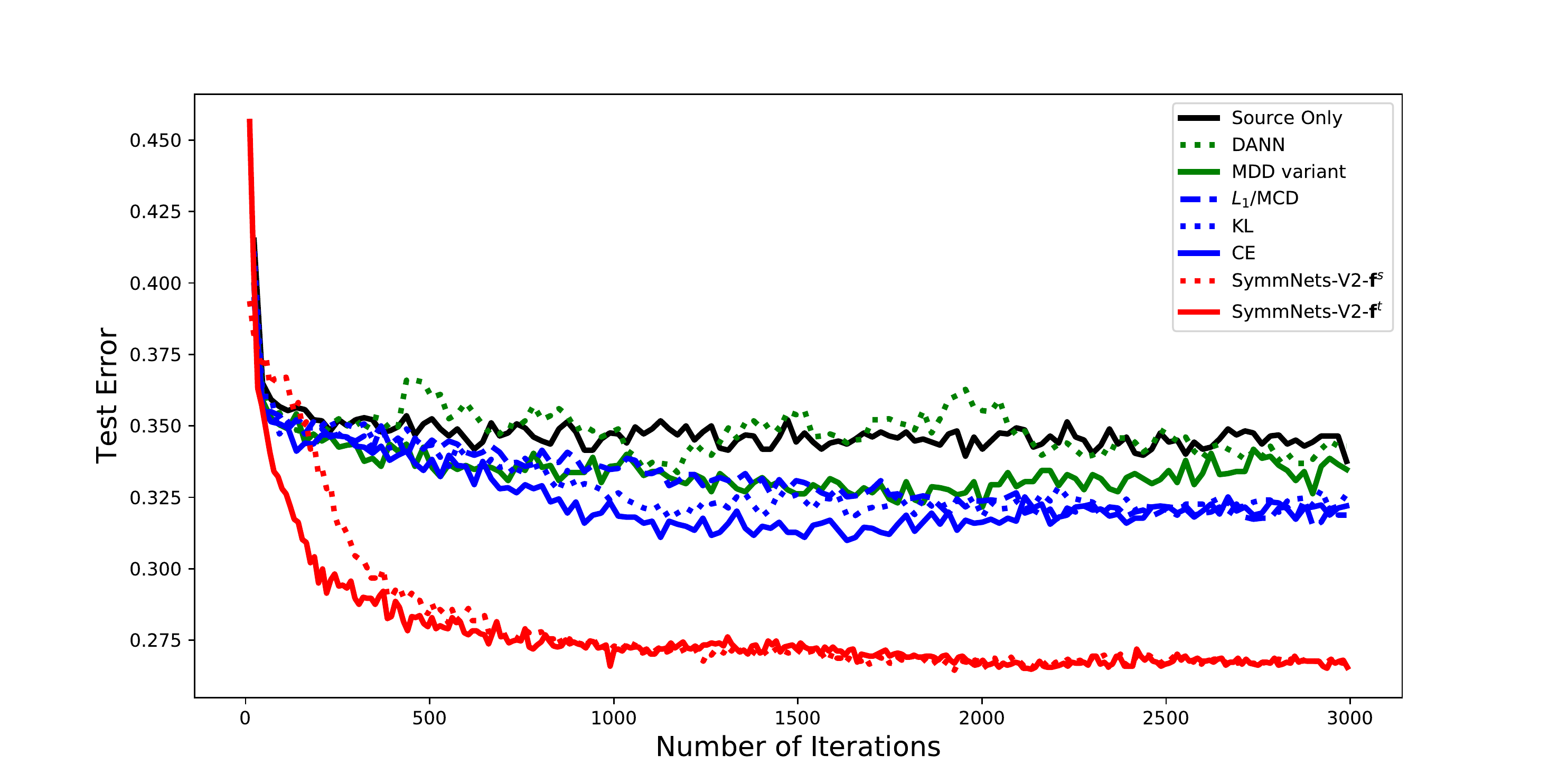}
	\end{center}
	\caption{ Convergence plottings on the adaptation task \textbf{W} $\to$ \textbf{A} of the Office-31 \cite{office_31} by the Source Only, McDalNets based on the $\widehat{\text{\rm MCSD}}$ surrogate (\ref{equ-mdd-form-long}) (variant of MDD \cite{mdd}) and $\widetilde{\text{\rm MCSD}}$ surrogate (\ref{equ-dann-form-dann}) (DANN \cite{reverse_grad,dann}), McDalNets based on the MCSD surrogates $L_1$ (\ref{equ-mcd}) (MCD \cite{mcd}), KL (\ref{equ-kl}), and CE (\ref{equ-cross-entropy}), and SymmNets-V2. SymmNets-V2-$\mathbf{f}^s$ and SymmNets-V2-$\mathbf{f}^t$ represent the results obtained from the source classifier $\mathbf{f}^s$ and target classifier $\mathbf{f}^t$, respectively.}
	\label{fig:convergence}
\end{figure}

In this section, we investigate the effects of different components in our proposed SymmNets-V2 by conducting ablation experiments on the datasets of Office-31 \cite{office_31} and VisDA-2017 \cite{visda} under the setting of closed set UDA, where networks are adapted from a 50-layer ResNet. To investigate how the cross-domain training term ${\cal{L}}_{\widehat{P}^{\psi}}^t$ (\ref{Loss_task_tc}) contributes to a better adaptation in our overall adversarial learning objective (\ref{Loss:overall}), we remove it from  (\ref{Loss:overall}) and denote the method as ``SymmNets-V2 (w/o ${\cal{L}}_{\widehat{P}^{\psi}}^t$)''. To evaluate the efficacy of our adversarial training, we remove the domain discrimination loss $\text{\rm DisCRIM}_{\widehat{P}^{\psi}, \widehat{Q}_x^{\psi}}^{st}$ (\ref{Loss_domain_stc}) and the domain confusion loss of target data $\text{\rm ConFUSE}_{\widehat{Q}_x^{\psi}}^{st}$ (\ref{Loss:do_conf_t}) from the overall objective (\ref{Loss:overall}), and use the following degenerate form to replace the domain confusion loss of source data $\text{\rm ConFUSE}_{\widehat{P}^{\psi}}^{st}$ (\ref{Loss:cate_conf_s})
\begin{align}
\min_{\mathnormal{\psi}} - \frac{1}{2n_s} \sum_{i=1}^{n_s} \omega_{y_i^s} \mathrm{log}(p^{st}_{y_i^s}(\x_i^s)) - \frac{1}{2n_s} \sum_{i=1}^{n_s} \omega_{y_i^s} \mathrm{log}(p^{st}_{y_i^s+K}(\x_i^s)) ;
\end{align}
we denote this method as ``SymmNets-V2 (w/o adversarial training) ''.
Note that classification results for SymmNets (w/o ${\cal{L}}_{\widehat{P}^{\psi}}^t$) are obtained from the source task classifier $\mathbf{f}^{s}$ due to the inexistence of the direct supervision signals for the target task classifier $\mathbf{f}^{t}$. Results in Table \ref{Tab:ablation_study} show that SymmNets-V2 outperforms ``SymmNets-V2 (w/o adversarial training)'' by a large margin, verifying the efficacy of the discrepancy minimization via our proposed adversarial training. The performance slump of ``SymmNets-V2 (w/o ${\cal{L}}_{\widehat{P}^{\psi}}^t$)'' manifests the importance of the cross-domain training term ${\cal{L}}_{\widehat{P}^{\psi}}^t$ (\ref{Loss_task_tc}) for learning a well-performed target task classifier in adversarial training.

\begin{figure}[h!]
	\begin{center}
		\includegraphics[width=0.9\linewidth] {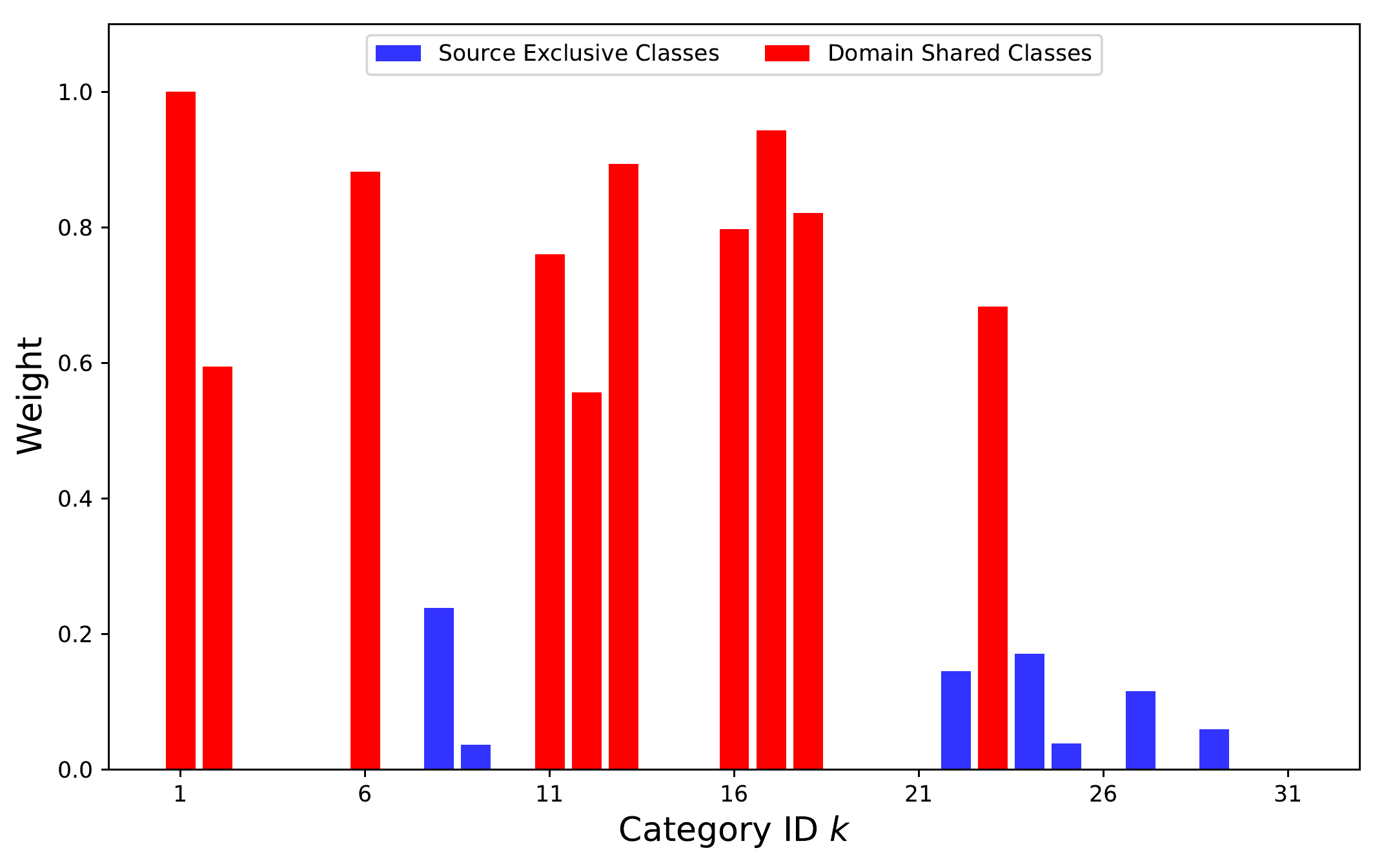}
	\end{center}
	\caption{Histograms of class weight $\omega_k$ learned by SymmNets-V2 (with active $\omega_k$) on the task of \textbf{A} $\to$ \textbf{W} under the setting of partial UDA. Model is adapted from a 50-layer ResNet.}
	\label{fig:partial_weight_distribution}
\end{figure}

\begin{figure}[h!]
	\begin{center}
		\includegraphics[width=0.9\linewidth] {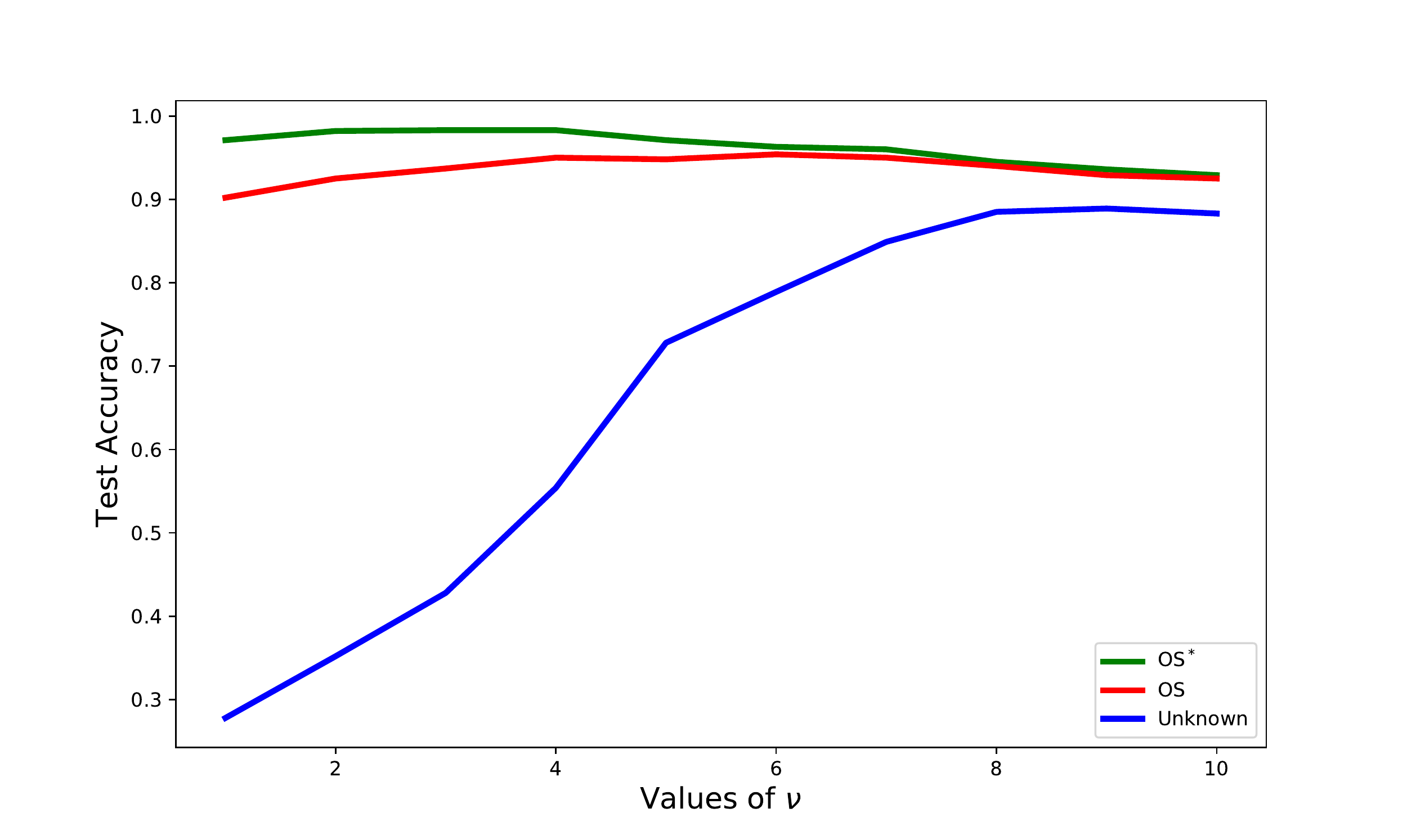}
	\end{center}
	\caption{Curve plottings for test accuracy of the unknown class (Unknown) and the mean accuracy over all classes (OS) and domain-shared classes (OS$^*$), when setting different values of $\nu$ in open set UDA. The results are reported on the \textbf{A} $\to$ \textbf{W} task of Office-31 dataset \cite{office_31} based on the SymmNets-V2 adapted from a 50-layer ResNet.}
	\label{fig:open_set_eta_change}
\end{figure}

\subsection{Ablation Studies of SymmNets}

\begin{figure*}[h!]
	\begin{minipage}[t]{0.13\linewidth}
		\vspace*{-2cm}
		DANN \cite{reverse_grad,dann}	
	\end{minipage}
	\hfill
	\subfigure[Close Set UDA]{
		\begin{minipage}[t]{0.28\linewidth}
			\centering
			\includegraphics[width=0.77\linewidth] {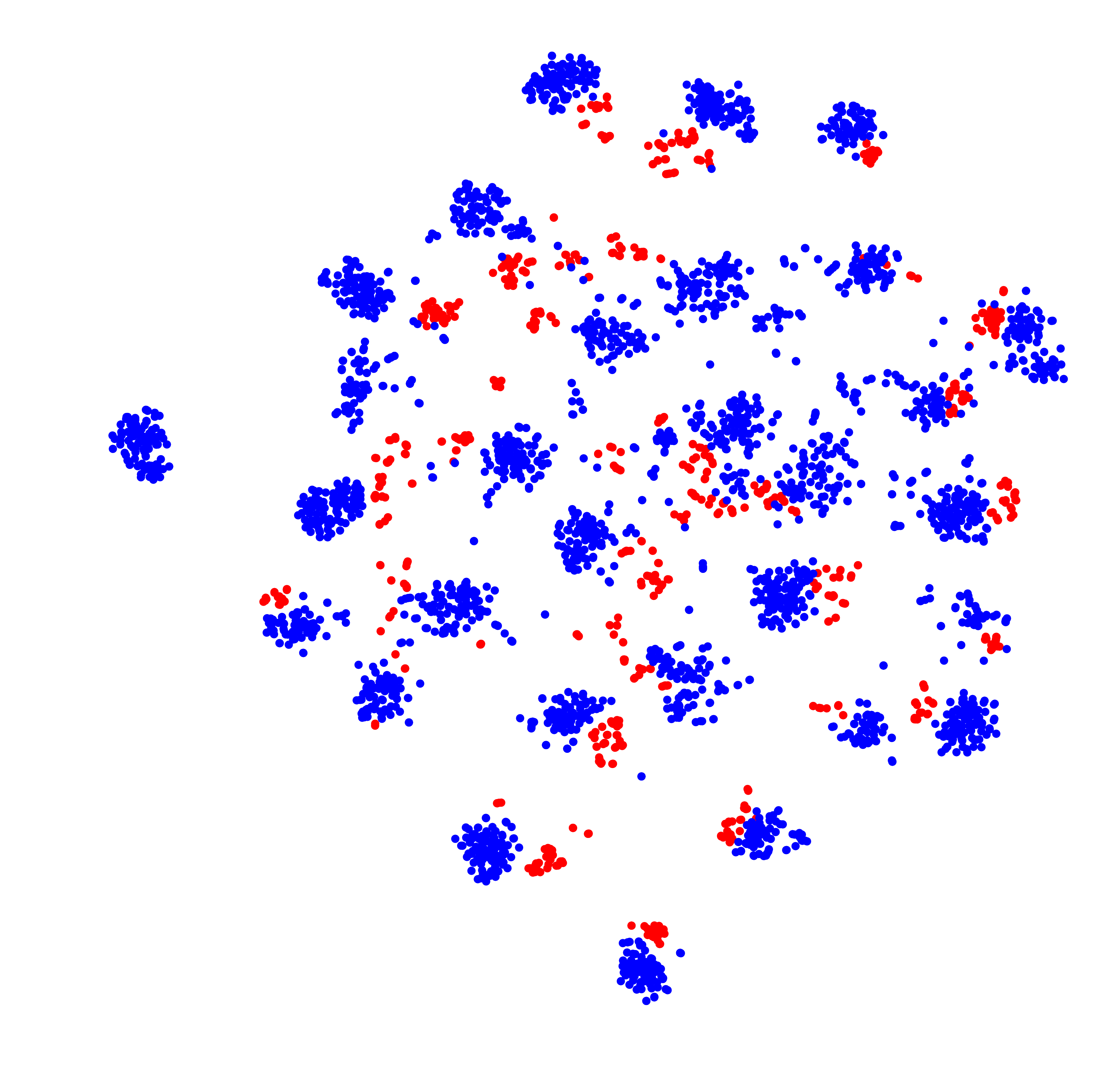}
	\end{minipage}}
	\hfill
	\subfigure[Partial UDA]{
		\begin{minipage}[t]{0.28\linewidth}
			\centering
			\includegraphics[width=0.77\linewidth] {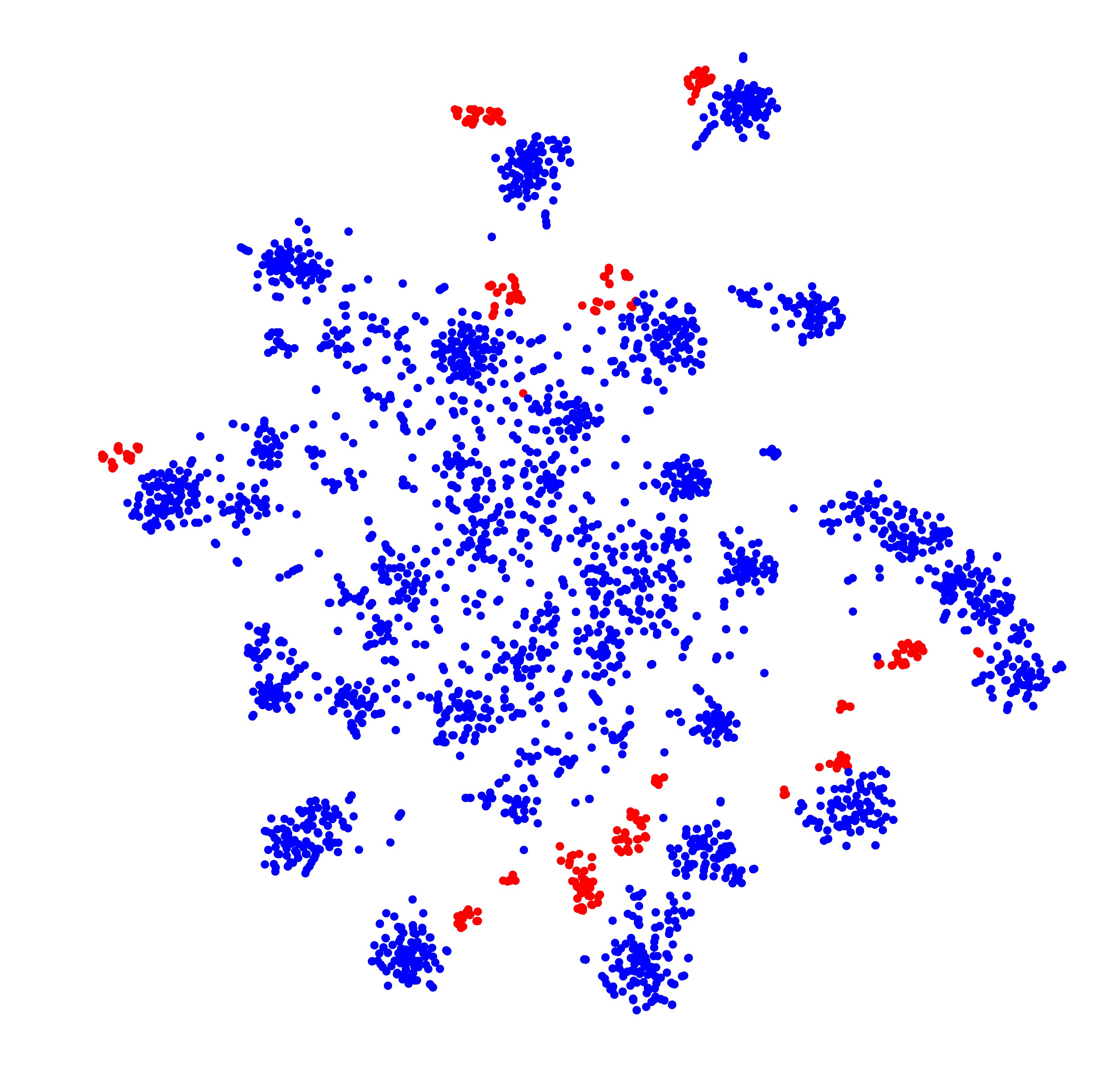}
	\end{minipage}}
	\hfill
	\subfigure[Open Set UDA]{
		\begin{minipage}[t]{0.28\linewidth}
			\centering
			\includegraphics[width=0.77\linewidth] {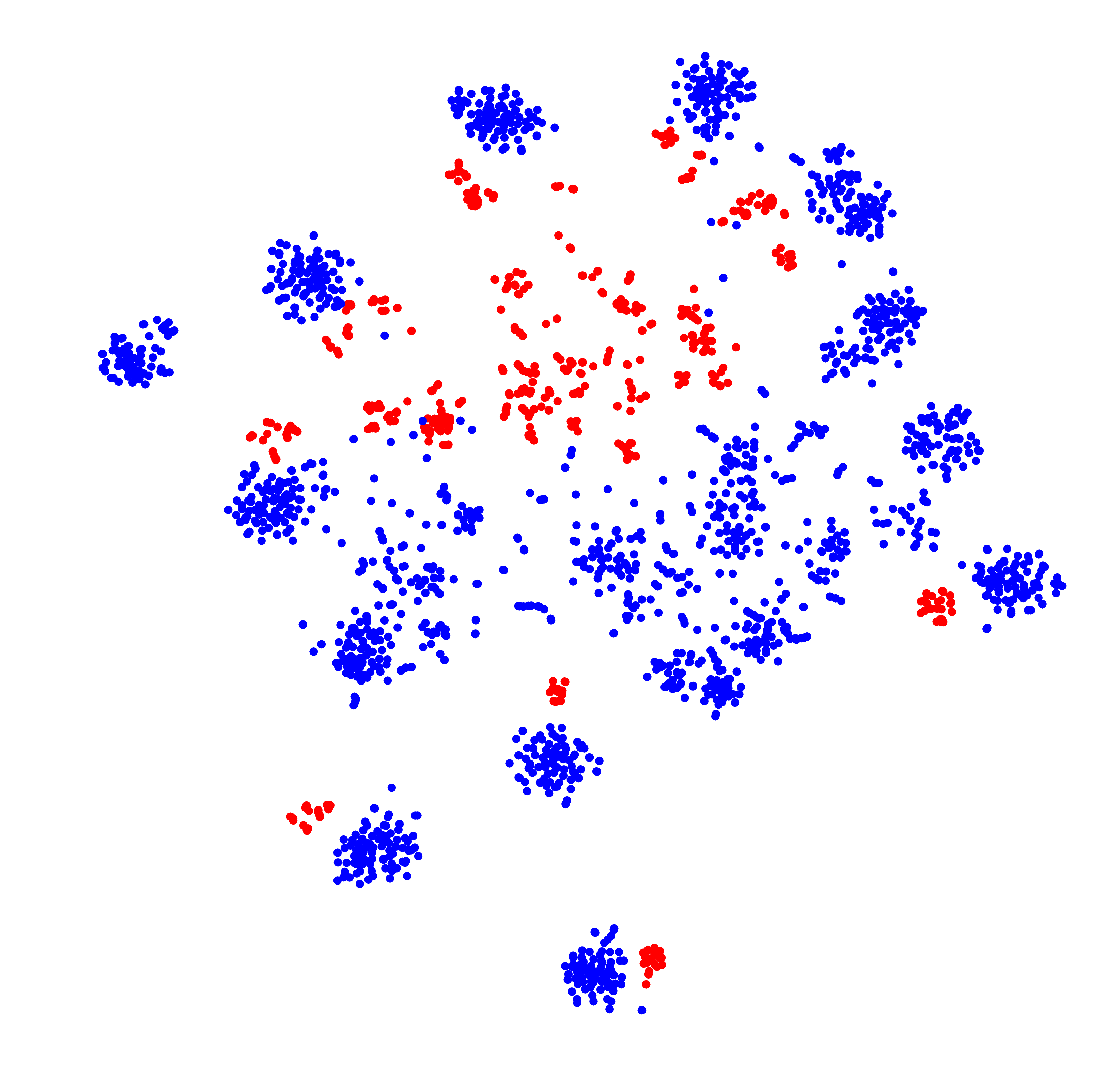}
	\end{minipage}}
	\\
	\begin{minipage}[t]{0.13\linewidth}
		\vspace*{-2cm}
		SymmNets-V2	
	\end{minipage}
	\hfill
	\subfigure[Close Set UDA]{
		\begin{minipage}[t]{0.28\linewidth}
			\centering
			\includegraphics[width=0.77\linewidth] {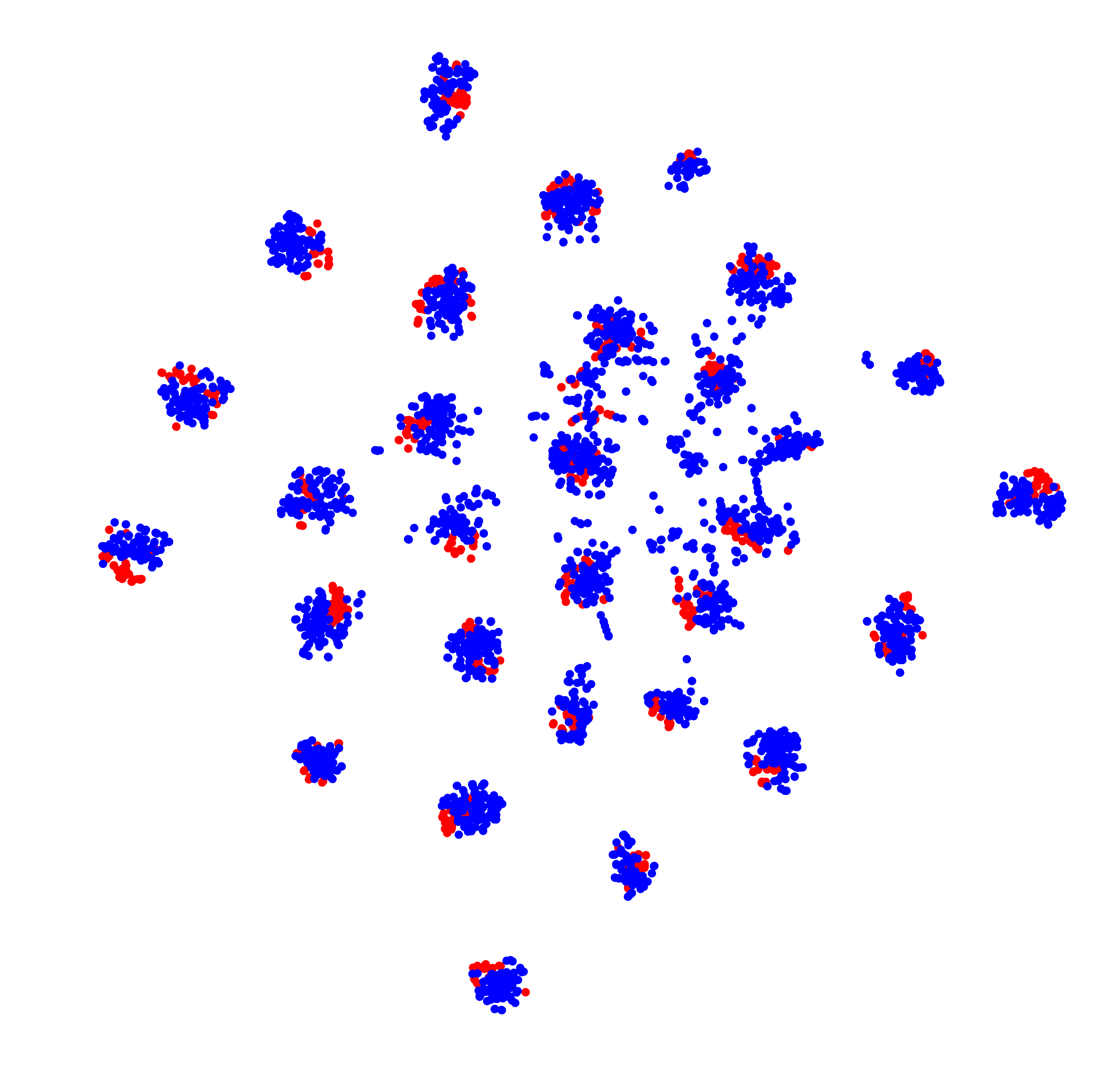}
	\end{minipage}}
	\hfill
	\subfigure[Partial UDA]{
		\begin{minipage}[t]{0.28\linewidth}
			\centering
			\includegraphics[width=0.77\linewidth] {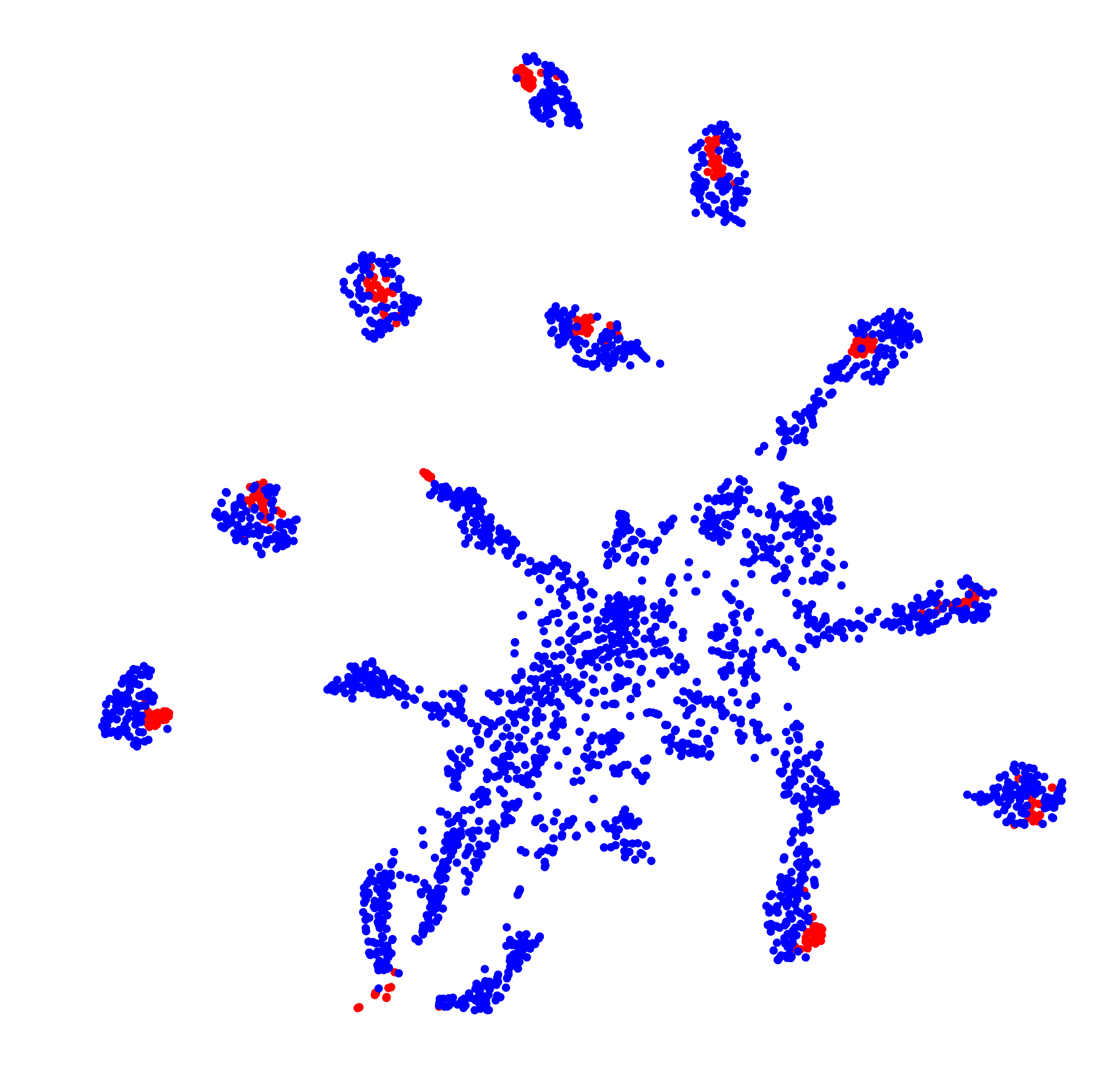}
	\end{minipage}}
	\hfill
	\subfigure[Open Set UDA]{
		\begin{minipage}[t]{0.28\linewidth}
			\centering
			\includegraphics[width=0.77\linewidth] {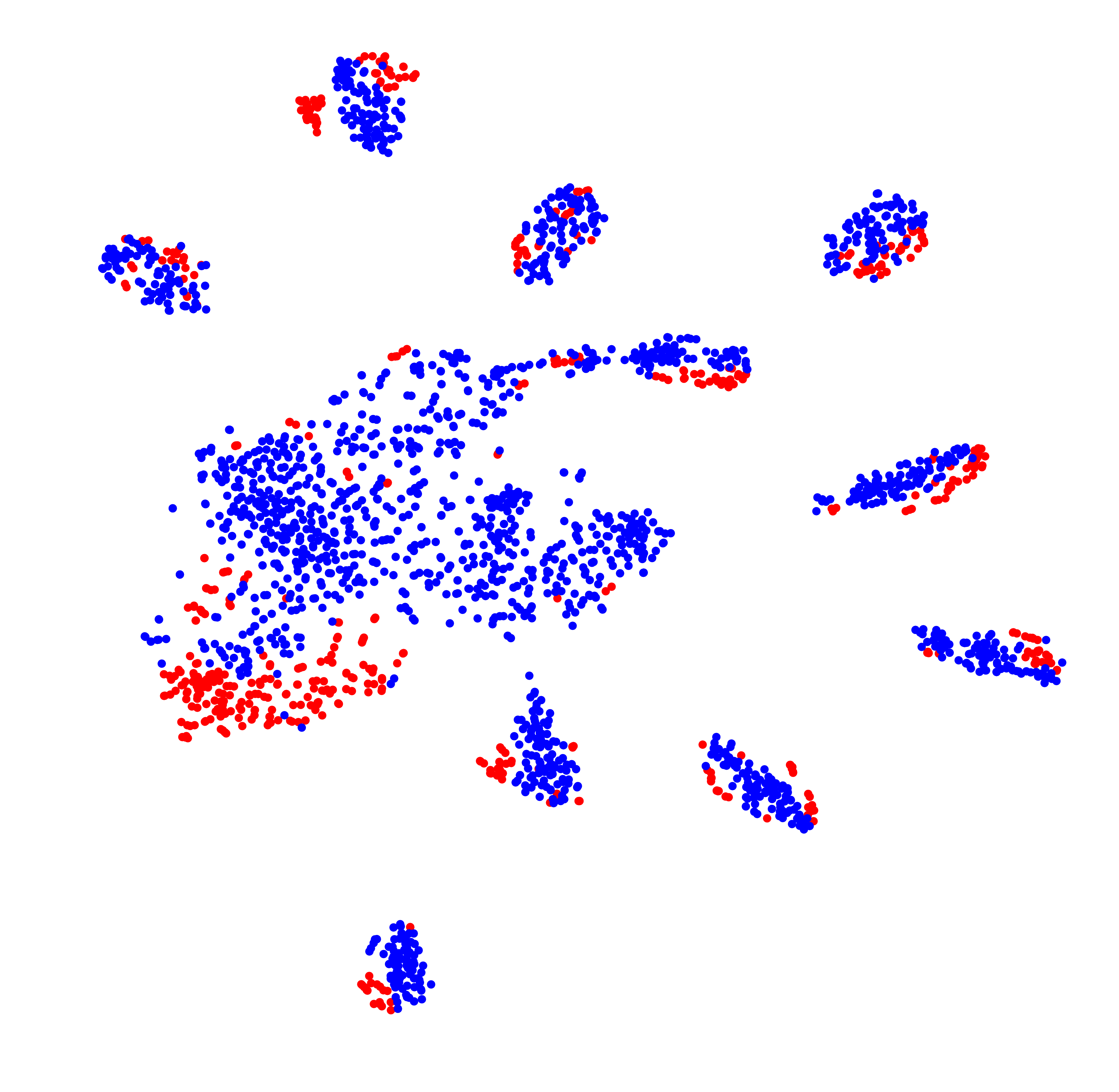}
	\end{minipage}}
	
	\caption{The t-SNE visualization of feature representations learned by DANN (top row) and SymmNets-V2 (bottom row) under settings of the closed set, partial, and open set UDA. Blue and red points are the respective samples from the source domain \textbf{A} and target domain \textbf{W}. For partial UDA, we illustrate the feature representations learned by SymmNets-V2 (With active $\omega_k$), where we focus on domain-shared classes, and leave the source classes exclusive to the target domain as an indistinguishable cluster via the soft class weighting scheme, as discussed in Section \ref{SecExt2PartialOpenSet}. A visualization with class label information is given in the appendices.}
	\label{Fig:sne}
\end{figure*}

\vspace{0.1cm}
\noindent \textbf{Soft Class Weighting Scheme in Partial UDA}
To investigate the efficacy of the soft class weighting scheme, we activate it with the strategy described in Section \ref{SecExt2PartialOpenSet}, giving rise to the method of ``SymmNets-V2 (with active $\omega_k$)''. Tables \ref{Tab:partial_office31} and \ref{Tab:partial_Office-Home} show that results of SymmNets-V2 (with active $\omega_k$) improve over those of SymmNets-V2, empirically verifying its effectiveness.
To have an intuitive understanding of what has happened, we illustrate in Figure \ref{fig:partial_weight_distribution} the learned weight of each source class on the adaptation task of $\textbf{A}$ $\to \textbf{W}$. SymmNets-V2 (with active $\omega_k$) assigns much larger weights to domain-shared classes than the classes exclusive to the source domain, thus suppressing misalignment across the two domains.

\begin{table}[htp]
	\centering
	\caption{Accuracy (\%) on the VisDA-2017 dataset \cite{visda} for closed set UDA. All comparative methods are based on a 101-layer ResNet except the MDD and CDAN+E, which are based on a 50-layer ResNet  }
	\label{Tab:visda}
	\begin{tabular}{lc}
		\hline
		Methods         &                Synthetic $\to$ Real  \\
		\hline
		Source Only \cite{resnet}       &              52.4              \\
		DANN \cite{dann}  &              57.4              \\
		CDAN+E \cite{cada}&              70.0              \\
		MCD \cite{mcd}  &                71.9              \\
		ADR \cite{adr}  &                73.5             \\
		MDD \cite{mdd}  &  74.6 \\
		SWD \cite{swd}  &                76.4 \\
		\hline
		\textbf{SymmNets-V1}  \cite{symnets}        &               72.1             \\
		\textbf{SymmNets-V2}                        &               \textbf{76.8}             \\
		\hline
		\hline
		TPN \cite{tpn}  &                80.4 \\
		CAN \cite{can}  &                \textbf{87.2}     \\
		\hline
		\textbf{SymmNets-V2-SC}                &               86.0             \\
		\hline
		
		\hline
	\end{tabular}
\end{table}

\begin{table*}[htb]
	\centering
	\caption{Accuracy (\%) on the Office-31 dataset \cite{office_31} for closed set UDA. Results are based on models adapted from a 50-layer ResNet.}
	\label{Tab:office31}
	\begin{tabular}{lccccccc}
		\hline
		Methods                          & A $\to$ W     & D $\to$ W     & W $\to$ D         & A $\to$ D     & D $\to$ A      & W $\to$ A     & Avg  \\
		\hline
		Source Only\cite{resnet}           & 68.4$\pm$0.2  & 96.7$\pm$0.1  & 99.3$\pm$0.1     & 68.9$\pm$0.2  & 62.5$\pm$0.3   & 60.7$\pm$0.3   & 76.1 \\
		DAN \cite{dan}                  & 80.5$\pm$0.4  & 97.1$\pm$0.2  & 99.6$\pm$0.1     & 78.6$\pm$0.2  & 63.6$\pm$0.3   & 62.8$\pm$0.2   & 80.4 \\
		RTN \cite{rtn}                  & 84.5$\pm$0.2  & 96.8$\pm$0.1  & 99.4$\pm$0.1     & 77.5$\pm$0.3  & 66.2$\pm$0.2   & 64.8$\pm$0.3   & 81.6 \\
		DANN \cite{reverse_grad,dann}     & 82.0$\pm$0.4  & 96.9$\pm$0.2  & 99.1$\pm$0.1     & 79.7$\pm$0.4  & 68.2$\pm$0.4   & 67.4$\pm$0.5   & 82.2 \\
		ADDA \cite{adda}                & 86.2$\pm$0.5  & 96.2$\pm$0.3  & 98.4$\pm$0.3     & 77.8$\pm$0.3  & 69.5$\pm$0.4   & 68.9$\pm$0.5   & 82.9 \\
		JAN-A\cite{jan}                  & 86.0$\pm$0.4  & 96.7$\pm$0.3  & 99.7$\pm$0.1     & 85.1$\pm$0.4  & 69.2$\pm$0.3   & 70.7$\pm$0.5   & 84.6 \\
		MADA \cite{mada}                & 90.0$\pm$0.1  & 97.4$\pm$0.1  & 99.6$\pm$0.1     & 87.8$\pm$0.2  & 70.3$\pm$0.3   & 66.4$\pm$0.3   & 85.2 \\
		SimNet \cite{SimNet}            & 88.6$\pm$0.5  & 98.2$\pm$0.2  & 99.7$\pm$0.2     & 85.3$\pm$0.3  & 73.4$\pm$0.8   & 71.8$\pm$0.6   & 86.2 \\
		MCD \cite{mcd}                  & 89.6$\pm$0.2  & 98.5$\pm$0.1  & 100.0$\pm$.0     & 91.3$\pm$0.2  & 69.6$\pm$0.1   & 70.8$\pm$0.3   & 86.6 \\
		CDAN+E \cite{cada}               & 94.1$\pm$0.1  & 98.6$\pm$0.1  & \textbf{100.0}$\pm$.0     & 92.9$\pm$0.2  & 71.0$\pm$0.3   & 69.3$\pm$0.3   & 87.7 \\
		MDD \cite{mdd}                   & \textbf{94.5}$\pm$0.3  & 98.4$\pm$0.1  & \textbf{100.0}$\pm$.0   & 93.5$\pm$0.2     & \textbf{74.6}$\pm$0.3  & 72.2$\pm$0.1   & 88.9 \\
		\hline
		\textbf{SymmNets-V1 \cite{symnets}} & 90.8$\pm$0.1  & \textbf{98.8}$\pm$0.3           & \textbf{100.0}$\pm$.0 & \textbf{93.9}$\pm$0.5 & \textbf{74.6}$\pm$0.6                   & 72.5$\pm$0.5   & 88.4  \\
		\textbf{SymmNets-V2}                & 94.2$\pm$0.1  & \textbf{98.8}$\pm$0.0 & \textbf{100.0}$\pm$.0     & 93.5$\pm$0.3  & 74.4$\pm$0.1   & \textbf{73.4}$\pm$0.2                                  & \textbf{89.1}  \\
		\hline
		\hline
		Kang \emph{et al.} \cite{attention_alignment}   & 86.8$\pm$0.2  & \textbf{99.3}$\pm$0.1  & \textbf{100.0}$\pm$.0     & 88.8$\pm$0.4  & 74.3$\pm$0.2   & 73.9$\pm$0.2   & 87.2 \\
		TADA \cite{tada}                 & 94.3$\pm$0.3  & 98.7$\pm$0.1  & 99.8$\pm$0.2   & 91.6$\pm$0.3      & 72.9$\pm$0.2  & 73.0$\pm$0.3  & 88.4 \\
		CADA-P \cite{Kurmi_2019_CVPR}    & \textbf{97.0}$\pm$0.2  & 99.3$\pm$0.1  & \textbf{100.0}$\pm$.0     & \textbf{95.6}$\pm$0.1  & 71.5$\pm$0.2            & 73.1$\pm$0.3            & 89.5 \\
		CAN \cite{can}                   & 94.5$\pm$0.3  & 99.1$\pm$0.2  & 99.8$\pm$0.2     &95.0$\pm$0.3   &\textbf{78.0}$\pm$0.3    & \textbf{77.0}$\pm$0.3   & 90.6 \\
		\hline
		\textbf{SymmNets-V2-SC}                 & 94.9$\pm$0.3  & 99.1$\pm$0.1  & \textbf{100.0}$\pm$.0 & \textbf{95.6}$\pm$0.3 & 77.6$\pm$0.4 & \textbf{77.0}$\pm$0.3 & \textbf{90.7}  \\
		
		\hline
	\end{tabular}
\end{table*}

\begin{table*}[htb]
	\centering
	\caption{Accuracy (\%) on the ImageCLEF-DA dataset \cite{ImageCLEFDA} for closed set UDA. Results are based on models adapted from a 50-layer ResNet.}
	\label{Tab:ImageCLEF}
	\begin{tabular}{lccccccc}
		\hline
		Methods                         & I $\to$ P     & P $\to$ I     & I $\to$ C         & C $\to$ I    & C $\to$ P      & P $\to$ C      & Avg  \\		
		\hline
		Source Only\cite{resnet}           & 74.8$\pm$0.3  & 83.9$\pm$0.1  & 91.5$\pm$0.3     & 78.0$\pm$0.2  & 65.5$\pm$0.3   & 91.2$\pm$0.3   & 80.7 \\
		
		DAN \cite{dan}                  & 74.5$\pm$0.4  & 82.2$\pm$0.2  & 92.8$\pm$0.2     & 86.3$\pm$0.4  & 69.2$\pm$0.4   & 89.8$\pm$0.4   & 82.5 \\
		
		DANN \cite{reverse_grad,dann}     & 75.0$\pm$0.6  & 86.0$\pm$0.3  & 96.2$\pm$0.4     & 87.0$\pm$0.5  & 74.3$\pm$0.5   & 91.5$\pm$0.6   & 85.0 \\
		
		JAN \cite{jan}                  & 76.8$\pm$0.4  & 88.0$\pm$0.2  & 94.7$\pm$0.2     & 89.5$\pm$0.3  & 74.2$\pm$0.3   & 91.7$\pm$0.3   & 85.8 \\
		
		MADA \cite{mada}                & 75.0$\pm$0.3  & 87.9$\pm$0.2  & 96.0$\pm$0.3     & 88.8$\pm$0.3  & 75.2$\pm$0.2   & 92.2$\pm$0.3   & 85.8 \\
		
		CDAN+E \cite{cada}              & 77.7$\pm$0.3  & 90.7$\pm$0.2  & \textbf{97.7}$\pm$0.3     & 91.3$\pm$0.3  & 74.2$\pm$0.2   & 94.3$\pm$0.3   & 87.7 \\
		\hline
		\textbf{SymmNets-V1 \cite{symnets}} & \textbf{80.2}$\pm$0.3  & \textbf{93.6}$\pm$0.2 & 97.0$\pm$0.3 & \textbf{93.4}$\pm$0.3  & 78.7$\pm$0.3          & \textbf{96.4}$\pm$0.1  &\textbf{89.9}  \\
		\textbf{SymmNets-V2}                & 79.0$\pm$0.3           &        93.5$\pm$0.2   & 96.9$\pm$0.2          & \textbf{93.4}$\pm$0.3  & \textbf{79.2}$\pm$0.3 & 96.2$\pm$0.1           & 89.7  \\
		\hline
		\hline
		CADA-P \cite{Kurmi_2019_CVPR}   & 78.0          & 90.5          & 96.7             & 92.0          & 77.2           & 95.5           & 88.3  \\
		\hline
		\textbf{SymmNets-V2-SC}     & \textbf{79.2}$\pm$0.2           &  \textbf{96.2}$\pm$0.3  & \textbf{96.8}$\pm$0.1           & \textbf{93.8}$\pm$0.2 & \textbf{77.8}$\pm$0.4  & \textbf{96.2}$\pm$0.  & \textbf{90.0}  \\
		\hline
	\end{tabular}
\end{table*}

\begin{table*}[h]
	\begin{center}
		
		\caption{Accuracy (\%) on the Office-Home dataset \cite{office_home} for closed set UDA. Results are based on models adapted from a 50-layer ResNet.}
		\label{Tab:Office-Home}
		\begin{tabular}{L{26.2mm}C{7.7mm}C{7.7mm}C{7.9mm}C{7.7mm}C{7.7mm}C{7.9mm}C{7.7mm}C{7.7mm}C{7.7mm}C{7.9mm}C{7.9mm}C{7.9mm}C{7.9mm}C{9mm}}
			\hline
			Methods                  &A$\to$C &A$\to$P &A$\to$R &C$\to$A &C$\to$P &C$\to$R &P$\to$A &P$\to$C &P$\to$R &R$\to$A &R$\to$C &R$\to$P    & Avg  \\
			\hline
			Source Only\cite{resnet}   &34.9      & 50.0     & 58.0      & 37.4      & 41.9      & 46.2     & 38.5     & 31.2     & 60.4     & 53.9     & 41.2     & 59.9 &46.1 \\			
			DAN \cite{dan}           & 43.6     & 57.0     & 67.9      & 45.8      & 56.5      & 60.4     & 44.0     & 43.6     & 67.7     & 63.1     & 51.5     &74.3  &56.3 \\			
			DANN \cite{reverse_grad,dann}& 45.6     & 59.3     & 70.1      & 47.0      & 58.5      & 60.9     & 46.1     & 43.7     & 68.5     & 63.2     &51.8      & 76.8 &57.6 \\		
			JAN \cite{jan}            & 45.9     & 61.2     & 68.9      & 50.4      & 59.7      & 61.0     & 45.8     & 43.4     & 70.3     & 63.9     &52.4      & 76.8 &58.3 \\
			CDAN+E \cite{cada}     & 50.7    & 70.6     & 76.0     & 57.6       & 70.0      & 70.0     & 57.4    & 50.9     & 77.3     &70.9      & 56.7     &81.6   &65.8 \\		
			MDD     \cite{mdd}     & \textbf{54.9}    & 73.7 &77.8 & 60.0 & 71.4 & 71.8 &61.2 & \textbf{53.6} & 78.1 & 72.5 & \textbf{60.2} & 82.3 & \textbf{68.1} \\		
			\hline
			\textbf{SymmNets-V1 \cite{symnets}} & 47.7 & 72.9          & 78.5            & 64.2          &71.3          &\textbf{74.2} &64.2          &48.8 &79.5          &\textbf{74.5} &52.6 &82.7          & 67.6  \\
			\textbf{SymmNets-V2}                & 48.1 & \textbf{74.3}          & \textbf{78.7}            & \textbf{64.6}          &\textbf{71.8}          & 74.1         & \textbf{64.4}         & 50.0& \textbf{80.2}         & 74.3         & 53.1& \textbf{83.2}         & \textbf{68.1}   \\
			\hline
			\hline
			DWT-MEC \cite{dwt}   & 50.3  & 72.1     & 77.0      & 59.6      & 69.3      & 70.2     & 58.3     & 48.1     & 77.3     & 69.3     &53.6     &82.0   & 65.6 \\
			TADA   \cite{tada}     & 53.1 & 72.3 & 77.2 & 59.1 & 71.2 &72.1 & 59.7 & 53.1 & 78.4 &72.4 &60.0 & 82.9 & 67.6 \\
			CADA-P  \cite{Kurmi_2019_CVPR}       & \textbf{56.9}    &76.4  &\textbf{80.7} & 61.3 & \textbf{75.2} & 75.2 &63.2 & \textbf{54.5} & 80.7 & \textbf{73.9} & \textbf{61.5} & \textbf{84.1} & \textbf{70.2} \\
			\hline
			\textbf{SymmNets-V2-SC}             & 51.6 & \textbf{76.9} & 80.3   & \textbf{68.6} &71.8 &\textbf{78.3}          &\textbf{65.8} &50.5 &\textbf{81.2} &73.1          &54.2 &82.4 & 69.6  \\
			\hline
		\end{tabular}
	\end{center}
\end{table*}

\vspace{0.1cm}
\noindent \textbf{Investigation of the Values of $\nu$ in Open Set UDA}
We conduct experiments on the Office-31 dataset to investigate the effects of different values of $\nu$ for open set UDA. We plot in Figure \ref{fig:open_set_eta_change} the accuracy of the unknown class, and mean accuracy over domain-shared classes (OS$^*$) and all classes (OS) with different values of $\nu$. As $\nu$ increases, the accuracy of the unknown class improves significantly whereas the mean accuracy of domain-shared classes drops slightly. We empirically set $\nu=6$ in all experiments, which consistently gives good results.

\vspace{0.1cm}
\noindent\textbf{Feature Visualization} To have an intuitive understanding of what features comparative methods have learned, we visualize via t-SNE \cite{sne} in Figure \ref{Fig:sne} the network activations respectively from the feature extractors of DANN \cite{reverse_grad,dann} and SymmNets-V2 on the adaptation task of \textbf{A} $\to$ \textbf{W}. Compared with features learned by DANN, those by SymmNets-V2 are better aligned across the two domains for shared classes under all the settings of the closed set, partial, and open set UDA, and they are well distinguished for domain-specific classes under the settings of partial and open set UDA; the visualization confirms the fineness of SymmNets-V2 in characterizing multi-class UDA.

\begin{table*}[htb]
	\centering
	\caption{Accuracy (\%) on the Office-31 dataset \cite{office_31} for partial UDA. Results are based on models adapted from a 50-layer ResNet.}
	\label{Tab:partial_office31}
	\begin{tabular}{lccccccc}
		\hline
		Methods                          & A $\to$ W     & D $\to$ W     & W $\to$ D         & A $\to$ D     & D $\to$ A      & W $\to$ A     & Avg  \\
		\hline
		Source Only\cite{resnet}           & 54.52         & 94.57         & 94.27             & 65.61         & 73.17          & 71.71         & 75.64 \\
		DAN \cite{dan}                   & 46.44         & 53.56         & 58.60             & 42.68         & 65.66          & 65.34         & 55.38 \\
		DANN \cite{reverse_grad,dann}      & 41.35         & 46.78         & 38.85             & 41.36         & 41.34          & 44.68         & 42.39 \\
		ADDA \cite{adda}                 & 43.65         & 46.48         & 40.12             & 43.66         & 42.76          & 45.95         & 43.77 \\
		RTN \cite{rtn}                   & 75.25         & 97.12         & 98.32             & 66.88         & 85.59          & 85.70         & 84.81 \\
		JAN\cite{jan}                    & 43.39         & 53.56         & 41.40             & 35.67         & 51.04          & 51.57         & 46.11 \\
		%LEL                              & 73.22         & 93.90         & 96.82             & 76.43         & 83.62          & 84.76         & 84.79 \\
		PADA \cite{pada}                 & 86.54         & 99.32         & \textbf{100.00}     & 82.17         & 92.69          & 95.41         & 92.69 \\
		ETN \cite{transfer_example_partial}& 94.52       & \textbf{100.00}        & \textbf{100.00}            & 95.03         & \textbf{96.21}          & 94.64         & 96.73  \\
		\hline
		\textbf{SymmNets-V2}   & 83.10         & 92.91         & 94.27             & 77.71         & 74.42          & 73.49         & 82.61 \\
		\textbf{SymmNets-V2} (With active $\omega_k$) & \textbf{99.83} & 98.64        & \textbf{100.00}            & \textbf{97.85} & 93.25 & \textbf{96.00}         & \textbf{97.60} \\
		\hline
	\end{tabular}
\end{table*}

\begin{table*}[htb]
	\begin{center}
		\caption{Accuracy (\%) on the Office-Home dataset \cite{office_home} for partial UDA. Results are based on models adapted from a 50-layer ResNet.}
		\label{Tab:partial_Office-Home}
		\begin{tabular}{L{26.9mm}C{7.7mm}C{7.7mm}C{7.9mm}C{7.7mm}C{7.7mm}C{7.9mm}C{7.7mm}C{7.7mm}C{7.7mm}C{7.9mm}C{7.9mm}C{7.9mm}C{7.9mm}C{9mm}}
			\hline
			Methods                    &A$\to$C &A$\to$P &A$\to$R &C$\to$A &C$\to$P &C$\to$R &P$\to$A &P$\to$C &P$\to$R &R$\to$A &R$\to$C &R$\to$P & Avg  \\
			\hline
			Source Only\cite{resnet}     &38.57     & 60.78    & 75.21     & 39.94     &48.12      &52.90     &49.68     &30.91     &70.79     &65.38     &41.79     &70.42   &53.71 \\
			DAN \cite{dan}             &44.36     &61.79     &74.49      &41.78      &45.21      &54.11     &46.92     &38.14     &68.42     &64.37     &45.37     &68.85   &54.48 \\
			DANN \cite{reverse_grad,dann} &44.89    &54.06     &68.97      &36.27      &34.34      &45.22     &44.08     &38.03     &68.69     &52.98     &34.68     &46.50   &47.39 \\
			PADA   \cite{pada}         &51.95     &67.00     &78.74      &52.16      &53.78      &59.03     &52.61     &43.22     &78.79     &73.73     &56.60     &77.09   &62.06 \\
			ETN \cite{transfer_example_partial} & \textbf{59.24} &77.03 & 79.54   &62.92      & 65.73     &75.01     &68.29     &55.37     &84.37     &75.72     &\textbf{57.66}     &\textbf{84.54}   &70.45 \\
			\hline
			\textbf{SymmNets-V2}  & 53.12 & 67.87 &73.57      &62.43      & 56.73      &64.08     &\textbf{56.26}     &59.61     &69.36     &66.64     &52.30     &69.56   &62.63 \\
			\textbf{SymmNets-V2}  &  \multirow{2}{*}{55.46}   &\multirow{2}{*}{\textbf{78.71}} &\multirow{2}{*}{\textbf{84.59}} &\multirow{2}{*}{\textbf{70.98}}      & \multirow{2}{*}{\textbf{67.39}}      &\multirow{2}{*}{\textbf{77.91}}    &\multirow{2}{*}{\textbf{76.22}}     &\multirow{2}{*}{54.45}     &\multirow{2}{*}{\textbf{88.46}}     &\multirow{2}{*}{\textbf{77.23}}     &\multirow{2}{*}{57.07}     &\multirow{2}{*}{83.75}   &\multirow{2}{*}{\textbf{72.69}}  \\
			(With active $\omega_k$) &&&&&&&&&&&&& \\
			\hline
		\end{tabular}
	\end{center}
\end{table*}

\begin{table*}[h!]
	\begin{center}
		\caption{Accuracy (\%) on the Office-31 dataset \cite{office_31} for open set UDA. Results of all methods are based on models adapted from a 50-layer ResNet.}
		\label{Tab:open_Office31}
		
		\begin{tabular}{lllllllllllllll}
			\cline{1-15}
			\multirow{2}{*}{Methods}         & \multicolumn{2}{|c|}{A$\to$D}                                    & \multicolumn{2}{c|}{A$\to$W}                                      & \multicolumn{2}{c|}{D$\to$A}                                    & \multicolumn{2}{c|}{D$\to$W}                                    & \multicolumn{2}{c|}{W$\to$A}                                    & \multicolumn{2}{c|}{W$\to$D}                                    & \multicolumn{2}{c}{AVG}                          \\ \cline{2-15}
			& \multicolumn{1}{|c}{OS} & \multicolumn{1}{c|}{OS$^*$}           & \multicolumn{1}{c}{OS}   & \multicolumn{1}{c|}{OS$^*$}           & \multicolumn{1}{c}{OS} & \multicolumn{1}{c|}{OS$^*$}           & \multicolumn{1}{c}{OS} & \multicolumn{1}{c|}{OS$^*$}           & \multicolumn{1}{c}{OS} & \multicolumn{1}{c|}{OS$^*$}           & \multicolumn{1}{c}{OS} & \multicolumn{1}{c|}{OS$^*$}           & \multicolumn{1}{c}{OS} & \multicolumn{1}{c}{OS$^*$} \\ \hline

			\hline
			%\multicolumn{15}{l}{\textbf{Methods based on ResNet-50 }}  \\ \hline
			\multicolumn{1}{l|}{Source Only \cite{resnet}}      & 85.2                   & \multicolumn{1}{l|}{85.5}          & 82.5                     & \multicolumn{1}{l|}{82.7}          & 71.6                   & \multicolumn{1}{l|}{71.5}          & 94.1                   & \multicolumn{1}{l|}{94.3}          & 75.5                   & \multicolumn{1}{l|}{75.2}          & 96.6                   & \multicolumn{1}{l|}{97.0}          & 84.2                   & 84.4                     \\
			\multicolumn{1}{l|}{DANN \cite{dann}}      & 86.5                   & \multicolumn{1}{l|}{87.7}          & 85.3                     & \multicolumn{1}{l|}{87.7}          & 75.7                   & \multicolumn{1}{l|}{76.2}          & 97.5                   & \multicolumn{1}{l|}{\textbf{98.3}}          & 74.9                   & \multicolumn{1}{l|}{75.6}          & \textbf{99.5}                   & \multicolumn{1}{l|}{\textbf{100.0}}          & 86.6                   & 87.6                     \\
			\multicolumn{1}{l|}{ATI-$\lambda$\cite{open_set_math}}      & 84.3                   & \multicolumn{1}{l|}{86.6}          & 87.4                     & \multicolumn{1}{l|}{88.9}          & 78.0                   & \multicolumn{1}{l|}{79.6}          & 93.6                   & \multicolumn{1}{l|}{95.3}          & 80.4                   & \multicolumn{1}{l|}{81.4}          & 96.5                   & \multicolumn{1}{l|}{98.7}          & 86.7                   & 88.4                     \\
			\multicolumn{1}{l|}{AODA\cite{open_set_bp}}    & 88.6                   & \multicolumn{1}{l|}{89.2}          & 86.5                     & \multicolumn{1}{l|}{87.6}          & 88.9                   & \multicolumn{1}{l|}{90.6}          & 97.0          & \multicolumn{1}{l|}{96.5}          & 85.8                   & \multicolumn{1}{l|}{84.9}          & 97.9          & \multicolumn{1}{l|}{98.7}          & 90.8                   & 91.3                     \\
			\multicolumn{1}{l|}{STA \cite{liu2019separate}}      & 93.7                   & \multicolumn{1}{l|}{96.1}          &    89.5                  & \multicolumn{1}{l|}{92.1}          & 89.1                   & \multicolumn{1}{l|}{\textbf{93.5}}          & 97.5                   & \multicolumn{1}{l|}{96.5}          & 87.9                   & \multicolumn{1}{l|}{87.4}          & \textbf{99.5}                   & \multicolumn{1}{l|}{99.6}          & 92.9                   & 94.1                     \\
			\hline
			\multicolumn{1}{l|}{\textbf{SymmNets-V2} ($\nu=6$)}  & \textbf{96.3}    & \multicolumn{1}{l|}{\textbf{97.5}} & \textbf{95.7}            & \multicolumn{1}{l|}{\textbf{96.1}} & \textbf{91.6}          & \multicolumn{1}{l|}{91.7} & \textbf{97.8}                   & \multicolumn{1}{l|}{\textbf{98.3}} & \textbf{92.3}          & \multicolumn{1}{l|}{\textbf{92.9}} & 99.2                   & \multicolumn{1}{l|}{\textbf{\textbf{100.0}}} & \textbf{95.5}          & \textbf{96.1}            \\
			\hline
			
			\hline
		\end{tabular}
	\end{center}
\end{table*}

\begin{table*}[h!]
	\begin{center}
		\caption{Accuracy (\%) on Syn2Real dataset \cite{syn2real} for open set UDA. Results of all methods are based on models adapted from a 152-layer ResNet.}
		\label{Tab:open_visda}
		\begin{tabular}{l|C{3.7mm}C{3.7mm}C{3.7mm}C{3.7mm}C{3.7mm}C{3.7mm}C{3.7mm}C{3.7mm}C{3.7mm}C{3.7mm}C{3.7mm}C{3.7mm}C{3.7mm}|C{3.7mm}C{3.7mm}}
			\hline
			Methods & \rotatebox{45}{plane} & \rotatebox{45}{bcycle} &\rotatebox{45}{bus} & \rotatebox{45}{car} & \rotatebox{45}{horse} & \rotatebox{45}{knife} & \rotatebox{45}{mcycl} & \rotatebox{45}{person} & \rotatebox{45}{plant} & \rotatebox{45}{sktbrd} & \rotatebox{45}{train} & \rotatebox{45}{trunk} & \rotatebox{45}{unk} & \rotatebox{45}{OS$^*$} & \rotatebox{45}{OS}\\
			\hline
			\multicolumn{14}{l}{\textbf{Known-to-Unknown Ratio = 1:1}} & &  \\
			\hline
			Source Only \cite{resnet}   & 36   & 27  &  21 & 49 & 66 & 0  & 69 & 1  & 42 & 8  & 59 & 0  & 81 & 31 & 35  \\
			DANN \cite{dann}           & 53   & 5   &  31 & 61 & 75 & 3  & 81 & 11 & 63 & 29 & 68 & 5  & 76 & 43 & 40 \\
			% SE \cite{se}               & 94   & 82  &  87 & 67 & 94 & 31 & 91 & 63 & 89 & 76 & 80 & 33 & 53 & 73 & 72 & 64 \\
			AODA \cite{open_set_bp}    & 85   & 71  &  65 & 53 & 83 & \textbf{10} & 79 & 36 & 73 & 56 & 79 & \textbf{32} & \textbf{87} & 60 & 62 \\											
			\textbf{SymmNets-V2} ($\nu=6$)  & \textbf{93}   & \textbf{79}  &  \textbf{85} & \textbf{75} & \textbf{92} & 3  & \textbf{91} & \textbf{80} & \textbf{84} & \textbf{69} & \textbf{75} & 2 & 57 & \textbf{69} & \textbf{68} \\
			% \textbf{\textcolor{red}{SymNets*}} ($\eta=6$) & 89   & 70  &  81 & 65 & 88 & 4  & 89 & 67 & 78 & 62 & 81 & 25 & 74 & 67 & 67 & 71 \\
			%SymNet(w/o Source Unknown) & 85   & 71  &  65 & 53 & 83 & 10 & 79 & 36 & 73 & 56 & 79 & 32 & 87 & 60 & 62 & 73 \\
			\hline
			
			\hline
			\multicolumn{14}{l}{\textbf{Known-to-Unknown Ratio = 1:10}} & &  \\
			\hline
			Source Only \cite{resnet}   & 23   & 24  &  43 & 40 & 44 & 0  & 56 & 2  & 24 & 8  & 47 & 1  & \textbf{93} & 26 & 31 \\
			% SE \cite{se}               & 94   & 74  &  86 & 68 & 91 & 26 & 95 & 46 & 85 & 40 & 79 & 11 & 51 & 66 & 65 & 53 \\
			AODA \cite{open_set_bp}    & 80   & 63  &  59 & 63 & 83 & 12 & 89 & 5  & \textbf{61} & 14 & 79 & 0  & 69 & 51 & 52  \\
			\textbf{SymmNets-V2} ($\nu=6$) &\textbf{90}& \textbf{72}  &  \textbf{76} & \textbf{68} & \textbf{90} & \textbf{14} & \textbf{94} & \textbf{18} & 59 & \textbf{20} & \textbf{83} & \textbf{5}  & 70 & \textbf{59} & \textbf{59}  \\
			%\textbf{\textcolor{red}{SymNets*}} ($\eta=6$)          & 84   & 67  &  71 & 65 & 88 & 3  & 93 & 61 & 59 & 26 & 78 & 3  & 62 & 58 & 59 & 62 \\
			%SymNet(w/o Source Unknown) & 85   & 71  &  65 & 53 & 83 & 10 & 79 & 36 & 73 & 56 & 79 & 32 & 87 & 60 & 62 & 73 \\
			\hline

		\end{tabular}
	\end{center}
\end{table*}

\subsection{Comparisons with the State of the Art} \label{SecResults}

\noindent\textbf{Closed Set UDA} We report in Table \ref{Tab:visda}, Table \ref{Tab:office31}, Table \ref{Tab:ImageCLEF}, and Table \ref{Tab:Office-Home} the classification results respectively on the popular closed set UDA datasets of VisDA-2017 \cite{visda}, Office-31 \cite{office_31}, ImageCLEF-DA \cite{ImageCLEFDA}, and Office-Home \cite{office_home}. Compared with existing adversarial learning-based methods, including the seminal one of DANN \cite{dann} and the recent ones of MCD \cite{mcd}, CDAN \cite{cada}, MDD \cite{mdd}, and SWD \cite{swd}, our SymmNets-V2 achieves better performance on most of these benchmarks, demonstrating the efficacy and fineness of SymmNets-V2 in characterizing multi-class UDA. We note that there exist a few recent methods that focus on other strategies, such as the feature attention strategy \cite{tada,cada}, prototypical network \cite{tpn}, prediction consistency w.r.t input perturbation \cite{dwt}, and intra- and inter-class discrepancies \cite{can}, all of which are orthogonal to the strategy of adversarial training studied in the present work. To compare with these methods more fairly, we consider a few strategies of these methods amenable to adversarial training, including the class-aware sampling \cite{can} empowered by alternative optimization \cite{can}, use of pseudo labels of target data as in the prototypical network \cite{tpn}, and the min-entropy consensus \cite{dwt}, resulting in a variant of our method termed as ``SymmNets-V2 Strengthened for Closed Set UDA (SymmNets-V2-SC)''. SymmNets-V2-SC boosts the performance of SymmNets-V2 on the closed set UDA, especially on the VisDA-2017 dataset \cite{visda}, indicating a promising direction of combining multiple strategies for the setting of closed set UDA.

% For fair comparison, results of other methods are either directly reported from their original papers if available or quoted from \cite{adr, cada}.
% class-aware sampling \cite{can} empowered by alternative optimization \cite{can} to promote the class-wise consistency of source and target batches in every iteration
% adopt the min-entropy consensus \cite{dwt} to promote prediction consistency of the same sample under different disturbances.

\noindent\textbf{Partial UDA} We report in Table \ref{Tab:partial_office31} and Table \ref{Tab:partial_Office-Home} the classification results respectively on the popular partial UDA datasets of Office-31 \cite{office_31} and Office-Home \cite{office_home}. The seminal methods \cite{dan, dann} achieve worse results than the Source Only baseline; in contrast, our SymmNets-V2 improves over the Source Only baseline by a large margin, confirming the effectiveness of our method in characterizing the domain distance at a finer level. Our SymmNets-V2 (with active $\omega_k$) outperforms all state-of-the-art methods on the two benchmark datasets, again confirming the effectiveness of our method.

% Results of other methods are either directly reported from their original papers if available or quoted from \cite{pada}.

\noindent \textbf{Open Set UDA} We report in Table \ref{Tab:open_Office31} and Table \ref{Tab:open_visda} the classification results respectively on the popular open set UDA datasets of Office-31 \cite{office_31} and Syn2Real \cite{syn2real}.  Our SymmNets-V2 ($\nu=6$) outperforms all state-of-the-art methods on the two benchmarks, confirming the effectiveness of our method in aligning both the domain-shared classes and the unknown class across source and target domains.

\section{Conclusion}

In this paper, we study the formalism of unsupervised multi-class domain adaptation. We contribute a new bound for multi-class UDA based on a novel notion of Multi-Class Scoring Disagreement (MCSD); a corresponding data-dependent PAC bound is also developed based on the notion of Rademacher complexity. The proposed MCSD is able to fully characterize the relations between any pair of multi-class scoring hypotheses, which is finer compared with those in existing domain adaptation bounds. Our derived bounds naturally suggest the Multi-class Domain-adversarial learning Networks (McDalNets), which promotes the alignment of conditional feature distributions across source and target domains. We show that different instantiations of McDalNets via surrogate learning objectives either coincide with or resemble a few recently popular methods, thus (partially) underscoring their practical effectiveness. Based on our same theory of multi-class UDA, we also introduce a new algorithm of Domain-Symmetric Networks (SymmNets), which is featured by a novel adversarial strategy of domain confusion and discrimination. SymmNets affords simple extensions that work equally well under the problem settings of either closed set, partial, or open set UDA. Careful empirical studies show that algorithms of McDalNets based on the MCSD surrogates consistently improve over these based on the scalar-valued versions.
%consistently improve over its degenerate versions.
Experiments under the settings of closed set, partial, and open set UDA also confirm the effectiveness of our proposed SymmNets empirically. The contributed theory and algorithms connect better with the practice in multi-class UDA. We expect they could provide useful principles for algorithmic design in future research.

% use section* for acknowledgment
\ifCLASSOPTIONcompsoc
  % The Computer Society usually uses the plural form
  \section*{Acknowledgments}
\else
  % regular IEEE prefers the singular form
  \section*{Acknowledgment}
\fi

This work is supported in part by the National Natural Science Foundation of China (Grant No.: 61771201), the Program for Guangdong Introducing Innovative and Enterpreneurial Teams (Grant No.: 2017ZT07X183), and the Guangdong R\&D key project of China (Grant No.: 2019B010155001).

% if have a single appendix:
%\appendix[Proof of the Zonklar Equations]
% or
%\appendix  % for no appendix heading
% do not use \section anymore after \appendix, only \section*
% is possibly needed

% use appendices with more than one appendix
% then use \section to start each appendix
% you must declare a \section before using any
% \subsection or using \label (\appendices by itself
% starts a section numbered zero.)

% Can use something like this to put references on a page
% by themselves when using endfloat and the captionsoff option.
\ifCLASSOPTIONcaptionsoff
  \newpage
\fi

% trigger a \newpage just before the given reference
% number - used to balance the columns on the last page
% adjust value as needed - may need to be readjusted if
% the document is modified later
%\IEEEtriggeratref{8}
% The "triggered" command can be changed if desired:
%\IEEEtriggercmd{\enlargethispage{-5in}}

% references section

% can use a bibliography generated by BibTeX as a .bbl file
% BibTeX documentation can be easily obtained at:
% http://mirror.ctan.org/biblio/bibtex/contrib/doc/
% The IEEEtran BibTeX style support page is at:
% http://www.michaelshell.org/tex/ieeetran/bibtex/
%\bibliographystyle{IEEEtran}
% argument is your BibTeX string definitions and bibliography database(s)
%\bibliography{IEEEabrv,../bib/paper}

%{\small
%\bibliographystyle{IEEEtran}
%\bibliography{egbib}
%}

\bibliographystyle{IEEEtran_bst}
\bibliography{IEEEabrv,egbib}

% Generated by IEEEtran.bst, version: 1.14 (2015/08/26)
\begin{thebibliography}{10}
\providecommand{\url}[1]{#1}
\csname url@samestyle\endcsname
\providecommand{\newblock}{\relax}
\providecommand{\bibinfo}[2]{#2}
\providecommand{\BIBentrySTDinterwordspacing}{\spaceskip=0pt\relax}
\providecommand{\BIBentryALTinterwordstretchfactor}{4}
\providecommand{\BIBentryALTinterwordspacing}{\spaceskip=\fontdimen2\font plus
\BIBentryALTinterwordstretchfactor\fontdimen3\font minus
  \fontdimen4\font\relax}
\providecommand{\BIBforeignlanguage}[2]{{%
\expandafter\ifx\csname l@#1\endcsname\relax
\typeout{** WARNING: IEEEtran.bst: No hyphenation pattern has been}%
\typeout{** loaded for the language `#1'. Using the pattern for}%
\typeout{** the default language instead.}%
\else
\language=\csname l@#1\endcsname
\fi
#2}}
\providecommand{\BIBdecl}{\relax}
\BIBdecl

\bibitem{MLFoudatation2014}
S.~Shalev-Shwartz and S.~Ben-David, \emph{Understanding machine learning: From
  theory to algorithms}.\hskip 1em plus 0.5em minus 0.4em\relax Cambridge
  university press, 2014.

\bibitem{alexnet}
A.~Krizhevsky, I.~Sutskever, and G.~E. Hinton, ``Imagenet classification with
  deep convolutional neural networks,'' in \emph{Advances in neural information
  processing systems}, 2012, pp. 1097--1105.

\bibitem{girshick2014rich}
R.~Girshick, J.~Donahue, T.~Darrell, and J.~Malik, ``Rich feature hierarchies
  for accurate object detection and semantic segmentation,'' in
  \emph{Proceedings of the IEEE conference on computer vision and pattern
  recognition}, 2014, pp. 580--587.

\bibitem{jia2019orthogonal}
K.~Jia, S.~Li, Y.~Wen, T.~Liu, and D.~Tao, ``Orthogonal deep neural networks,''
  \emph{IEEE Transactions on Pattern Analysis and Machine Intelligence}, 2019.

\bibitem{vinyals2016matching}
O.~Vinyals, C.~Blundell, T.~Lillicrap, D.~Wierstra \emph{et~al.}, ``Matching
  networks for one shot learning,'' in \emph{Advances in neural information
  processing systems}, 2016, pp. 3630--3638.

\bibitem{zhang2019part}
Y.~Zhang, K.~Jia, and Z.~Wang, ``Part-aware fine-grained object categorization
  using weakly supervised part detection network,'' \emph{IEEE Transactions on
  Multimedia}, 2019.

\bibitem{jia2019deep}
K.~Jia, J.~Lin, M.~Tan, and D.~Tao, ``Deep multi-view learning using
  neuron-wise correlation-maximizing regularizers,'' \emph{IEEE Transactions on
  Image Processing}, vol.~28, no.~10, pp. 5121--5134, 2019.

\bibitem{transfer_survey}
S.~J. Pan, Q.~Yang \emph{et~al.}, ``A survey on transfer learning,'' \emph{IEEE
  Transactions on knowledge and data engineering}, vol.~22, no.~10, pp.
  1345--1359, 2010.

\bibitem{ben2007analysis}
S.~Ben-David, J.~Blitzer, K.~Crammer, and F.~Pereira, ``Analysis of
  representations for domain adaptation,'' in \emph{Advances in neural
  information processing systems}, 2007, pp. 137--144.

\bibitem{ben2010theory}
S.~Ben-David, J.~Blitzer, K.~Crammer, A.~Kulesza, F.~Pereira, and J.~W.
  Vaughan, ``A theory of learning from different domains,'' \emph{Machine
  learning}, vol.~79, no. 1-2, pp. 151--175, 2010.

\bibitem{mansour2009domain}
Y.~Mansour, M.~Mohri, and A.~Rostamizadeh, ``Domain adaptation: Learning bounds
  and algorithms,'' in \emph{22nd Conference on Learning Theory, COLT 2009},
  2009.

\bibitem{courty2016optimal}
N.~Courty, R.~Flamary, D.~Tuia, and A.~Rakotomamonjy, ``Optimal transport for
  domain adaptation,'' \emph{IEEE transactions on pattern analysis and machine
  intelligence}, vol.~39, no.~9, pp. 1853--1865, 2016.

\bibitem{dann}
Y.~Ganin, E.~Ustinova, H.~Ajakan, P.~Germain, H.~Larochelle, M.~Marchand, and
  V.~Lempitsky, ``Domain-adversarial training of neural networks,''
  \emph{Journal of Machine Learning Research}, vol.~17, no.~1, pp. 2096--2030,
  2017.

\bibitem{adda}
E.~Tzeng, J.~Hoffman, K.~Saenko, and T.~Darrell, ``Adversarial discriminative
  domain adaptation,'' in \emph{Computer Vision and Pattern Recognition
  (CVPR)}, vol.~1, no.~2, 2017, p.~4.

\bibitem{cada}
M.~Long, Z.~CAO, J.~Wang, and M.~I. Jordan, ``Conditional adversarial domain
  adaptation,'' in \emph{Advances in Neural Information Processing Systems 31},
  2018, pp. 1640--1650.

\bibitem{mcd}
K.~Saito, K.~Watanabe, Y.~Ushiku, and T.~Harada, ``Maximum classifier
  discrepancy for unsupervised domain adaptation,'' in \emph{Proceedings of the
  IEEE Conference on Computer Vision and Pattern Recognition}, 2018, pp.
  3723--3732.

\bibitem{symnets}
Y.~Zhang, H.~Tang, K.~Jia, and M.~Tan, ``Domain-symmetric networks for
  adversarial domain adaptation,'' in \emph{Proceedings of the IEEE Conference
  on Computer Vision and Pattern Recognition}, 2019, pp. 5031--5040.

\bibitem{mdd}
Y.~Zhang, T.~Liu, M.~Long, and M.~Jordan, ``Bridging theory and algorithm for
  domain adaptation,'' in \emph{International Conference on Machine Learning},
  2019, pp. 7404--7413.

\bibitem{adr}
K.~Saito, Y.~Ushiku, T.~Harada, and K.~Saenko, ``Adversarial dropout
  regularization,'' \emph{arXiv preprint arXiv:1711.01575}, 2017.

\bibitem{swd}
C.-Y. Lee, T.~Batra, M.~H. Baig, and D.~Ulbricht, ``Sliced wasserstein
  discrepancy for unsupervised domain adaptation,'' in \emph{Proceedings of the
  IEEE Conference on Computer Vision and Pattern Recognition}, 2019, pp.
  10\,285--10\,295.

\bibitem{Cicek_2019_ICCV}
S.~Cicek and S.~Soatto, ``Unsupervised domain adaptation via regularized
  conditional alignment,'' in \emph{The IEEE International Conference on
  Computer Vision (ICCV)}, October 2019.

\bibitem{dogan2016unified}
{\"U}.~Dogan, T.~Glasmachers, and C.~Igel, ``A unified view on multi-class
  support vector classification.'' \emph{Journal of Machine Learning Research},
  vol.~17, no.~45, pp. 1--32, 2016.

\bibitem{san}
Z.~Cao, M.~Long, J.~Wang, and M.~I. Jordan, ``Partial transfer learning with
  selective adversarial networks,'' in \emph{The IEEE Conference on Computer
  Vision and Pattern Recognition (CVPR)}, June 2018.

\bibitem{mada}
Z.~Pei, Z.~Cao, M.~Long, and J.~Wang, ``Multi-adversarial domain adaptation,''
  in \emph{AAAI Conference on Artificial Intelligence}, 2018.

\bibitem{pada}
Z.~Cao, L.~Ma, M.~Long, and J.~Wang, ``Partial adversarial domain adaptation,''
  in \emph{The European Conference on Computer Vision (ECCV)}, September 2018.

\bibitem{iwanpda}
J.~Zhang, Z.~Ding, W.~Li, and P.~Ogunbona, ``Importance weighted adversarial
  nets for partial domain adaptation,'' in \emph{Proceedings of the IEEE
  Conference on Computer Vision and Pattern Recognition}, 2018, pp. 8156--8164.

\bibitem{transfer_example_partial}
Z.~Cao, K.~You, M.~Long, J.~Wang, and Q.~Yang, ``Learning to transfer examples
  for partial domain adaptation,'' in \emph{Proceedings of the IEEE Conference
  on Computer Vision and Pattern Recognition}, 2019, pp. 2985--2994.

\bibitem{open_set_math}
P.~P. Busto, A.~Iqbal, and J.~Gall, ``Open set domain adaptation for image and
  action recognition,'' \emph{IEEE Transactions on Pattern Analysis and Machine
  Intelligence}, pp. 1--1, 2018.

\bibitem{open_set_bp}
K.~Saito, S.~Yamamoto, Y.~Ushiku, and T.~Harada, ``Open set domain adaptation
  by backpropagation,'' in \emph{Proceedings of the European Conference on
  Computer Vision (ECCV)}, 2018, pp. 153--168.

\bibitem{cortes2014domain}
C.~Cortes and M.~Mohri, ``Domain adaptation and sample bias correction theory
  and algorithm for regression,'' \emph{Theoretical Computer Science}, vol.
  519, pp. 103--126, 2014.

\bibitem{cortes2015adaptation}
C.~Cortes, M.~Mohri, and A.~Mu{\~n}oz~Medina, ``Adaptation algorithm and theory
  based on generalized discrepancy,'' in \emph{Proceedings of the 21th ACM
  SIGKDD International Conference on Knowledge Discovery and Data
  Mining}.\hskip 1em plus 0.5em minus 0.4em\relax ACM, 2015, pp. 169--178.

\bibitem{mohri2012new}
M.~Mohri and A.~M. Medina, ``New analysis and algorithm for learning with
  drifting distributions,'' in \emph{International Conference on Algorithmic
  Learning Theory}.\hskip 1em plus 0.5em minus 0.4em\relax Springer, 2012, pp.
  124--138.

\bibitem{zhang2012generalization}
C.~Zhang, L.~Zhang, and J.~Ye, ``Generalization bounds for domain adaptation,''
  in \emph{Advances in neural information processing systems}, 2012, pp.
  3320--3328.

\bibitem{courty2017joint}
N.~Courty, R.~Flamary, A.~Habrard, and A.~Rakotomamonjy, ``Joint distribution
  optimal transportation for domain adaptation,'' in \emph{Advances in Neural
  Information Processing Systems}, 2017, pp. 3730--3739.

\bibitem{kuroki2019unsupervised}
S.~Kuroki, N.~Charoenphakdee, H.~Bao, J.~Honda, I.~Sato, and M.~Sugiyama,
  ``Unsupervised domain adaptation based on source-guided discrepancy,'' in
  \emph{Proceedings of the AAAI Conference on Artificial Intelligence},
  vol.~33, 2019, pp. 4122--4129.

\bibitem{dan}
\BIBentryALTinterwordspacing
M.~Long, Y.~Cao, J.~Wang, and M.~I. Jordan, ``Learning transferable features
  with deep adaptation networks,'' in \emph{Proceedings of the 32Nd
  International Conference on International Conference on Machine Learning -
  Volume 37}, ser. ICML'15.\hskip 1em plus 0.5em minus 0.4em\relax JMLR.org,
  2015, pp. 97--105. [Online]. Available:
  \url{http://dl.acm.org/citation.cfm?id=3045118.3045130}
\BIBentrySTDinterwordspacing

\bibitem{jan}
M.~Long, H.~Zhu, J.~Wang, and M.~I. Jordan, ``Deep transfer learning with joint
  adaptation networks,'' in \emph{Proceedings of the 34th International
  Conference on Machine Learning-Volume 70}.\hskip 1em plus 0.5em minus
  0.4em\relax JMLR. org, 2017, pp. 2208--2217.

\bibitem{zhu2017unpaired}
J.-Y. Zhu, T.~Park, P.~Isola, and A.~A. Efros, ``Unpaired image-to-image
  translation using cycle-consistent adversarial networks,'' in
  \emph{Proceedings of the IEEE international conference on computer vision},
  2017, pp. 2223--2232.

\bibitem{rozantsev2018beyond}
A.~Rozantsev, M.~Salzmann, and P.~Fua, ``Beyond sharing weights for deep domain
  adaptation,'' \emph{IEEE transactions on pattern analysis and machine
  intelligence}, vol.~41, no.~4, pp. 801--814, 2018.

\bibitem{dwt}
S.~Roy, A.~Siarohin, E.~Sangineto, S.~R. Bulo, N.~Sebe, and E.~Ricci,
  ``Unsupervised domain adaptation using feature-whitening and consensus
  loss,'' in \emph{Proceedings of the IEEE Conference on Computer Vision and
  Pattern Recognition}, 2019, pp. 9471--9480.

\bibitem{dada}
H.~Tang and K.~Jia, ``Discriminative adversarial domain adaptation,'' in
  \emph{Association for the Advancement of Artificial Intelligence (AAAI)},
  2020.

\bibitem{pixel_level}
K.~Bousmalis, N.~Silberman, D.~Dohan, D.~Erhan, and D.~Krishnan, ``Unsupervised
  pixel-level domain adaptation with generative adversarial networks,'' in
  \emph{The IEEE Conference on Computer Vision and Pattern Recognition (CVPR)},
  vol.~1, no.~2, 2017, p.~7.

\bibitem{isola2017image}
P.~Isola, J.-Y. Zhu, T.~Zhou, and A.~A. Efros, ``Image-to-image translation
  with conditional adversarial networks,'' in \emph{Proceedings of the IEEE
  conference on computer vision and pattern recognition}, 2017, pp. 1125--1134.

\bibitem{domain_confusion}
E.~Tzeng, J.~Hoffman, T.~Darrell, and K.~Saenko, ``Simultaneous deep transfer
  across domains and tasks,'' in \emph{Proceedings of the IEEE International
  Conference on Computer Vision}, 2015, pp. 4068--4076.

\bibitem{gan}
I.~Goodfellow, J.~Pouget-Abadie, M.~Mirza, B.~Xu, D.~Warde-Farley, S.~Ozair,
  A.~Courville, and Y.~Bengio, ``Generative adversarial nets,'' in
  \emph{Advances in neural information processing systems}, 2014, pp.
  2672--2680.

\bibitem{villani2008optimal}
C.~Villani, \emph{Optimal transport: old and new}.\hskip 1em plus 0.5em minus
  0.4em\relax Springer Science \& Business Media, 2008, vol. 338.

\bibitem{VapnikBook}
V.~N. Vapnik, \emph{Statistical Learning Theory}.\hskip 1em plus 0.5em minus
  0.4em\relax John Wiley and Sons, 1998.

\bibitem{koltchinskii2002empirical}
V.~Koltchinskii, D.~Panchenko \emph{et~al.}, ``Empirical margin distributions
  and bounding the generalization error of combined classifiers,'' \emph{The
  Annals of Statistics}, vol.~30, no.~1, pp. 1--50, 2002.

\bibitem{lee2004multicategory}
Y.~Lee, Y.~Lin, and G.~Wahba, ``Multicategory support vector machines: Theory
  and application to the classification of microarray data and satellite
  radiance data,'' \emph{Journal of the American Statistical Association},
  vol.~99, no. 465, pp. 67--81, 2004.

\bibitem{liu2011reinforced}
Y.~Liu and M.~Yuan, ``Reinforced multicategory support vector machines,''
  \emph{Journal of Computational and Graphical Statistics}, vol.~20, no.~4, pp.
  901--919, 2011.

\bibitem{szedmak2005learning}
S.~Szedmak, J.~Shawe-Taylor \emph{et~al.}, ``Learning via linear operators:
  Maximum margin regression,'' in \emph{In Proceedings of 2001 IEEE
  International Conference on Data Mining}.\hskip 1em plus 0.5em minus
  0.4em\relax Citeseer, 2005.

\bibitem{mohri2012foundations}
M.~Mohri, A.~Rostamizadeh, and A.~Talwalkar, \emph{Foundations of machine
  learning}.\hskip 1em plus 0.5em minus 0.4em\relax MIT press, 2012.

\bibitem{ImageCLEFDA}
``Imageclef-da dataset,'' \url{http://imageclef.org/2014/adaptation/}.

\bibitem{office_31}
K.~Saenko, B.~Kulis, M.~Fritz, and T.~Darrell, ``Adapting visual category
  models to new domains,'' in \emph{European conference on computer
  vision}.\hskip 1em plus 0.5em minus 0.4em\relax Springer, 2010, pp. 213--226.

\bibitem{office_home}
H.~Venkateswara, J.~Eusebio, S.~Chakraborty, and S.~Panchanathan, ``Deep
  hashing network for unsupervised domain adaptation,'' in \emph{Proc. CVPR},
  2017, pp. 5018--5027.

\bibitem{mnist}
Y.~LeCun, L.~Bottou, Y.~Bengio, and P.~Haffner, ``Gradient-based learning
  applied to document recognition,'' \emph{Proceedings of the IEEE}, vol.~86,
  no.~11, pp. 2278--2324, 1998.

\bibitem{usps}
J.~J. Hull, ``A database for handwritten text recognition research,''
  \emph{IEEE Transactions on pattern analysis and machine intelligence},
  vol.~16, no.~5, pp. 550--554, 1994.

\bibitem{svhn}
Y.~Netzer, T.~Wang, A.~Coates, A.~Bissacco, B.~Wu, and A.~Y. Ng, ``Reading
  digits in natural images with unsupervised feature learning,'' in \emph{NIPS
  workshop on deep learning and unsupervised feature learning}, vol. 2011,
  no.~2, 2011, p.~5.

\bibitem{syn2real}
X.~Peng, B.~Usman, K.~Saito, N.~Kaushik, J.~Hoffman, and K.~Saenko, ``Syn2real:
  A new benchmark forsynthetic-to-real visual domain adaptation,'' \emph{CoRR},
  vol. abs/1806.09755, 2018.

\bibitem{visda}
X.~Peng, B.~Usman, N.~Kaushik, J.~Hoffman, D.~Wang, and K.~Saenko, ``Visda: The
  visual domain adaptation challenge,'' \emph{arXiv preprint arXiv:1710.06924},
  2017.

\bibitem{peng2018moment}
X.~Peng, Q.~Bai, X.~Xia, Z.~Huang, K.~Saenko, and B.~Wang, ``Moment matching
  for multi-source domain adaptation,'' \emph{arXiv preprint arXiv:1812.01754},
  2018.

\bibitem{reverse_grad}
Y.~Ganin and V.~S. Lempitsky, ``Unsupervised domain adaptation by
  backpropagation,'' in \emph{Proceedings of the 32nd International Conference
  on Machine Learning, {ICML} 2015, Lille, France, 6-11 July 2015}, 2015, pp.
  1180--1189.

\bibitem{caltech}
B.~Gong, Y.~Shi, F.~Sha, and K.~Grauman, ``Geodesic flow kernel for
  unsupervised domain adaptation,'' in \emph{Computer Vision and Pattern
  Recognition (CVPR), 2012 IEEE Conference on}.\hskip 1em plus 0.5em minus
  0.4em\relax IEEE, 2012, pp. 2066--2073.

\bibitem{imagenet}
O.~Russakovsky, J.~Deng, H.~Su, J.~Krause, S.~Satheesh, S.~Ma, Z.~Huang,
  A.~Karpathy, A.~Khosla, M.~Bernstein \emph{et~al.}, ``Imagenet large scale
  visual recognition challenge,'' \emph{International Journal of Computer
  Vision}, vol. 115, no.~3, pp. 211--252, 2015.

\bibitem{bn}
S.~Ioffe and C.~Szegedy, ``Batch normalization: Accelerating deep network
  training by reducing internal covariate shift,'' \emph{arXiv preprint
  arXiv:1502.03167}, 2015.

\bibitem{lecun1998gradient}
Y.~LeCun, L.~Bottou, Y.~Bengio, P.~Haffner \emph{et~al.}, ``Gradient-based
  learning applied to document recognition,'' \emph{Proceedings of the IEEE},
  vol.~86, no.~11, pp. 2278--2324, 1998.

\bibitem{resnet}
K.~He, X.~Zhang, S.~Ren, and J.~Sun, ``Deep residual learning for image
  recognition,'' in \emph{Proceedings of the IEEE conference on computer vision
  and pattern recognition}, 2016, pp. 770--778.

\bibitem{tpn}
Y.~Pan, T.~Yao, Y.~Li, Y.~Wang, C.-W. Ngo, and T.~Mei, ``Transferrable
  prototypical networks for unsupervised domain adaptation,'' in
  \emph{Proceedings of the IEEE Conference on Computer Vision and Pattern
  Recognition}, 2019, pp. 2239--2247.

\bibitem{can}
G.~Kang, L.~Jiang, Y.~Yang, and A.~G. Hauptmann, ``Contrastive adaptation
  network for unsupervised domain adaptation,'' in \emph{Proceedings of the
  IEEE Conference on Computer Vision and Pattern Recognition}, 2019, pp.
  4893--4902.

\bibitem{rtn}
M.~Long, H.~Zhu, J.~Wang, and M.~I. Jordan, ``Unsupervised domain adaptation
  with residual transfer networks,'' in \emph{Advances in Neural Information
  Processing Systems}, 2016, pp. 136--144.

\bibitem{SimNet}
P.~O. Pinheiro and A.~Element, ``Unsupervised domain adaptation with similarity
  learning,'' in \emph{Proceedings of the IEEE Conference on Computer Vision
  and Pattern Recognition}, 2018, pp. 8004--8013.

\bibitem{attention_alignment}
G.~Kang, L.~Zheng, Y.~Yan, and Y.~Yang, ``Deep adversarial attention alignment
  for unsupervised domain adaptation: the benefit of target expectation
  maximization,'' in \emph{Proceedings of the European Conference on Computer
  Vision (ECCV)}, 2018, pp. 401--416.

\bibitem{tada}
X.~Wang, L.~Li, W.~Ye, M.~Long, and J.~Wang, ``Transferable attention for
  domain adaptation,'' in \emph{Proceedings of the AAAI Conference on
  Artificial Intelligence}, vol.~33, 2019, pp. 5345--5352.

\bibitem{Kurmi_2019_CVPR}
V.~K. Kurmi, S.~Kumar, and V.~P. Namboodiri, ``Attending to discriminative
  certainty for domain adaptation,'' in \emph{The IEEE Conference on Computer
  Vision and Pattern Recognition (CVPR)}, June 2019.

\bibitem{sne}
L.~v.~d. Maaten and G.~Hinton, ``Visualizing data using t-sne,'' \emph{Journal
  of machine learning research}, vol.~9, no. Nov, pp. 2579--2605, 2008.

\bibitem{liu2019separate}
H.~Liu, Z.~Cao, M.~Long, J.~Wang, and Q.~Yang, ``Separate to adapt: Open set
  domain adaptation via progressive separation,'' in \emph{Proceedings of the
  IEEE Conference on Computer Vision and Pattern Recognition}, 2019, pp.
  2927--2936.

\bibitem{gcan}
X.~Ma, T.~Zhang, and C.~Xu, ``Gcan: Graph convolutional adversarial network for
  unsupervised domain adaptation,'' in \emph{The IEEE Conference on Computer
  Vision and Pattern Recognition (CVPR)}, June 2019.

\bibitem{importance_weight}
J.~Zhang, Z.~Ding, W.~Li, and P.~Ogunbona, ``Importance weighted adversarial
  nets for partial domain adaptation,'' in \emph{Proceedings of the IEEE
  Conference on Computer Vision and Pattern Recognition}, 2018, pp. 8156--8164.

\end{thebibliography}


% Generated by IEEEtran.bst, version: 1.14 (2015/08/26)
\begin{thebibliography}{10}
\providecommand{\url}[1]{#1}
\csname url@samestyle\endcsname
\providecommand{\newblock}{\relax}
\providecommand{\bibinfo}[2]{#2}
\providecommand{\BIBentrySTDinterwordspacing}{\spaceskip=0pt\relax}
\providecommand{\BIBentryALTinterwordstretchfactor}{4}
\providecommand{\BIBentryALTinterwordspacing}{\spaceskip=\fontdimen2\font plus
\BIBentryALTinterwordstretchfactor\fontdimen3\font minus
  \fontdimen4\font\relax}
\providecommand{\BIBforeignlanguage}[2]{{%
\expandafter\ifx\csname l@#1\endcsname\relax
\typeout{** WARNING: IEEEtran.bst: No hyphenation pattern has been}%
\typeout{** loaded for the language `#1'. Using the pattern for}%
\typeout{** the default language instead.}%
\else
\language=\csname l@#1\endcsname
\fi
#2}}
\providecommand{\BIBdecl}{\relax}
\BIBdecl

\bibitem{mohri2012foundations}
M.~Mohri, A.~Rostamizadeh, and A.~Talwalkar, \emph{Foundations of machine
  learning}.\hskip 1em plus 0.5em minus 0.4em\relax MIT press, 2012.

\bibitem{office_31}
K.~Saenko, B.~Kulis, M.~Fritz, and T.~Darrell, ``Adapting visual category
  models to new domains,'' in \emph{European conference on computer
  vision}.\hskip 1em plus 0.5em minus 0.4em\relax Springer, 2010, pp. 213--226.

\bibitem{ImageCLEFDA}
``Imageclef-da dataset,'' \url{http://imageclef.org/2014/adaptation/}.

\bibitem{office_home}
H.~Venkateswara, J.~Eusebio, S.~Chakraborty, and S.~Panchanathan, ``Deep
  hashing network for unsupervised domain adaptation,'' in \emph{Proc. CVPR},
  2017, pp. 5018--5027.

\bibitem{syn2real}
X.~Peng, B.~Usman, K.~Saito, N.~Kaushik, J.~Hoffman, and K.~Saenko, ``Syn2real:
  A new benchmark forsynthetic-to-real visual domain adaptation,'' \emph{CoRR},
  vol. abs/1806.09755, 2018.

\bibitem{visda}
X.~Peng, B.~Usman, N.~Kaushik, J.~Hoffman, D.~Wang, and K.~Saenko, ``Visda: The
  visual domain adaptation challenge,'' \emph{arXiv preprint arXiv:1710.06924},
  2017.

\bibitem{mnist}
Y.~LeCun, L.~Bottou, Y.~Bengio, and P.~Haffner, ``Gradient-based learning
  applied to document recognition,'' \emph{Proceedings of the IEEE}, vol.~86,
  no.~11, pp. 2278--2324, 1998.

\bibitem{svhn}
Y.~Netzer, T.~Wang, A.~Coates, A.~Bissacco, B.~Wu, and A.~Y. Ng, ``Reading
  digits in natural images with unsupervised feature learning,'' in \emph{NIPS
  workshop on deep learning and unsupervised feature learning}, vol. 2011,
  no.~2, 2011, p.~5.

\bibitem{usps}
J.~J. Hull, ``A database for handwritten text recognition research,''
  \emph{IEEE Transactions on pattern analysis and machine intelligence},
  vol.~16, no.~5, pp. 550--554, 1994.

\bibitem{adr}
K.~Saito, Y.~Ushiku, T.~Harada, and K.~Saenko, ``Adversarial dropout
  regularization,'' \emph{arXiv preprint arXiv:1711.01575}, 2017.

\bibitem{adda}
E.~Tzeng, J.~Hoffman, K.~Saenko, and T.~Darrell, ``Adversarial discriminative
  domain adaptation,'' in \emph{Computer Vision and Pattern Recognition
  (CVPR)}, vol.~1, no.~2, 2017, p.~4.

\bibitem{peng2018moment}
X.~Peng, Q.~Bai, X.~Xia, Z.~Huang, K.~Saenko, and B.~Wang, ``Moment matching
  for multi-source domain adaptation,'' \emph{arXiv preprint arXiv:1812.01754},
  2018.

\bibitem{resnet}
K.~He, X.~Zhang, S.~Ren, and J.~Sun, ``Deep residual learning for image
  recognition,'' in \emph{Proceedings of the IEEE conference on computer vision
  and pattern recognition}, 2016, pp. 770--778.

\bibitem{reverse_grad}
Y.~Ganin and V.~S. Lempitsky, ``Unsupervised domain adaptation by
  backpropagation,'' in \emph{Proceedings of the 32nd International Conference
  on Machine Learning, {ICML} 2015, Lille, France, 6-11 July 2015}, 2015, pp.
  1180--1189.

\bibitem{dann}
Y.~Ganin, E.~Ustinova, H.~Ajakan, P.~Germain, H.~Larochelle, M.~Marchand, and
  V.~Lempitsky, ``Domain-adversarial training of neural networks,''
  \emph{Journal of Machine Learning Research}, vol.~17, no.~1, pp. 2096--2030,
  2017.

\bibitem{mdd}
Y.~Zhang, T.~Liu, M.~Long, and M.~Jordan, ``Bridging theory and algorithm for
  domain adaptation,'' in \emph{International Conference on Machine Learning},
  2019, pp. 7404--7413.

\bibitem{mcd}
K.~Saito, K.~Watanabe, Y.~Ushiku, and T.~Harada, ``Maximum classifier
  discrepancy for unsupervised domain adaptation,'' in \emph{Proceedings of the
  IEEE Conference on Computer Vision and Pattern Recognition}, 2018, pp.
  3723--3732.

\bibitem{sne}
L.~v.~d. Maaten and G.~Hinton, ``Visualizing data using t-sne,'' \emph{Journal
  of machine learning research}, vol.~9, no. Nov, pp. 2579--2605, 2008.

\bibitem{alexnet}
A.~Krizhevsky, I.~Sutskever, and G.~E. Hinton, ``Imagenet classification with
  deep convolutional neural networks,'' in \emph{Advances in neural information
  processing systems}, 2012, pp. 1097--1105.

\bibitem{imagenet}
O.~Russakovsky, J.~Deng, H.~Su, J.~Krause, S.~Satheesh, S.~Ma, Z.~Huang,
  A.~Karpathy, A.~Khosla, M.~Bernstein \emph{et~al.}, ``Imagenet large scale
  visual recognition challenge,'' \emph{International Journal of Computer
  Vision}, vol. 115, no.~3, pp. 211--252, 2015.

\bibitem{dan}
\BIBentryALTinterwordspacing
M.~Long, Y.~Cao, J.~Wang, and M.~I. Jordan, ``Learning transferable features
  with deep adaptation networks,'' in \emph{Proceedings of the 32Nd
  International Conference on International Conference on Machine Learning -
  Volume 37}, ser. ICML'15.\hskip 1em plus 0.5em minus 0.4em\relax JMLR.org,
  2015, pp. 97--105. [Online]. Available:
  \url{http://dl.acm.org/citation.cfm?id=3045118.3045130}
\BIBentrySTDinterwordspacing

\bibitem{cada}
M.~Long, Z.~CAO, J.~Wang, and M.~I. Jordan, ``Conditional adversarial domain
  adaptation,'' in \emph{Advances in Neural Information Processing Systems 31},
  2018, pp. 1640--1650.

\bibitem{symnets}
Y.~Zhang, H.~Tang, K.~Jia, and M.~Tan, ``Domain-symmetric networks for
  adversarial domain adaptation,'' in \emph{Proceedings of the IEEE Conference
  on Computer Vision and Pattern Recognition}, 2019, pp. 5031--5040.

\bibitem{gcan}
X.~Ma, T.~Zhang, and C.~Xu, ``Gcan: Graph convolutional adversarial network for
  unsupervised domain adaptation,'' in \emph{The IEEE Conference on Computer
  Vision and Pattern Recognition (CVPR)}, June 2019.

\bibitem{san}
Z.~Cao, M.~Long, J.~Wang, and M.~I. Jordan, ``Partial transfer learning with
  selective adversarial networks,'' in \emph{The IEEE Conference on Computer
  Vision and Pattern Recognition (CVPR)}, June 2018.

\bibitem{importance_weight}
J.~Zhang, Z.~Ding, W.~Li, and P.~Ogunbona, ``Importance weighted adversarial
  nets for partial domain adaptation,'' in \emph{Proceedings of the IEEE
  Conference on Computer Vision and Pattern Recognition}, 2018, pp. 8156--8164.

\end{thebibliography}

% <OR> manually copy in the resultant .bbl file
% set second argument of \begin to the number of references
% (used to reserve space for the reference number labels box)
%\begin{thebibliography}{1}
%
%\bibitem{IEEEhowto:kopka}
%H.~Kopka and P.~W. Daly, \emph{A Guide to \LaTeX}, 3rd~ed.\hskip 1em plus
%  0.5em minus 0.4em\relax Harlow, England: Addison-Wesley, 1999.
%
%\end{thebibliography}

% biography section
%
% If you have an EPS/PDF photo (graphicx package needed) extra braces are
% needed around the contents of the optional argument to biography to prevent
% the LaTeX parser from getting confused when it sees the complicated
% \includegraphics command within an optional argument. (You could create
% your own custom macro containing the \includegraphics command to make things
% simpler here.)
%\begin{IEEEbiography}[{\includegraphics[width=1in,height=1.25in,clip,keepaspectratio]{mshell}}]{Michael Shell}
% or if you just want to reserve a space for a photo:

\begin{IEEEbiography}[{\includegraphics[width=1in,height=1.25in,clip,keepaspectratio]{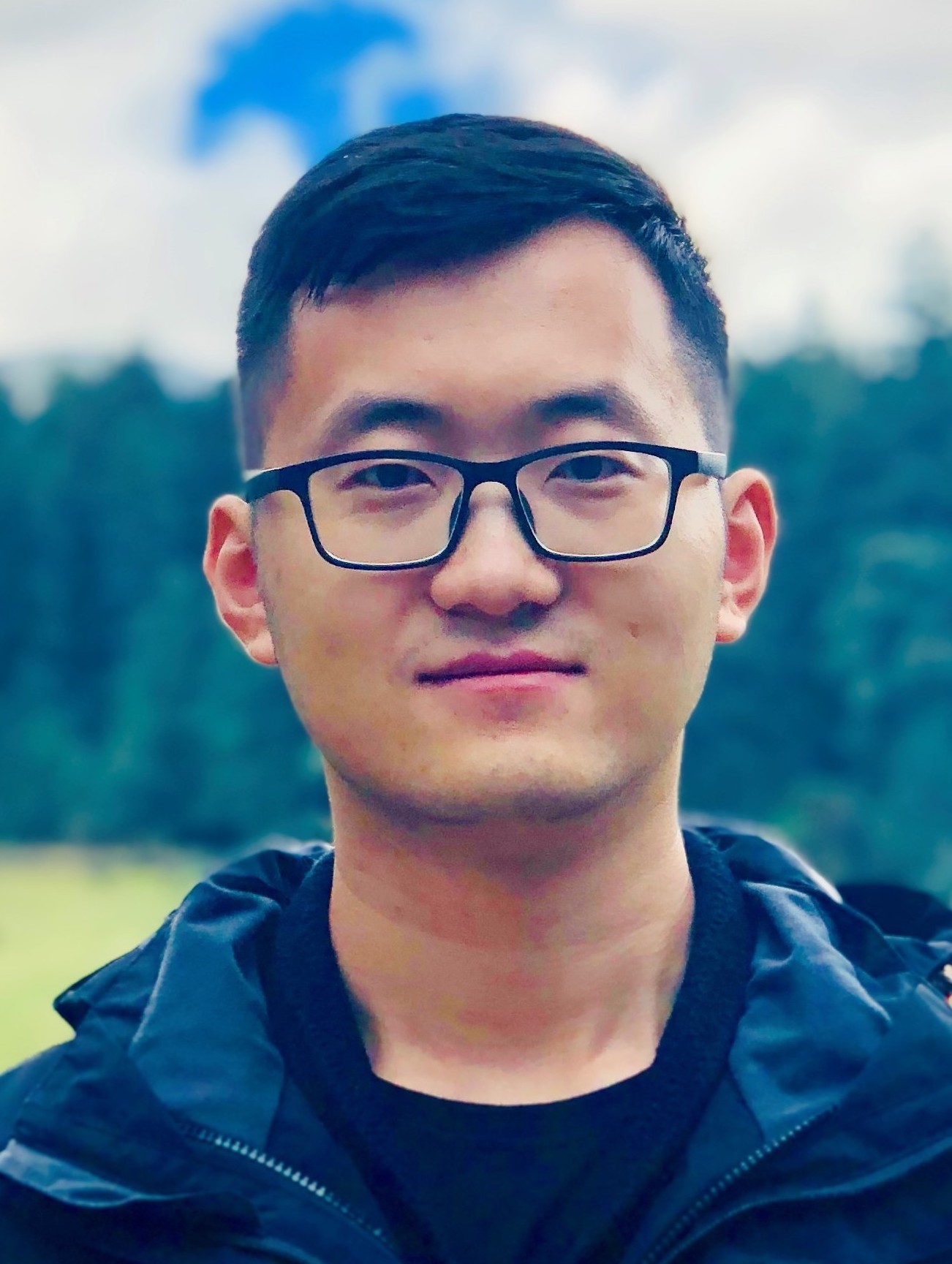}}]{Yabin Zhang}
	received the B.E. degree in School of Electronic and Information Engineering from South China University of Technology, Guangzhou, China, in 2017, where he is currently pursuing the master's degree. His current research interests include computer vision and deep learning, especially the deep transfer learning.
\end{IEEEbiography}

\begin{IEEEbiography}[{\includegraphics[width=1in,height=1.25in,clip,keepaspectratio]{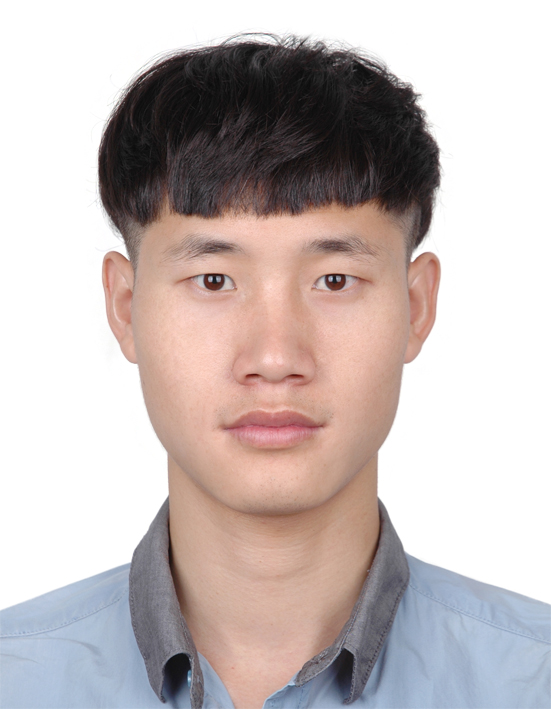}}]{Bin Deng}
	received the B.S. degree in information and computing science from South China Agricultural University, Guangzhou, China, in 2015, and the M.S. degree in pattern recognition and intelligent system from Shenzhen University, Shenzhen, China, in 2018.
	
	He is currently pursuing the Ph.D. degree in the School of Electronic and Information Engineering, South China University of Technology. His research interests include machine learning, pattern recognition, and hyperspectral image processing.
\end{IEEEbiography}

\begin{IEEEbiography}[{\includegraphics[width=1in,height=1.25in,clip,keepaspectratio]{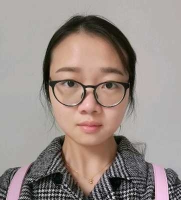}}]{Hui Tang}
	received the B.E. degree in the School of Electronic and Information Engineering, South China University of Technology, in 2018. She is currently pursuing the Ph.D. degree in the School of Electronic and Information Engineering, South China University of Technology. Her research interests are in computer vision and machine learning.
\end{IEEEbiography}

\begin{IEEEbiography}[{\includegraphics[width=1in,height=1.25in,clip,keepaspectratio]{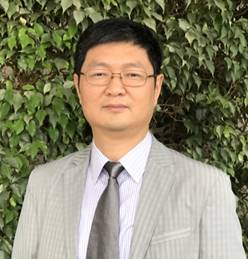}}]{Lei Zhang}
	(M’04, SM’14, F’18) received his B.Sc. degree in 1995 from Shenyang Institute of Aeronautical Engineering, Shenyang, P.R. China, and M.Sc. and Ph.D degrees in Control Theory and Engineering from Northwestern Polytechnical University, Xi’an, P.R. China, in 1998 and 2001, respectively. From 2001 to 2002, he was a research associate in the Department of Computing, The Hong Kong Polytechnic University. From January 2003 to January 2006 he worked as a Postdoctoral Fellow in the Department of Electrical and Computer Engineering, McMaster University, Canada. In 2006, he joined the Department of Computing, The Hong Kong Polytechnic University, as an Assistant Professor. Since July 2017, he has been a Chair Professor in the same department. His research interests include Computer Vision, Image and Video Analysis, Pattern Recognition, and Biometrics, etc. Prof. Zhang has published more than 200 papers in those areas. As of 2020, his publications have been cited more than 52,000 times in literature. Prof. Zhang is a Senior Associate Editor of IEEE Trans. on Image Processing, and is/was an Associate Editor of IEEE Trans. on Pattern Analysis and Machine Intelligence, SIAM Journal of Imaging Sciences, IEEE Trans. on CSVT, and Image and Vision Computing, etc. He is a “Clarivate Analytics Highly Cited Researcher” from 2015 to 2019. More information can be found in his homepage \url{http://www4.comp.polyu.edu.hk/~cslzhang/}. 
\end{IEEEbiography}

\begin{IEEEbiography}[{\includegraphics[width=1in,height=1.25in,clip,keepaspectratio]{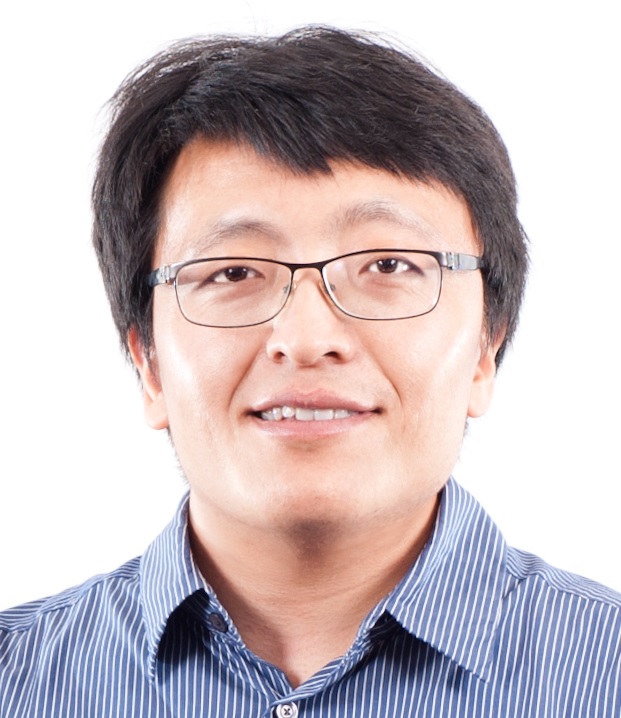}}]{Kui Jia}
	received the B.E. degree from Northwestern Polytechnic University, Xi’an, China, in 2001, the M.E. degree from the National University of Singapore, Singapore, in 2004, and the Ph.D. degree in computer science from the Queen Mary University of London, London, U.K., in 2007.
	
	He was with the Shenzhen Institute of Advanced Technology of the Chinese Academy of Sciences, Shenzhen, China, Chinese University of Hong Kong, Hong Kong, the Institute of Advanced Studies, University of Illinois at Urbana-Champaign, Champaign, IL, USA, and the University of Macau, Macau, China. He is currently a Professor with the School of Electronic and Information Engineering, South China University of Technology, Guangzhou, China. His recent research focuses on theoretical deep learning and its applications in vision and robotic problems, including deep learning of 3D data and deep transfer learning.
\end{IEEEbiography}

% insert where needed to balance the two columns on the last page with
% biographies
%\newpage

%\begin{IEEEbiographynophoto}{Jane Doe}
%Biography text here.
%\end{IEEEbiographynophoto}

% You can push biographies down or up by placing
% a \vfill before or after them. The appropriate
% use of \vfill depends on what kind of text is
% on the last page and whether or not the columns
% are being equalized.

%\vfill

% Can be used to pull up biographies so that the bottom of the last one
% is flush with the other column.
%\enlargethispage{-5in}

\clearpage

\appendices

\newtheorem{proposition}{Proposition}[section]
\newtheorem{lemma}{Lemma}[section]
\newtheorem{theo}{Theorem}[section]
\section{Proof of Theorem \textcolor{red}{\textbf{1}}}\label{appendix-theorem1}

We begin with the following lemmas to prove the Theorem \textcolor{red}{1}.

\begin{lemma}\label{lemma-1}
	Fix $\rho>0$. For any scoring functions $\f, \f'\in \mathcal{F}$, the following holds for any distribution $D$ over $\mathcal{X\times Y}$,
	\begin{equation}
	\mathcal{E}_D(h_{\f}) \leq \mathcal{E}^{(\rho)}_D(\f') + \text{\rm MCSD}_{D_x}^{(\rho)}(\f, \f')
	\end{equation}
	where
	\begin{equation}
	\mathcal{E}_D(h_{\f}) := \mathbb{E}_{(\x,y)\sim D} \mathbb{I}[h_{\f}(\x)\neq y] ,
	\end{equation}
	\begin{equation}
	\mathcal{E}^{(\rho)}_D(\f') := \mathbb{E}_{(\x,y)\sim D} \sum_{k=1}^K \Phi_{\rho}( \mu_k(\f'(\x),y) )  .
	\end{equation}
\end{lemma}
\begin{proof}
	To prove the above inequality, we only need to prove that for any $(\x,y)\sim D$ and $\f, \f'\in \mathcal{F}$, the inequality
	\begin{equation}\label{equ-bound1}
	\begin{aligned}
	\mathbb{I}[h_{\f}(\x)\neq y] \leq  & L^{(\rho)}(\f'(\x),y) + \\
	& \frac{1}{K}\|\M^{(\rho)}(\f(\x)) - \M^{(\rho)}(\f'(\x))\|_1
	\end{aligned}
	\end{equation}
	holds, where $L^{(\rho)}(\f'(\x),y) = \sum_{k=1}^K\Phi_{\rho}(\mu_k(\f'(\x),y))$.
	
	We prove in the following that the inequality (\ref{equ-bound1}) holds in three separate cases, depending on the relationship between $h_{\f}(\x)$, $h_{\f'}(\x)$, and the class label $y$. For convenience, we also denote $h_{\f}(\x) = y_h$ and $h_{\f'}(\x) = y'_h$.
	
	\noindent\emph{Case 1:} When $h_{\f}(\x) = y_h = y$, no matter whether $h_{\f'}(\x) = y'_h = y$ or not, we have $\mathbb{I}(h_{\f}(\x)\neq y) = 0$,  and the inequality (\ref{equ-bound1}) holds obviously.
	
	\noindent\emph{Case 2:} When $h_{\f}(\x) = y_h \neq y$ and $h_{\f'}(\x) = y'_h \neq y$, due to the sum-to-zero constraint of $\sum_{k\in \mathcal{Y}}f'_k(\x)=0$, we have $f'_{y'_h}(\x)>=0$, and therefore
	\begin{multline*}
	L^{(\rho)}(\f'(\x),y)\geq \\ \Phi_\rho(\mu_{y'_h}(\f'(\x),y)) =  \Phi_\rho(-f'_{y'_h}(\x)) = 1 = \mathbb{I}[h_{\f}(\x)\neq y] .
	\end{multline*}
	Then the inequality (\ref{equ-bound1}) holds.
	
	\noindent\emph{Case 3:} When $h_{\f}(\x) = y_h \neq y$ and $h_{\f'}(\x) = y'_h = y$, we first show that when there exists $k\neq y$ such that $f'_k(\x) \geq 0$, we have $L^{(\rho)}(\f'(\x),y)\geq 1$, resulting in (\ref{equ-bound1}) directly; meanwhile, we show that when $f'_k(\x) < 0$ for all $k\neq y$ and $f'_y(\x) \leq \rho$, we have
	\begin{multline*}
	L^{(\rho)}(\f'(\x),y) = \sum_{k\neq y}[1 + \frac{f'_k(\x)}{\rho}] + 1 -  \frac{f'_y(\x)}{\rho} = \\ K-1-\frac{f'_y(\x)}{\rho} + 1 -  \frac{f'_y(\x)}{\rho} = K-2\frac{f'_y(\x)}{\rho} \geq 1,
	\end{multline*}
	and the inequality (\ref{equ-bound1}) holds; we proceed to discuss under the conditions of $f'_k(\x) < 0$ for all $k\neq y$ and $f'_y(\x) > \rho$.
	\begin{enumerate}
		\item Consider that $f'_{y_h}(\x) \leq -\rho$, we discuss under the separate conditions of either $f_{y_h}(\x) \geq \rho$ or $0\leq f_{y_h}(\x) < \rho$. When $f_{y_h}(\x) \geq \rho$, we have
		\begin{equation*}
		\begin{aligned}
		&\frac{1}{K}\|\M^{(\rho)}(\f(\x)) - \M^{(\rho)}(\f'(\x))\|_1 \\
		&\geq \frac{1}{K}[\sum_{k\neq y_h}|\Phi_\rho(-f_{y_h}(\x)) - \Phi_\rho(-f'_{y_h}(\x))| \\
		&+ \ |\Phi_\rho(f_{y_h}(\x)) - \Phi_\rho(f'_{y_h}(\x))|] = \frac{1}{K}[K - 1 + 1] = 1;
		\end{aligned}
		\end{equation*}
		when $0\leq f_{y_h}(\x) < \rho$, we further discuss under the separate conditions of either $f_y(\x) \leq 0$ or $f_y(\x) > 0$. When $f_y(\x)\leq 0$, we have
		\begin{equation*}
		\begin{aligned}
		& \frac{1}{K}\|\M^{(\rho)}(\f(\x)) - \M^{(\rho)}(\f'(\x))\|_1 \\
		& \geq \frac{1}{K}[\sum_{k\neq y_h}|\Phi_\rho(-f_{y_h}(\x)) - \Phi_\rho(-f'_{y_h}(\x))| + \\
		& |\Phi_\rho(f'_y(\x)) - \Phi_\rho(f_y(\x))|] = \frac{1}{K}[K-1 + 1] = 1;
		\end{aligned}
		\end{equation*}
		when $f_y(\x) > 0$, we have
		\begin{equation*}
		\begin{aligned}
		& \frac{1}{K}\|\M^{(\rho)}(\f(\x)) - \M^{(\rho)}(\f'(\x))\|_1 \\
		& \geq \frac{1}{K}[\sum_{k\neq y_h}|\Phi_\rho(-f_{y_h}(\x)) - \Phi_\rho(-f'_{y_h}(\x))| \\
		& \ \ + |\Phi_\rho(f_{y_h}(\x)) - \Phi_\rho(f'_{y_h}(\x))| + |\Phi_\rho(f'_y(\x)) - \Phi_\rho(f_y(\x))|] \\
		& = \frac{1}{K}[K-1 + 1 + \frac{f_{y_h}(\x)}{\rho} - \frac{f_y(\x)}{\rho}]\geq 1.
		\end{aligned}
		\end{equation*}
		Therefore, the inequality (\ref{equ-bound1}) holds.
		\item Consider that $-\rho < f'_{y_h}(\x) < 0$, we discuss in the conditions where both $0 \leq f_{y_h}(\x) < \rho$ and $f_{y_h}(\x) \geq f_y(\x) > 0$ are met, or either of them is not met. When $f_{y_h}(\x) \geq \rho$ or $f_y(\x) \leq 0$, we have
		\begin{equation*}
		\begin{aligned}
		& \frac{1}{K}\|\M^{(\rho)}(\f(\x)) - \M^{(\rho)}(\f'(\x))\|_1 \\
		& \geq \frac{1}{K}[\sum_{k\neq y_h}|\Phi_\rho(-f_{y_h}(\x)) - \Phi_\rho(-f'_{y_h}(\x))| \\
		& \ \ + |\Phi_\rho(f_{y_h}(\x)) - \Phi_\rho(f'_{y_h}(\x))| + |\Phi_\rho(f_{y}(\x)) - \Phi_\rho(f'_{y}(\x))|] \\
		& = \frac{1}{K}[-(K-1)\frac{f'_{y_h}(\x)}{\rho} + |\Phi_\rho(f_{y_h}(\x)) - \Phi_\rho(f'_{y_h}(\x))| \\
		& \ \ + |\Phi_\rho(f_{y}(\x)) - \Phi_\rho(f'_{y}(\x))|] \\
		& \geq  \frac{1}{K}[-(K-1)\frac{f'_{y_h}(\x)}{\rho} + 1] \\
		& \geq \frac{1}{K}[-K\frac{f'_{y_h}(\x)}{\rho}] = -\frac{f'_{y_h}(\x)}{\rho};
		\end{aligned}
		\end{equation*}
		when $0 \leq f_{y_h}(\x) < \rho$ and $f_y(\x) > 0$, we have
		\begin{equation*}
		\begin{aligned}
		& \frac{1}{K}\|\M^{(\rho)}(\f(\x)) - \M^{(\rho)}(\f'(\x))\|_1 \\
		& \geq \frac{1}{K}[\sum_{k\neq y_h}|\Phi_\rho(-f_{y_h}(\x)) - \Phi_\rho(-f'_{y_h}(\x))| \\
		& \ + |\Phi_\rho(f_{y_h}(\x)) - \Phi_\rho(f'_{y_h}(\x))| + |\Phi_\rho(f_{y}(\x)) - \Phi_\rho(f'_{y}(\x))|] \\
		& \geq \frac{1}{K}[-(K-1)\frac{f'_{y_h}(\x)}{\rho} + |\Phi_\rho(f_{y_h}(\x)) - \Phi_\rho(f'_{y_h}(\x))| \\ & \ + |\Phi_\rho(f_{y}(\x)) - \Phi_\rho(f'_{y}(\x))|] \\
		& \geq \frac{1}{K}[-(K-1)\frac{f'_{y_h}(\x)}{\rho} + 1 + \frac{f_{y_h}(\x)}{\rho} - \frac{f_y(\x)}{\rho}] \\ & \geq \frac{1}{K}[-K\frac{f'_{y_h}(\x)}{\rho}]\geq -\frac{f'_{y_h}(\x)}{\rho}.
		\end{aligned}
		\end{equation*}
		Therefore, $\frac{1}{K}\|\M^{(\rho)}(\f(\x)) - \M^{(\rho)}(\f'(\x))\|_1 \geq -\frac{f'_{y_h}(\x)}{\rho}$ holds. At the same time, we also have $L^{(\rho)}(\f'(\x),y) \geq 1 + \frac{f'_{y_h}(\x)}{\rho}$, thus the inequality (\ref{equ-bound1}) holds.
	\end{enumerate}
	The proof is finished. \hfill{\QEDclosed}
\end{proof} %\rightline{\QEDclosed}

\begin{lemma}\label{lemma-2}
	Fix $\rho>0$. For any scoring functions $\f, \f'\in \mathcal{F}$, the following holds for any distribution $D$ over $\mathcal{X\times Y}$,
	\begin{equation}
	\text{\rm MCSD}_{D_x}^{(\rho)}(\f,\f') \leq \mathcal{E}^{(\rho)}_D(\f) + \mathcal{E}^{(\rho)}_D(\f')
	\end{equation}
	where
	\begin{equation}
	\mathcal{E}^{(\rho)}_D(\f) := \mathbb{E}_{(\x,y)\sim D} \sum_{k=1}^K \Phi_{\rho}( \mu_k(\f(\x),y) ).
	\end{equation}
\end{lemma}
\begin{proof}
	To prove the above inequality, we only need to prove that for any $(\x,y)\sim D$ and $\f, \f'\in \mathcal{F}$, the inequality
	\begin{multline}\label{equ-bound2}
	%\begin{aligned}
	\frac{1}{K}\|\M^{(\rho)}(\f(\x)) - \M^{(\rho)}(\f'(\x))\|_1 \\
	\qquad \leq L^{(\rho)}(\f(\x),y) + L^{(\rho)}(\f'(\x),y)
	%\end{aligned}
	\end{multline}
	holds, where $L^{(\rho)}(\f(\x),y) = \sum_{k=1}^K\Phi_{\rho}(\mu_k(\f(\x),y))$.
	
	Before proving (\ref{equ-bound2}), we first show that the following inequality
	\begin{equation}\label{equ-help1}
	\begin{aligned}
	& |\Phi_{\rho}(f_{y'}(\x)) - \Phi_{\rho}(f'_{y'}(\x))| + \sum_{k\neq y'}|\Phi_{\rho}(-f_k(\x)) - \Phi_{\rho}(-f'_k(\x))| \\
	& \leq \Phi_{\rho}(f_y(\x)) + \Phi_{\rho}(f'_y(\x)) + \sum_{k\neq y}[\Phi_{\rho}(-f_k(\x)) + \Phi_{\rho}(-f'_k(\x))]
	\end{aligned}
	\end{equation}
	holds for any $y'\in \mathcal{Y}$. If $y' = y$, the inequality (\ref{equ-help1}) holds obviously. We then discuss in the following under the condition of $y' \neq y$. In this case, the left hand side of the inequality (\ref{equ-help1}) is equal to
	\begin{multline}\label{equ-help2}
	%\begin{aligned}
	%|\Phi_{\rho}(f_{y'}(\x)) - \Phi_{\rho}(f'_{y'}(\x))| + \sum_{k\neq y'}|\Phi_{\rho}(-f_k(\x)) - \Phi_{\rho}(-f'_k(\x))| =\\
	|\Phi_{\rho}(f_{y'}(\x)) - \Phi_{\rho}(f'_{y'}(\x))| + |\Phi_{\rho}(-f_{y}(\x)) - \Phi_{\rho}(-f'_{y}(\x))|  \\
	\qquad + \sum_{k\neq y, y'}|\Phi_{\rho}(-f_k(\x)) - \Phi_{\rho}(-f'_k(\x))|,
	%\end{aligned}
	\end{multline}
	and the right hand side of the inequality (\ref{equ-help1}) is equal to
	\begin{multline}\label{equ-help3}
	%	\Phi_{\rho}(f_{y}(\x)) + \Phi_{\rho}(f'_{y}(\x)) + \sum_{c\neq y}[\Phi_{\rho}(-f_c(\x)) + \Phi_{\rho}(-f'_c(\x))] = \\
	\Phi_{\rho}(-f_{y'}(\x)) + \Phi_{\rho}(-f'_{y'}(\x)) + \Phi_{\rho}(f_{y}(\x)) + \Phi_{\rho}(f'_{y}(\x))  \\
	\qquad + \sum_{k\neq y, y'}[\Phi_{\rho}(-f_k(\x)) + \Phi_{\rho}(-f'_k(\x))].
	\end{multline}
	By observing equations (\ref{equ-help2}) and (\ref{equ-help3}), it is obvious that the inequality
	\begin{multline}\label{equ-help4}
	\sum_{k\neq y, y'}|\Phi_{\rho}(-f_k(\x)) - \Phi_{\rho}(-f'_k(\x))| \\ \leq \sum_{k\neq y, y'}[\Phi_{\rho}(-f_k(\x)) + \Phi_{\rho}(-f'_k(\x))]
	\end{multline}
	holds. We then discuss in conditions where both $f_y(\x) > 0$ and $f'_y(\x) > 0$ are met, or either of them is not met. When $f_y(\x) \leq 0$ or $f'_y(\x) \leq 0$, we have
	\begin{multline*}
	\Phi_{\rho}(f_{y}(\x)) + \Phi_{\rho}(f'_{y}(\x))\\ \geq 1 \geq |\Phi_{\rho}(-f_{y}(\x)) - \Phi_{\rho}(-f'_{y}(\x))|;
	\end{multline*}
	when $f_y(\x) > 0$ and $f'_y(\x) > 0$, we have
	\begin{multline*}
	|\Phi_{\rho}(-f_{y}(\x)) - \Phi_{\rho}(-f'_{y}(\x))| \\ = |1 - 1|  = 0 \leq \Phi_{\rho}(f_{y}(\x)) + \Phi_{\rho}(f'_{y}(\x)).
	\end{multline*}
	Therefore, we have
	\begin{align}\label{equ-help5}
	|\Phi_{\rho}(-f_{y}(\x)) - \Phi_{\rho}(-f'_{y}(\x))| \leq \Phi_{\rho}(f_{y}(\x)) + \Phi_{\rho}(f'_{y}(\x)).
	\end{align}
	Similarly, we also have
	\begin{align}\label{equ-help6}
	|\Phi_{\rho}(f_{y'}(\x)) - \Phi_{\rho}(f'_{y'}(\x))| \leq \Phi_{\rho}(-f_{y'}(\x)) + \Phi_{\rho}(-f'_{y'}(\x)).
	\end{align}
	By combining the inequalities (\ref{equ-help4}), (\ref{equ-help5}), and (\ref{equ-help6}), we can get the result of the inequality (\ref{equ-help1}). Therefore, the inequality (\ref{equ-help1}) holds for any $y'\in \mathcal{Y}$.
	
	We now turn to prove that the inequality (\ref{equ-bound2}) holds. Based on the inequality of (\ref{equ-help1}), we therefore have
	\begin{equation*}
	\begin{aligned}
	& \|\M^{(\rho)}(\f(\x)) - \M^{(\rho)}(\f'(\x))\|_1 \\
	& = \sum_{y'\in \mathcal{Y}} [|\Phi_{\rho}(f_{y'}(\x)) - \Phi_{\rho}(f'_{y'}(\x))| \\
	& \qquad + \sum_{k\neq y'}|\Phi_{\rho}(-f_k(\x)) - \Phi_{\rho}(-f'_k(\x))|] \\
	& \leq K[\Phi_{\rho}(f_y(\x)) + \Phi_{\rho}(f'_y(\x)) \\
	& \qquad + \sum_{k\neq y}[\Phi_{\rho}(-f_k(\x)) + \Phi_{\rho}(-f'_k(\x))]] \\
	& = K[L^{(\rho)}(\f(\x),y) + L^{(\rho)}(\f'(\x),y)],
	\end{aligned}
	\end{equation*}
	thus resulting in (\ref{equ-bound2}) directly. The proof is finished. \hfill{\QEDclosed}
	%	\noindent (2). Let $y' = y_1$ be the index label such that equality (\ref{equ-help2}) (denote as $Eq1(y_1)$) reaches the maximum value and let $y = y_2$ be the index label such that equality (\ref{equ-help3}) (denote as $Eq2(y_2)$) reaches the minimum value. Based on the definition of $\M$ and $L$ and $Eq1(y_1) \leq Eq2(y_2)$, we therefore have
	%	\begin{multline*}
	%	\frac{1}{K}\|\M^{(\rho)}(\f(\x)) - \M^{(\rho)}(\f'(\x))\|_1 \\ \leq \frac{1}{K}[K*Eq1(y_1)] = Eq1(y_1) \leq Eq2(y_2) \\ \leq L^{(\rho)}(\f(\x),y) + L^{(\rho)}(\f'(\x),y),
	%	\end{multline*}
	%	thus (\ref{equ-bound2}) always holds for any $(\x,y)\sim D$. The proof is finished.
\end{proof}%\rightline{\QEDclosed{}}

%\begin{theo}[Theorem 1]\label{theorem-1}
%	Fix $\rho>0$. For any scoring functions $\f \in \mathcal{F}$, the source and target distributions $P$ and $Q$ over $\mathcal{X\times Y}$ and their corresponding marginal distributions $P_x$ and $Q_x$,
%	\begin{equation}
%	\mathcal{E}_Q(h_{\f}) \leq \mathcal{E}^{(\rho)}_P(\f) + d^{(\rho)}_\mathcal{F}(P_x, Q_x) + \lambda
%	\end{equation}
%	where $\lambda = \mathcal{E}^{(\rho)}_P(\f^*) + \mathcal{E}^{(\rho)}_Q(\f^*)$ is the minimum value with $\f^* \in \mathcal{F}$.
%\end{theo}
\begin{theo}[Theorem \textcolor{red}{1}]\label{theorem-1}
	Fix $\rho>0$. For any scoring function $\f \in \mathcal{F}$, the following holds over the source and target distributions $P$ and $Q$,
	\begin{equation}\label{appendix-EqnMCSDUDAGenBound}
	\mathcal{E}_Q(h_{\f}) \leq \mathcal{E}^{(\rho)}_P(\f) + d^{(\rho)}_{MCSD}(P_x, Q_x) + \lambda,
	\end{equation}
	where the constant $\lambda = \mathcal{E}^{(\rho)}_P(\f^*) + \mathcal{E}^{(\rho)}_Q(\f^*)$ with $\f^* = \arg\min\limits_{\f\in \mathcal{F}}\mathcal{E}^{(\rho)}_P(\f) + \mathcal{E}^{(\rho)}_Q(\f)$, and
	\begin{equation}
	\mathcal{E}_Q(h_{\f}) := \mathbb{E}_{(\x,y)\sim Q} \mathbb{I}[h_{\f}(\x)\neq y] ,
	\end{equation}
	\begin{equation}
	\mathcal{E}^{(\rho)}_P(\f) := \mathbb{E}_{(\x,y)\sim P} \sum_{k=1}^K \Phi_{\rho}( \mu_k(\f(\x),y) )  .
	\end{equation}
\end{theo}
\begin{proof}
	Based on the Lemma \ref{lemma-1} and Lemma \ref{lemma-2}, we have
	\begin{equation*}
	\begin{aligned}
	& \mathcal{E}_Q(h_{\f}) \\
	& \leq \mathcal{E}^{(\rho)}_Q(\f^*) + \text{\rm MCSD}_{Q_x}^{(\rho)}(\f, \f^*) \\
	& \leq \mathcal{E}^{(\rho)}_P(\f) + \mathcal{E}^{(\rho)}_P(\f^*) - \text{\rm MCSD}_{P_x}^{(\rho)}(\f,\f^*) \\
	& \qquad + \mathcal{E}^{(\rho)}_Q(\f^*) + \text{\rm MCSD}_{Q_x}^{(\rho)}(\f, \f^*) \\
	& = \mathcal{E}^{(\rho)}_P(\f) + \text{\rm MCSD}_{Q_x}^{(\rho)}(\f, \f^*) - \text{\rm MCSD}_{P_x}^{(\rho)}(\f,\f^*) + \lambda \\
	& \leq \mathcal{E}^{(\rho)}_P(\f) + \lambda \\
	&  \qquad + \sup\limits_{\f',\f''\in \mathcal{F}}[\text{\rm MCSD}_{Q_x}^{(\rho)}(\f', \f'') - \text{\rm MCSD}_{P_x}^{(\rho)}(\f',\f'')]  \\
	& = \mathcal{E}^{(\rho)}_P(\f) + d^{(\rho)}_{MCSD}(P_x, Q_x) + \lambda.
	\end{aligned}
	\end{equation*} \hfill{\QEDclosed}
\end{proof}%\rightline{\QEDclosed{}}

%\section{Degenerate Versions of MCSD}
\vspace{-0.8cm}
\section{Scalar-valued, Absolute Margin-based Divergences}
%\begin{prop}
%\label{TheoremMCSDDegen2HDHUDAGenBound}
%Fix $\rho>0$. For any scoring function $\f \in \mathcal{F}$,
%\begin{equation}\label{EqnMCSDDegen2HDHUDAGenBound}
%\mathcal{E}_Q(h_{\f}) \leq \mathcal{E}^{(\rho)}_P(\f) + d^{(\rho)}_{\widehat{MCSD}}(P_x, Q_x) + \lambda,
%\end{equation}
%\end{prop}
\begin{lemma}[Proposition \textcolor{red}{1}]\label{lemma-for-proposition-1}
	Fix $\rho>0$. For any scoring function $\f \in \mathcal{F}$, the following holds over the source and target distributions $P$ and $Q$,
	\begin{equation}\label{equ-degeneration1-bound}
	\mathcal{E}_Q(h_{\f}) \leq \mathcal{E}^{(\rho)}_P(\f) + d^{(\rho)}_{\widetilde{MCSD}}(P_x, Q_x) + \lambda,
	\end{equation}
	where $\mathcal{E}_Q(h_{\f})$, $\mathcal{E}^{(\rho)}_P(\f)$, and $\lambda$ are defined as the same as these in the Theorem \ref{theorem-1}.
\end{lemma}
\begin{proof}
	The proof follows the same argument as that of the Theorem \ref{theorem-1}. The only difference is that the term $d^{(\rho)}_{MCSD}(P_x, Q_x)$ is replaced by $d^{(\rho)}_{\widetilde{MCSD}}(P_x, Q_x)$. Therefore, to prove the above point, we only need to prove that
	\begin{equation}\label{equ-degegeneration1-bound-1}
	\mathcal{E}_D(h_{\f}) \leq \mathcal{E}^{(\rho)}_D(\f') + \widetilde{\text{\rm MCSD}}_{D_x}^{(\rho)}(\f, \f')
	\end{equation}
	and
	\begin{equation}\label{equ-degegeneration1-bound-2}
	\widetilde{\text{\rm MCSD}}_{D_x}^{(\rho)}(\f,\f') \leq \mathcal{E}^{(\rho)}_D(\f) + \mathcal{E}^{(\rho)}_D(\f')
	\end{equation}
	satisfy for any scoring functions $\f, \f'\in \mathcal{F}$ with respect to any distribution $D$ over $\mathcal{X\times Y}$. We now turn to prove (\ref{equ-degegeneration1-bound-1}) and (\ref{equ-degegeneration1-bound-2}) respectively in the following.
	
	To prove (\ref{equ-degegeneration1-bound-1}), we only need to prove that for any $(\x,y)\sim D$, the inequality
	\begin{multline}\label{equ-degeneration1-help1}
	\mathbb{I}(h_{\f}(\x)\neq y) \\
	\leq L^{(\rho)}(\f'(\x),y) + \Phi_{\rho/2}[\mu_{h_{\f'}(\x)}(\f'(\x),h_{\f}(\x))]
	\end{multline}
	holds, where $L^{(\rho)}(\f'(\x),y) = \sum_{k=1}^K\Phi_{\rho}(\mu_k(\f'(\x),y))$.
	If $h_{\f'}(\x)\neq h_{\f}(\x)$ or $h_{\f'}(\x) \neq y$, the right-hand side of the above inequality will reach the value of 1, which is obviously an upper bound of the left-hand side.
	Otherwise $h_{\f'}(\x) = h_{\f}(\x) = y$, and
	\begin{equation*}
	\mathbb{I}(h_{\f}(\x)\neq y) = 0 \leq L^{(\rho)}(\f'(\x),y).
	\end{equation*}
	Therefore, the inequality (\ref{equ-degeneration1-help1}) holds and then (\ref{equ-degegeneration1-bound-1}) holds.
	
	To prove (\ref{equ-degegeneration1-bound-2}), we only need to prove that for any $(\x,y)\sim D$, the inequality
	\begin{multline}\label{equ-degeneration1-help2}
	\Phi_{\rho/2}[\mu_{h_{\f'}(\x)}(\f'(\x),h_{\f}(\x))] \\ \leq L^{(\rho)}(\f(\x),y) + L^{(\rho)}(\f'(\x),y)
	\end{multline}
	holds. If $h_{\f'}(\x) \neq y$ or $h_{\f}(\x) \neq y$, the right-hand side of the above inequality will reach the value of 1, which is obviously an upper bound of the left-hand side.
	Otherwise $h_{\f'}(\x) = h_{\f}(\x) = y$, and
	\begin{multline*}
	\Phi_{\rho/2}[\mu_{h_{\f'}(\x)}(\f'(\x),h_{\f}(\x))] \\ \leq L^{(\rho)}(\f'(\x),y) \leq L^{(\rho)}(\f(\x),y) + L^{(\rho)}(\f'(\x),y).
	\end{multline*}
	Therefore, the inequality (\ref{equ-degeneration1-help2}) holds and then (\ref{equ-degegeneration1-bound-2}) holds. \hfill{\QEDclosed}
\end{proof} %\rightline{\QEDclosed{}}
%%%%%%%%%%%%%%%%%%%%%%%%%%%%%%%%%%%%%%%%%%%%%%%%%%%%%%%%%%

\begin{lemma}[Proposition \textcolor{red}{2}]\label{lemma-for-proposition-2}
	Fix $\rho>0$. For any scoring function $\f \in \mathcal{F}$, the following holds over the source and target distributions $P$ and $Q$,
	\begin{equation}\label{equ-degeneration2-bound}
	\mathcal{E}_Q(h_{\f}) \leq \mathcal{E}^{(\rho)}_P(\f) + d^{(\rho)}_{\widehat{MCSD}}(P_x, Q_x) + \lambda,
	\end{equation}
	where $\mathcal{E}_Q(h_{\f})$, $\mathcal{E}^{(\rho)}_P(\f)$, and $\lambda$ are defined as the same as these in Theorem \ref{theorem-1}.
\end{lemma}
\begin{proof}
	Following the similar proof of Lemma \ref{lemma-for-proposition-1}, to show the above result, we only need to prove that
	\begin{equation}\label{equ-degegeneration2-bound-1}
	\mathcal{E}_D(h_{\f}) \leq \mathcal{E}^{(\rho)}_D(\f') + \widehat{\text{\rm MCSD}}_{D_x}^{(\rho)}(\f, \f')
	\end{equation}
	and
	\begin{equation}\label{equ-degegeneration2-bound-2}
	\widehat{\text{\rm MCSD}}_{D_x}^{(\rho)}(\f,\f') \leq \mathcal{E}^{(\rho)}_D(\f) + \mathcal{E}^{(\rho)}_D(\f')
	\end{equation}
	satisfy for any scoring functions $\f, \f'\in \mathcal{F}$ with respect to any distribution $D$ over $\mathcal{X\times Y}$. We now turn to prove (\ref{equ-degegeneration2-bound-1}) and (\ref{equ-degegeneration2-bound-2}) respectively in the following.
	To prove (\ref{equ-degegeneration2-bound-1}), we only need to prove that for any $(\x,y)\sim D$, the inequality
	\begin{align*}
	& \ \mathbb{I}(h_{\f}(\x)\neq y) \\
	& \leq L^{(\rho)}(\f'(\x),y) + \mathbb{I}[\Phi_{\rho}[\mu_{h_{\f'}(\x)}(\f'(\x),h_{\f}(\x))] = 1]
	\end{align*}
	holds, where $L^{(\rho)}(\f'(\x),y) = \sum_{k=1}^K\Phi_{\rho}(\mu_k(\f'(\x),y))$. If $h_{\f'}(\x)\neq h_{\f}(\x)$ or $h_{\f'}(\x) \neq y$, the right-hand side of the above inequality will reach the value of 1, which is obviously an upper bound of the left-hand side. Otherwise $h_{\f'}(\x) = h_{\f}(\x) = y$, and
	\begin{multline*}
	\mathbb{I}(h_{\f}(\x)\neq y) = 0 \\ \leq L^{(\rho)}(\f'(\x),y) + \mathbb{I}[\Phi_{\rho}[\mu_{h_{f'}(\x)}(\f'(\x),h_{\f}(\x))] = 1].
	\end{multline*}
	
	To prove (\ref{equ-degegeneration2-bound-2}), we only need to prove that for any $(\x,y)\sim D$, the inequality
	\begin{multline*}
	\mathbb{I}[\Phi_{\rho}[\mu_{h_{\f'}(\x)}(\f'(\x),h_{\f}(\x))] = 1] \\ \leq L^{(\rho)}(\f(\x),y) + L^{(\rho)}(\f'(\x),y)
	\end{multline*}
	holds. If $h_{\f'}(\x)\neq y$ or $h_{\f}(\x) \neq y$, the right-hand side of the above inequality will reach the value of 1, which is obviously an upper bound of the left-hand side. Otherwise $h_{\f'}(\x) = h_{\f}(\x) = y$, and
	\begin{multline*}
	\mathbb{I}[\Phi_{\rho}[\mu_{h_{\f'}(\x)}(\f'(\x),h_{\f}(\x))] = 1] = 0  \\ \leq L^{(\rho)}(\f(\x),y) + L^{(\rho)}(\f'(\x),y).
	\end{multline*} \hfill{\QEDclosed}
\end{proof}%\rightline{\QEDclosed{}}
%%%%%%%%%%%%%%%%%%%%%%%%%%%%%%%%%%%%%%%%%%%%%%%%%%%
\vspace{-0.7cm}
\section{Proof of Theorem \textcolor{red}{2}}
We begin with the following lemmas to prove the Theorem \textcolor{red}{2}.

\begin{lemma}\label{lemma-Rad_Bound}
	(Two-sided Rademacher complexity bound, a modified version of Theorem 3.1, Mohri et al.\cite{mohri2012foundations}) Let $\mathcal{G}$ be a family of functions mapping from $\mathcal{Z}$ to $[0, 1]$. Let $D$ be any distribution over $\mathcal{Z}$, and  $\mathcal{S} = \{z_1,...,z_m\}$ be a sample drawn i.i.d. from $D$. Then, for any $\delta > 0$, with probability at least $1 - \delta$, the following holds for all $g \in \mathcal{G}$:
	\begin{equation}
	|\mathbb{E}[g(z)] - \frac{1}{m} \sum_{i=1}^{m} g(z_i)| \leq 2\widehat{\mathfrak{R}}_\mathcal{S}(\mathcal{G}) + 3\sqrt{\frac{\log\frac{4}{\delta}}{2m}}.
	\end{equation}
\end{lemma}

\begin{lemma}\label{lemma-Talagrand}
	(Talagrand's lemma, Lemma 4.2 of Mohri et al. \cite{mohri2012foundations}) Let $\Phi: \mathbb{R}\rightarrow \mathbb{R}$ be an l-Lipschitz. Then, for any hypothesis set $\mathcal{H}$ of real-valued functions, the following inequality holds:
	\begin{equation}
	\widehat{\mathfrak{R}}_{\widehat{D}}(\Phi \circ \mathcal{H}) \leq l \widehat{\mathfrak{R}}_{\widehat{D}}(\mathcal{H}).
	\end{equation}
\end{lemma}

\begin{lemma}\label{lemma-lossErr}
	Let $\mathcal{F}$ be the space of scoring functions mapping from $\mathcal{X}$ to $\mathbb{R}^K$. Let $D$ be a distribution over $\mathcal{X\times Y}$ and let $\widehat{D}$ be the corresponding empirical distribution for a sample $\mathcal{S} = \left\{(\x_1,y_1),...,(\x_m,y_m)\right\}$ drawn i.i.d. from $D$. Fix $\rho >0$. Then, for any $\delta > 0$, with probability at least $1 - \delta$, the following holds for all $\f\in \mathcal{F}$:
	\begin{equation}
	|\mathcal{E}^{(\rho)}_D(\f) - \mathcal{E}^{(\rho)}_{\widehat{D}}(\f)| \leq \frac{2K^2}{\rho}\widehat{\mathfrak{R}}_\mathcal{S}(\Pi_1\mathcal{F}) + 3K\sqrt{\frac{\log\frac{4}{\delta}}{2m}}
	\end{equation}
	where
	\begin{equation}
	\begin{aligned}
	\mathcal{E}^{(\rho)}_{D}(\f) &:= \mathbb{E}_{(\x,y)\sim D} \sum_{k=1}^K \Phi_{\rho}( \mu_k(\f(\x_i),y_i) )  \\
	& \ = \mathbb{E}_{(\x,y)\sim D}L^{(\rho)}(\f(\x_i),y_i).
	\end{aligned}
	\end{equation}
\end{lemma}
\begin{proof}
	Since the loss function $L^{(\rho)}$ is bounded by $K$, we scale the loss $L^{(\rho)}$ to $[0,1]$ by dividing by $K$, and denote the new class by $L^{(\rho)}/K$. By applying Lemma \ref{lemma-Rad_Bound} to $L^{(\rho)}/K$, for any $\delta > 0$, with probability at least $1 - \delta$, the following inequality holds for all $\f\in \mathcal{F}$,
	\begin{equation*}
	|\frac{\mathcal{E}^{(\rho)}_D(\f)}{K} - \frac{\mathcal{E}^{(\rho)}_{\widehat{D}}(\f)}{K}| \leq 2\widehat{\mathfrak{R}}_\mathcal{S}(L^{(\rho)}/K) + 3\sqrt{\frac{\log\frac{4}{\delta}}{2m}}.
	\end{equation*}
	Based on the property of the Rademacher complexity, we have $\widehat{\mathfrak{R}}_\mathcal{S}(L^{(\rho)}/K) = \frac{1}{K}\widehat{\mathfrak{R}}_\mathcal{S}(L^{(\rho)})$, and based on Lemma \ref{lemma-Talagrand} and the sub-additivity of the supremum, we have
	\begin{align*}
	%\begin{eqnarray*}
	%\begin{equation*}
	%\begin{aligned}
	& \widehat{\mathfrak{R}}_\mathcal{S}(L^{(\rho)}) \\
	& = \frac{1}{m}\mathbb{E}_{\sigma}[\sup\limits_{\f\in\mathcal{F}}\sum_{i=1}^{m}\sigma_i L^{(\rho)}(\f(\x_i),y_i)] \\
	& = \frac{1}{m}\mathbb{E}_{\sigma}[\sup\limits_{\f\in\mathcal{F}}\sum_{i=1}^{m}\sum_{y\in \mathcal{Y}}\sigma_i L^{(\rho)}(\f(\x_i),y)\mathbb{I}(y=y_i)] \\
	& \leq \frac{1}{m}\sum_{y\in \mathcal{Y}}\mathbb{E}_{\sigma}[\sup\limits_{\f\in\mathcal{F}}\sum_{i=1}^{m}\sigma_i L^{(\rho)}(\f(\x_i),y)\mathbb{I}(y=y_i)] \\
	& = \frac{1}{m}\sum_{y\in \mathcal{Y}}\mathbb{E}_{\sigma}[\sup\limits_{\f\in\mathcal{F}}\sum_{i=1}^{m}\sigma_i L^{(\rho)}(\f(\x_i),y)(\frac{2\mathbb{I}(y=y_i)-1}{2}+\frac{1}{2})] \\
	& \leq \frac{1}{2m}\sum_{y\in \mathcal{Y}}\mathbb{E}_{\sigma}[\sup\limits_{\f\in\mathcal{F}}\sum_{i=1}^{m}\sigma_i\epsilon_i L^{(\rho)}(\f(\x_i),y)] \ + \\
	& \qquad \qquad \frac{1}{2m}\sum_{y\in \mathcal{Y}}\mathbb{E}_{\sigma}[\sup\limits_{\f\in\mathcal{F}}\sum_{i=1}^{m}\sigma_i L^{(\rho)}(\f(\x_i),y)] \\
	& = \frac{1}{m}\sum_{y\in \mathcal{Y}}\mathbb{E}_{\sigma}[\sup\limits_{\f\in\mathcal{F}}\sum_{i=1}^{m}\sigma_i L^{(\rho)}(\f(\x_i),y)] \\
	%\end{aligned}
	%\end{equation*}
	%\end{eqnarray*}
	%	\end{align*}
	%	where $\epsilon_i = 2\mathbb{I}(y = y_i)-1\in \left\{-1, 1\right\}$ and we use the fact that $\epsilon_i\sigma_i$ has the same distribution as $\sigma_i$. Therefore, based on Lemma \ref{lemma-Talagrand} and the sub-additivity of sup, we have
	%	\begin{align*}
	%	% \begin{equation*}
	%	%     \begin{aligned}
	%	& \widehat{\mathfrak{R}}_\mathcal{S}(L^{(\rho)}) \\
	%	& \leq \frac{1}{m}\sum_{y\in \mathcal{Y}}\mathbb{E}_{\sigma}[\sup\limits_{\f\in\mathcal{F}}\sum_{i=1}^{m}\sigma_i L^{(\rho)}(\f(\x_i),y)] \\
	& = \frac{1}{m}\sum_{y\in \mathcal{Y}}\mathbb{E}_{\sigma}[\sup\limits_{\f\in\mathcal{F}}\sum_{i=1}^{m}\sigma_i [\sum_{k\neq y}\Phi_\rho(-f_k(\x_i)) + \Phi_\rho(f_y(\x_i))]] \\
	& \leq  \frac{1}{m}\sum_{y\in \mathcal{Y}}\sum_{k\neq y}\mathbb{E}_{\sigma}[\sup\limits_{\f\in\mathcal{F}}\sum_{i=1}^{m}\sigma_i \Phi_\rho(-f_k(\x_i))] \ + \\
	& \qquad \qquad \frac{1}{m}\sum_{y\in \mathcal{Y}}\mathbb{E}_{\sigma}[\sup\limits_{\f\in\mathcal{F}}\sum_{i=1}^{m}\sigma_i \Phi_\rho(f_y(\x_i))] \\
	& \leq  \frac{1}{m}\sum_{y\in \mathcal{Y}}\sum_{k\neq y}\mathbb{E}_{\sigma}[\sup\limits_{f\in\Pi_1(\mathcal{F})}\sum_{i=1}^{m}\sigma_i \Phi_\rho(-f(\x_i))] \ + \\
	& \qquad \qquad \frac{1}{m}\sum_{y\in \mathcal{Y}}\mathbb{E}_{\sigma}[\sup\limits_{f\in\Pi_1(\mathcal{F})}\sum_{i=1}^{m}\sigma_i \Phi_\rho(f(\x_i))] \\
	& \leq  \frac{1}{m\rho}\sum_{y\in \mathcal{Y}}\sum_{k\neq y}\mathbb{E}_{\sigma}[\sup\limits_{f\in\Pi_1(\mathcal{F})}\sum_{i=1}^{m}\sigma_i [-f(\x_i)]] \ + \\
	& \qquad \qquad \frac{1}{m\rho}\sum_{y\in \mathcal{Y}}\mathbb{E}_{\sigma}[\sup\limits_{f\in\Pi_1(\mathcal{F})}\sum_{i=1}^{m}\sigma_i f(\x_i)] \\
	& = \frac{1}{m\rho}\sum_{y\in \mathcal{Y}}\sum_{k\in \mathcal{Y}}\mathbb{E}_{\sigma}[\sup\limits_{f\in\Pi_1(\mathcal{F})}\sum_{i=1}^{m}\sigma_i f(\x_i)] \\
	& = \frac{K^2}{\rho}\widehat{\mathfrak{R}}_\mathcal{S}(\Pi_1(\mathcal{F})),
	%     \end{aligned}
	% \end{equation*}
	\end{align*}
	where $\epsilon_i = 2\mathbb{I}(y = y_i)-1\in \left\{-1, 1\right\}$ and we use the fact that $\epsilon_i\sigma_i$ has the same distribution as $\sigma_i$.
	The proof is finished by combing the above inequalities. \hfill{\QEDclosed}
\end{proof}%\rightline{\QEDclosed{}}

\begin{lemma}\label{lemma-dispErr}
	Let $\mathcal{F}$ be the space of scoring functions mapping from $\mathcal{X}$ to $\mathbb{R}^K$. Let $D$ be a distribution over $\mathcal{X}$ and let $\widehat{D}$ be the corresponding empirical distribution for a sample $\mathcal{S} = \left\{\x_1,...,\x_m\right\}$ drawn i.i.d. from $D$. Fix $\rho > 0$. Then, for any $\delta > 0$, with probability at least $1 - \delta$, the following holds for all $\f, \f'\in \mathcal{F}$:
	\begin{equation}
	\begin{aligned}
	& |\text{\rm MCSD}^{(\rho)}_D(\f,\f') - \text{\rm MCSD}^{(\rho)}_{\widehat{D}}(\f,\f')| \\
	& \qquad \leq \frac{4K}{\rho}\widehat{\mathfrak{R}}_\mathcal{S}(\Pi_1(\mathcal{F})) + 3K\sqrt{\frac{\log\frac{4}{\delta}}{2m}}
	\end{aligned}
	\end{equation}
\end{lemma}
\begin{proof}
	Denote the hypothesis set $\mathcal{M} := \{\x\rightarrow \|\M^{(\rho)}(\f(\x)) - \M^{(\rho)}(\f'(\x))\|_1/{K^2}|\f,\f'\in \mathcal{F}\}$ as a new class. Then the class $\mathcal{M}$ is a family of functions mapping from $\mathcal{X}$ to $[0,1]$. By applying the Lemma \ref{lemma-Rad_Bound} to $\mathcal{M}$,  for any $\delta>0$, with probability at least $1-\delta$ we have
	\begin{equation*}
	|\frac{\text{\rm MCSD}^{(\rho)}_D(\f,\f')}{K} - \frac{\text{\rm MCSD}^{(\rho)}_{\widehat{D}}(\f,\f')}{K}| \leq 2\widehat{\mathfrak{R}}_\mathcal{S}(\mathcal{M}) + 3\sqrt{\frac{\log\frac{4}{\delta}}{2m}}
	\end{equation*}
	and based on the Lemma \ref{lemma-Talagrand} and the sup-additivity of the supremum, we have
	\begin{align*}
	% \begin{equation*}
	%     \begin{aligned}
	& \widehat{\mathfrak{R}}_\mathcal{S}(\mathcal{M}) \\
	& = \frac{1}{K^2m}\mathbb{E}_{\sigma}[\sup\limits_{\f,\f'\in \mathcal{F}}\sum_{i=1}^{m}\sigma_i\|\M^{(\rho)}(\f(\x_i)) - \M^{(\rho)}(\f'(\x))\|_1] \\
	& \leq \frac{1}{K^2m}\sum_{k,k'\in \mathcal{Y}} \mathbb{E}_{\sigma}[\sup\limits_{\f,\f'\in \mathcal{F}}\sum_{i=1}^{m}\sigma_i |\Phi_\rho(\mu_k (\f(\x_i), k')) \\
	& \qquad \qquad \qquad \qquad \qquad \qquad - \Phi_\rho(\mu_k (\f'(\x_i), k'))|] \\
	& \leq \frac{1}{K^2m}\sum_{k,k'\in \mathcal{Y}} \mathbb{E}_{\sigma}[\sup\limits_{\f,\f'\in \mathcal{F}}\sum_{i=1}^{m}\sigma_i [\Phi_\rho(\mu_k (\f(\x_i), k')) \\
	& \qquad \qquad \qquad \qquad \qquad \qquad - \Phi_\rho(\mu_k (\f'(\x_i), k'))]] \\
	& \leq \frac{2}{K^2m}\sum_{k,k'\in \mathcal{Y}} \mathbb{E}_{\sigma}[\sup\limits_{\f\in \mathcal{F}}\sum_{i=1}^{m}\sigma_i \Phi_\rho(\mu_k (\f(\x_i), k'))] \\
	& \leq \frac{2}{K^2m\rho}\sum_{k,k'\in \mathcal{Y}} \mathbb{E}_{\sigma}[\sup\limits_{\f\in \mathcal{F}}\sum_{i=1}^{m}\sigma_i \mu_k (\f(\x_i), k')] \\
	& = \frac{2}{K^2m\rho}\sum_{k,k'\in \mathcal{Y}} \mathbb{E}_{\sigma}[\sup\limits_{\f\in \mathcal{F}}\sum_{i=1}^{m}\sigma_i f_k(\x_i)] \\
	& \leq \frac{2}{K^2m\rho}\sum_{k,k'\in \mathcal{Y}} \mathbb{E}_{\sigma}[\sup\limits_{f\in \Pi_1(\mathcal{F})}\sum_{i=1}^{m}\sigma_i f(\x_i)] \\
	& = \frac{2}{\rho} \widehat{\mathfrak{R}}_\mathcal{S}(\Pi_1(\mathcal{F })).
	%     \end{aligned}
	% \end{equation*}
	\end{align*}
	The proof is finished by combing the above inequalities. \hfill{\QEDclosed}
\end{proof}%\rightline{\QEDclosed{}}

\begin{lemma}\label{lemma-divergenceErr}
	Let $\mathcal{F}$ be the space of scoring functions mapping from $\mathcal{X}$ to $\mathbb{R}^K$. Let $P_x$ and $Q_x$ be source and target marginal distributions over $\mathcal{X}$, and let $\widehat{P}_x$ and $\widehat{Q}_x$ be the corresponding empirical distributions for a sample of $\mathcal{S} = \{\x_i^s\}_{i=1}^{n_s}$ and a sample of $\mathcal{T} = \{\x_i^t\}_{i=1}^{n_t}$ respectively. Fix $\rho>0$. Then, for any $\delta > 0$, with probability at least $1 - 2\delta$, the following holds:
	\begin{equation}
	\begin{aligned}
	d^{(\rho)}_{MCSD}(P_x, Q_x) & \leq d^{(\rho)}_{MCSD}(\widehat{P}_x, \widehat{Q}_x) + \frac{4K}{\rho}\widehat{\mathfrak{R}}_\mathcal{S}(\Pi_1(\mathcal{F}))  \\
	& \qquad + \frac{4K}{\rho}\widehat{\mathfrak{R}}_\mathcal{T}(\Pi_1(\mathcal{F}))+ 3K\sqrt{\frac{\log\frac{4}{\delta}}{2n_s}} \\
	& \qquad \qquad + 3K\sqrt{\frac{\log\frac{4}{\delta}}{2n_t}}
	\end{aligned}
	\end{equation}
\end{lemma}
\begin{proof}
	Based on the Lemma \ref{lemma-dispErr} and the sub-additivity of the supremum, by using the union bound, for any $\delta > 0$, with probability at least $1 - 2\delta$, we have
	\begin{align*}
	% \begin{equation*}
	%     \begin{aligned}
	& d^{(\rho)}_{MCSD}(P_x, Q_x) \\
	& = \sup_{\f, \f'\in \mathcal{F}} [\text{\rm MCSD}^{(\rho)}_{Q_x}(\f, \f') - \text{\rm MCSD}^{(\rho)}_{P_x}(\f,\f')] \\
	& = \sup_{\f, \f'\in \mathcal{F}} [\text{\rm MCSD}^{(\rho)}_{Q_x}(\f, \f') - \text{\rm MCSD}^{(\rho)}_{\widehat{Q}_x}(\f, \f') \\
	& \qquad + \text{\rm MCSD}^{(\rho)}_{\widehat{Q}_x}(\f, \f') - \text{\rm MCSD}^{(\rho)}_{\widehat{P}_x}(\f,\f') \\
	& \qquad + \text{\rm MCSD}^{(\rho)}_{\widehat{P}_x}(\f, \f') - \text{\rm MCSD}^{(\rho)}_{P_x}(\f,\f')] \\
	& \leq \sup_{\f, \f'\in \mathcal{F}} [\text{\rm MCSD}^{(\rho)}_{Q_x}(\f, \f') - \text{\rm MCSD}^{(\rho)}_{\widehat{Q}_x}(\f, \f')] \\
	& \qquad  + \sup_{\f, \f'\in \mathcal{F}} [\text{\rm MCSD}^{(\rho)}_{\widehat{Q}_x}(\f, \f') - \text{\rm MCSD}^{(\rho)}_{\widehat{P}_x}(\f,\f')] \\
	& \qquad  + \sup_{\f, \f'\in \mathcal{F}} [\text{\rm MCSD}^{(\rho)}_{\widehat{P}_x}(\f, \f') - \text{\rm MCSD}^{(\rho)}_{P_x}(\f,\f')] \\
	& \leq \sup_{\f, \f'\in \mathcal{F}} [\text{\rm MCSD}^{(\rho)}_{\widehat{Q}_x}(\f, \f') - \text{\rm MCSD}^{(\rho)}_{\widehat{P}_x}(\f,\f')] \\
	& \qquad + \sup_{\f, \f'\in \mathcal{F}} |\text{\rm MCSD}^{(\rho)}_{Q_x}(\f, \f') - \text{\rm MCSD}^{(\rho)}_{\widehat{Q}_x}(\f, \f')| \\
	& \qquad  + \sup_{\f, \f'\in \mathcal{F}} |\text{\rm MCSD}^{(\rho)}_{\widehat{P}_x}(\f, \f') - \text{\rm MCSD}^{(\rho)}_{P_x}(\f,\f')| \\
	& \leq \sup_{\f, \f'\in \mathcal{F}} [\text{\rm MCSD}^{(\rho)}_{\widehat{Q}_x}(\f, \f') - \text{\rm MCSD}^{(\rho)}_{\widehat{P}_x}(\f,\f')] \\
	& \qquad + \frac{4K}{\rho}\widehat{\mathfrak{R}}_\mathcal{T}(\Pi_1(\mathcal{F})) + 3K\sqrt{\frac{\log\frac{4}{\delta}}{2n_t}} \\
	& \qquad + \frac{4K}\rho{}\widehat{\mathfrak{R}}_\mathcal{S}(\Pi_1(\mathcal{F})) + 3K\sqrt{\frac{\log\frac{4}{\delta}}{2n_s}}  \\
	&= d^{(\rho)}_{MCSD}(\widehat{P}_x, \widehat{Q}_x) + \frac{4K}{\rho}\widehat{\mathfrak{R}}_\mathcal{S}(\Pi_1(\mathcal{F})) \\
	& \qquad + \frac{4K}{\rho}\widehat{\mathfrak{R}}_\mathcal{T}(\Pi_1(\mathcal{F})) + 3K\sqrt{\frac{\log\frac{4}{\delta}}{2n_s}} + 3K\sqrt{\frac{\log\frac{4}{\delta}}{2n_t}}.
	%     \end{aligned}
	% \end{equation*}
	\end{align*} \hfill{\QEDclosed}
\end{proof}%\rightline{\QEDclosed}

%\begin{theo}[Theorem 2]
%	Let $\mathcal{F}$ be hypothesis set mapping from $\mathcal{X}$ to $\mathbb{R}^K$. Let $P$ and $Q$ are source and target distributions over $\mathcal{X\times Y}$ and $P_x$ and $Q_x$ be their corresponding marginal distributions over $\mathcal{X}$. Let $\widehat{P}$ and $\widehat{Q}_x$ denote the corresponding empirical distributions for a sample of $\mathcal{S} = \{(\x_i^s,y_i^s)\}_{i=1}^{n_s}$ and a sample of $\mathcal{T} = \{\x_j^t\}_{j=1}^{n_t}$ respectively. Fix $\rho>0$. Then, for any $\delta > 0$, with probability at least $1 - 3\delta$, the following holds for all $\f\in \mathcal{F}$:
%	\begin{equation}\label{equ-theorem-appendix}
%	\begin{aligned}
%	\mathcal{E}_Q(h_{\f}) & \leq \mathcal{E}^{(\rho)}_{\widehat{P}}(\f) + d^{(\rho)}_{MCSD}(\widehat{P}_x, \widehat{Q}_x) \\
% & + (\frac{2K^2}{\rho} + \frac{4K}{\rho})\widehat{\mathfrak{R}}_\mathcal{S}(\Pi_1\mathcal{F})
%	+ \frac{4K}{\rho}\widehat{\mathfrak{R}}_\mathcal{T}(\Pi_1(\mathcal{F})) \\
%& + 6K\sqrt{\frac{\log\frac{4}{\delta}}{2n_s}} + 3K\sqrt{\frac{\log\frac{4}{\delta}}{2n_t}} + \lambda
%	\end{aligned}
%	\end{equation}
%	where $\lambda = \mathcal{E}^{(\rho)}_P(\f^*) + \mathcal{E}^{(\rho)}_Q(\f^*)$ is the minimum value with $\f^* \in \mathcal{F}$.
%\end{theo}
%\begin{proof}
%	By applying Lemma \ref{lemma-lossErr}, Lemma \ref{lemma-divergenceErr} and Lemma \ref{theorem-1} as well as the union bound, we can get the result directly.
%\end{proof}\rightline{\QEDclosed{}}

\begin{theo}[Theorem \textcolor{red}{2}]
	Let $\mathcal{F}$ be the space of scoring functions mapping from $\mathcal{X}$ to $\mathbb{R}^K$. Let $P$ and $Q$ be the source and target distributions over $\mathcal{X\times Y}$, and let $P_x$ and $Q_x$ be the corresponding marginal distributions over $\mathcal{X}$. Let $\widehat{P}$ and $\widehat{Q}_x$ be the corresponding empirical distributions for a sample $\mathcal{S} = \{(\x_i^s,y_i^s)\}_{i=1}^{n_s}$ and a sample $\mathcal{T} = \{\x_j^t\}_{j=1}^{n_t}$. Fix $\rho>0$. Then, for any $\delta > 0$, with probability at least $1 - 3\delta$, the following holds for all $\f\in \mathcal{F}$
	\begin{equation}
	\begin{aligned}
	\mathcal{E}_Q(h_{\f}) \leq & \mathcal{E}^{(\rho)}_{\widehat{P}}(\f) +  d^{(\rho)}_{MCSD}(\widehat{P}_x, \widehat{Q}_x) \\
	& + (\frac{2K^2}{\rho} + \frac{4K}{\rho})\widehat{\mathfrak{R}}_\mathcal{S}(\Pi_1\mathcal{F})
	+ \frac{4K}{\rho}\widehat{\mathfrak{R}}_\mathcal{T}(\Pi_1(\mathcal{F})) \\
	& + 6K\sqrt{\frac{\log\frac{4}{\delta}}{2n_s}} + 3K\sqrt{\frac{\log\frac{4}{\delta}}{2n_t}} + \lambda,
	\end{aligned}
	\end{equation}
	where the constant $\lambda = \min\limits_{\f\in \mathcal{F}}\mathcal{E}^{(\rho)}_P(\f) + \mathcal{E}^{(\rho)}_Q(\f)$, and
	\begin{equation}
	\mathcal{E}^{(\rho)}_{\widehat{P}}(\f) := \frac{1}{n_s}\sum_{i=1}^{n_s} \sum_{k=1}^K \Phi_{\rho}( \mu_k(\f(\x^s_i),y^s_i) )  .
	\end{equation}
\end{theo}
\begin{proof}
	The bound is achieved by applying Lemma \ref{lemma-lossErr}, Lemma \ref{lemma-divergenceErr}, Lemma \ref{theorem-1}, and the union bound. \hfill{\QEDclosed}
\end{proof}%\rightline{\QEDclosed{}}
%%%%%%%%%%%%%%%%%%%%%%%%%%%%%%%%%%%%%%%%%%%%%%%%%%%%%%%%

\section{Connecting Theory with Algorithms}\label{appendix-theory-connect-algorithm}

%\begin{prop}\label{PropMCDRelateMCSD}
%{\color{blue} Given the ramp loss $\Phi_\rho$ defined as (\ref{equ-ramp_loss}), there exists a distance measure $\varphi: \mathbb{R}\times \mathbb{R} \rightarrow \mathbb{R}_+$ defined as
%\begin{equation}
%\varphi(a,b) =  (K-1)|\Phi_\rho(-a)-\Phi_\rho(-b)|+|\Phi_\rho(a)-\Phi_\rho(b)|, \nonumber	
%\end{equation}
%such that the matrix-formed $\| \M^{(\rho)}(\f'(\x)) - \M^{(\rho)}(\f''(\x)) \|_1$ in MCSD (\ref{EqnMCSD}) can be calculated as the sum of $\varphi$-distance values of $K$ entry pairs between $f'_k(x)$ and $f''_k(x)$, i.e.,
%\begin{equation}
%\| \M^{(\rho)}(\f'(\x)) - \M^{(\rho)}(\f''(\x)) \|_1 = \sum_{k=1}^K \varphi(f'_k(\x), f''_k(\x)). \nonumber
%\end{equation}  }
%\end{prop}

\noindent\textbf{Connections between KL (\textcolor{red}{25}) and CE (\textcolor{red}{26})}

\begin{align} \label{appendix-Connection-CE-KL}
\notag \mathbb{E}_{\x\sim D} \frac{1}{2}[&\text{\rm CE}(\phi(\f'(\psi(\x))), \phi(\f''(\psi(\x)))) \\
\notag + &\text{\rm CE}(\phi(\f''(\psi(\x))), \phi(\f'(\psi(\x))))] \\
\notag = \mathbb{E}_{\x\sim D} \frac{1}{2}[& - \sum_{k=1}^{K} \phi_k(\f'(\psi(\x))) \mathrm{log}(\phi_k(\f''(\psi(\x))))  \\
\notag& -  \sum_{k=1}^{K} \phi_k(\f''(\psi(\x))) \mathrm{log}(\phi_k(\f'(\psi(\x))))] \\
\notag = \mathbb{E}_{\x\sim D} \frac{1}{2}[& - \sum_{k=1}^{K} \phi_k(\f'(\psi(\x))) \mathrm{log}(\frac{\phi_k(\f''(\psi(\x)))}{\phi_k(\f'(\psi(\x)))} )  \\
\notag & - \sum_{k=1}^{K} \phi_k(\f''(\psi(\x))) \mathrm{log}(\frac{\phi_k(\f'(\psi(\x)))}{\phi_k(\f''(\psi(\x)))})] \\
\notag + \mathbb{E}_{\x\sim D} \frac{1}{2}[& - \sum_{k=1}^{K} \phi_k(\f'(\psi(\x))) \mathrm{log}(\phi_k(\f'(\psi(\x))) )  \\
\notag & - \sum_{k=1}^{K} \phi_k(\f''(\psi(\x))) \mathrm{log}(\phi_k(\f''(\psi(\x))))] \\
\notag = \mathbb{E}_{\x\sim D} \frac{1}{2}[&\text{\rm KL}(\phi(\f'(\psi(\x))), \phi(\f''(\psi(\x)))) \\
\notag + &\text{\rm KL}(\phi(\f''(\psi(\x))), \phi(\f'(\psi(\x))))] + \\
\mathbb{E}_{\x\sim D} \frac{1}{2}[&\text{\rm H}(\phi(\f'(\psi(\x))))
+ \text{\rm H}(\phi(\f''(\psi(\x))))].
\end{align}
It is obvious that the objective of CE (\textcolor{red}{26}) equals to the combination of the objective of KL (\textcolor{red}{25}) and terms related to the entropy of class probabilities of $\f'$ and $\f''$.

\begin{proposition}[Proposition \textcolor{red}{3}]
	\label{appendix-PropMCDRelateMCSD}
	Given the ramp loss $\Phi_\rho$ defined as (\textcolor{red}{5}), there exists a distance measure $\varphi: \mathbb{R}\times \mathbb{R} \rightarrow \mathbb{R}_+$ defined as
	\begin{equation}
	\varphi(a,b) =  (K-1)|\Phi_\rho(-a)-\Phi_\rho(-b)|+|\Phi_\rho(a)-\Phi_\rho(b)|, \nonumber	
	\end{equation}
	such that the matrix-formed $\| \M^{(\rho)}(\f'(\x)) - \M^{(\rho)}(\f''(\x)) \|_1$ in MCSD (\textcolor{red}{7}) can be calculated as the sum of $\varphi$-distance values of $K$ entry pairs between $f'_k(\x)$ and $f''_k(\x)$, i.e.,
	\begin{equation}
	\| \M^{(\rho)}(\f'(\x)) - \M^{(\rho)}(\f''(\x)) \|_1 = \sum_{k=1}^K \varphi(f'_k(\x), f''_k(\x)). \nonumber
	\end{equation}
	%There exists a distance measure $\varphi: \mathbb{R}\times \mathbb{R} \rightarrow \mathbb{R}_+$, such that the value of $\| \M^{(\rho)}(\f'(\x)) - \M^{(\rho)}(\f''(\x)) \|_1$ can be calculated as the sum of $\varphi$-distance values of $K$ entry pairs between $f'_k(x)$ and $f''_k(x)$, i.e.,
	%	\begin{equation}\label{equ-distance}
	%	\| \M^{(\rho)}(\f'(\x)) - \M^{(\rho)}(\f''(\x)) \|_1 = \sum_{k=1}^K \varphi(f'_k(\x), f''_k(\x)).
	%	\end{equation}
	\begin{proof}
		It is obvious that the function $\varphi$ satisfies the properties of symmetry, non-negative, and triangle inequality, and thus it is a distance measure. The proposition follows directly from the definitions of $\varphi$ and $\M$ (cf. the Equation (\textcolor{red}{8})). \hfill{\QEDclosed}
	\end{proof}%\rightline{\QEDclosed{}}
\end{proposition}

To intuitively understand the $\varphi$-distance defined above, we plot in Figure \ref{Fig:diff_appendix} the values of $|a-b|$ and $\varphi(a,b)$ as the functions of $a$ and $b$. We can see that the $\varphi$-distance $\varphi(a,b)$ can be considered as a variant form of the absolute distance $|a-b|$. From the Figure \ref{Fig:diff_appendix}, we can also see that the maximization (minimization) of $\varphi(a,b)$ can be achieved by maximizing (minimizing) the difference between $a$ and $b$.
\renewcommand\thefigure{\thesection}
\begin{figure}[h]
	\centering
	\subfigure[$|a-b|$]{
		\centering
		\includegraphics[width=0.48\linewidth] {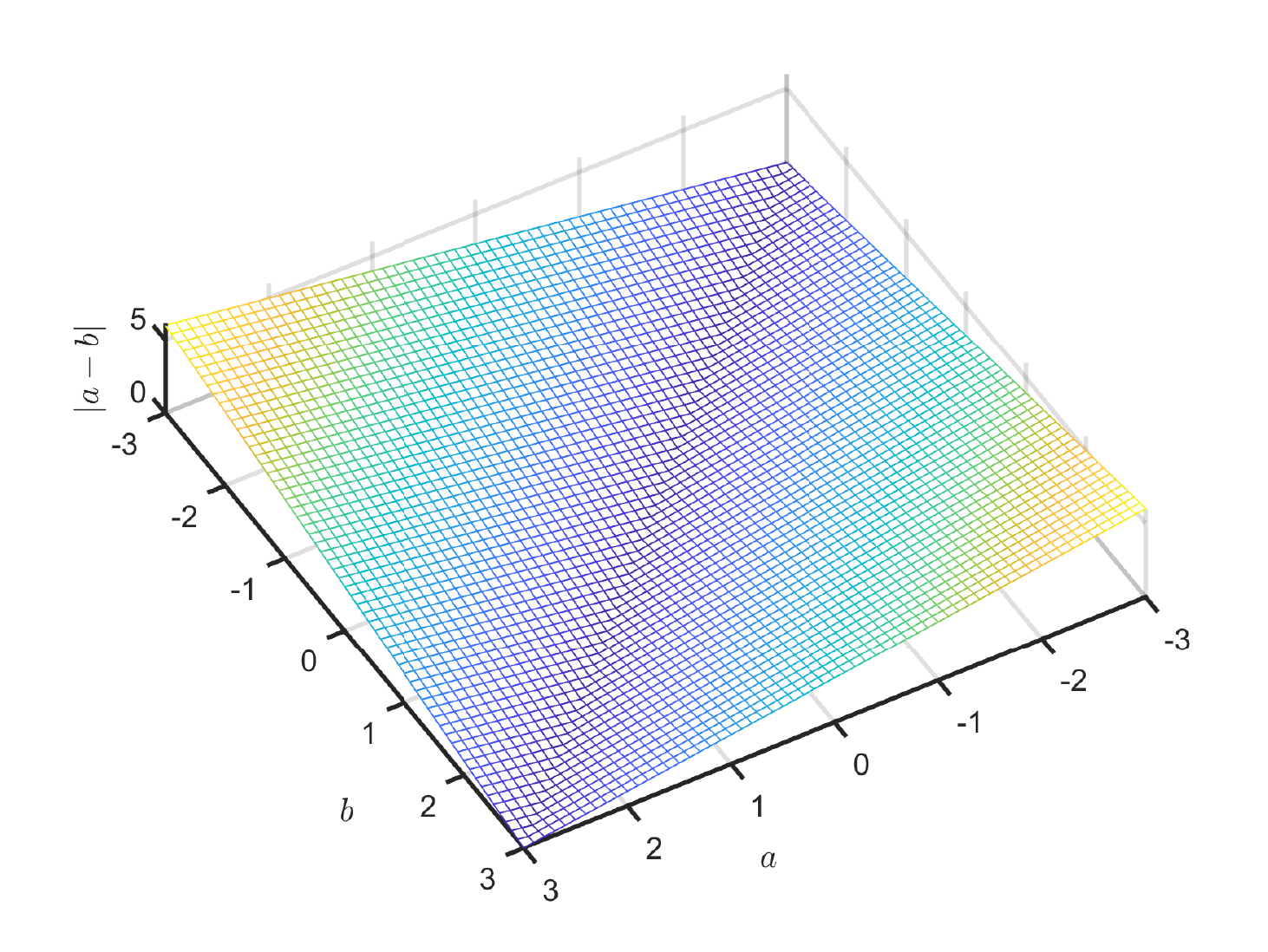}}
	\subfigure[$\varphi(a,b)$]{
		\includegraphics[width=0.48\linewidth] {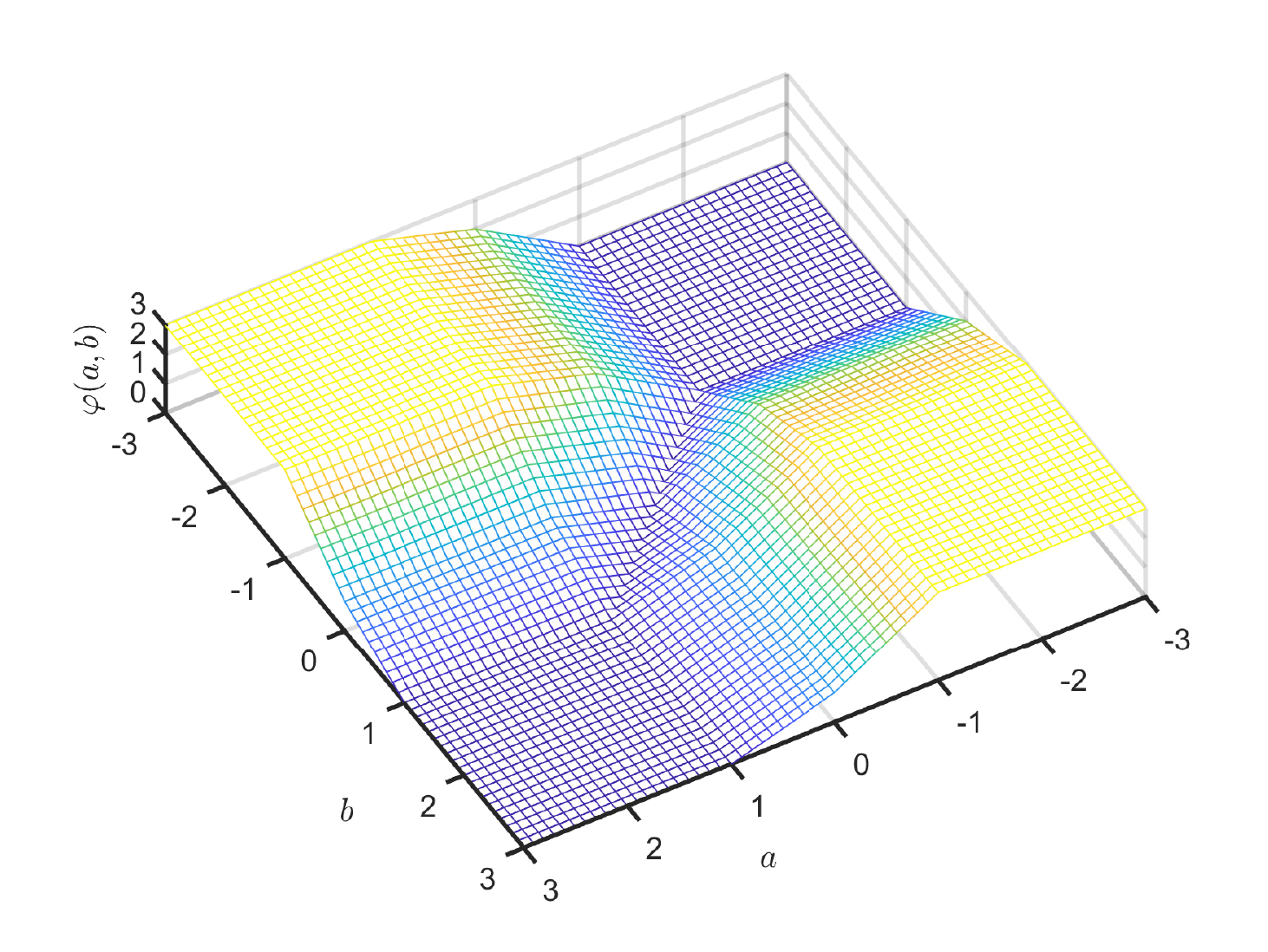}}
	\label{Fig:diff_appendix}
	\caption{Plotting of the values of (a) $|a-b|$ and (b) $\varphi(a,b)$ with the arguments $a$ and $b$. Here we set $K=3$ and $\rho = 1$. The results with other values of $K$ and $\rho$ are similar.}
\end{figure}

\noindent\textbf{Explanation of $\text{\rm ConFUSE}_{\widehat{Q}_x^{\psi}}^{st}$ (\textcolor{red}{36})}

Based on the equations of (\ref{appendix-Connection-CE-KL}), we have
\begin{align} \label{appendix-ExplationOfTargetConfusionCE}
\notag & \text{\rm ConFUSE}_{\widehat{Q}_x^{\psi}}^{st} (\f^{s},\f^t)  %\\ \notag =  & - \frac{1}{2n_t} \sum_{j=1}^{n_t} %\sum_{k=1}^{K} p^{st}_{k+K}(\vect{x}_j^t) \mathrm{log}(p^{st}_k(\vect{x}_j^t)) \\ & - \frac{1}{2n_t} \sum_{j=1}^{n_t} \sum_{k=1}^{K} p^{st}_k(\vect{x}_j^t) \mathrm{log}( p^{st}_{k+K}(\vect{x}_j^t))
%\\ \notag =  & - \frac{1}{2n_t} \sum_{j=1}^{n_t} \sum_{k=1}^{K} p^{st}_{k+K}(\vect{x}_j^t) \mathrm{log}(\frac{p^{st}_k(\vect{x}_j^t)}{p^{st}_{k+K}(\vect{x}_j^t)}) \\ \notag & - \frac{1}{2n_t} \sum_{j=1}^{n_t} \sum_{k=1}^{K} p^{st}_k(\vect{x}_j^t) \mathrm{log}(\frac{p^{st}_{k+K}(\vect{x}_j^t)}{p^{st}_k(\vect{x}_j^t)})
%\\  \notag & - \frac{1}{2n_t} \sum_{j=1}^{n_t} \sum_{k=1}^{K} p^{st}_{k+K}(\vect{x}_j^t) \mathrm{log}({p^{st}_{k+K}(\vect{x}_j^t)}) \\ & - \frac{1}{2n_t} \sum_{j=1}^{n_t} \sum_{k=1}^{K} p^{st}_k(\vect{x}_j^t) \mathrm{log}({p^{st}_k(\vect{x}_j^t)})
\\ \notag =  & \frac{1}{2n_t} \sum_{j=1}^{n_t} [\text{\rm KL}( \vect{p}^{st}_{1:K}(\vect{x}_j^t), \vect{p}^{st}_{K:2K}(\vect{x}_j^t)) \\
\notag & \quad \qquad + \text{\rm KL}( \vect{p}^{st}_{K:2K}(\vect{x}_j^t), \vect{p}^{st}_{1:K}(\vect{x}_j^t))]
\\     + & \frac{1}{2n_t} \sum_{j=1}^{n_t} [ \widetilde{\text{\rm H}}(\vect{p}^{st}_{1:K}(\vect{x}_j^t)) + \widetilde{\text{\rm H}}(\vect{p}^{st}_{K:2K}(\vect{x}_j^t))],
\end{align}
where the $\vect{p}^{st}_{1:K}(\vect{x}^t)$ and $\vect{p}^{st}_{K:2K}(\vect{x}^t)$ are vectors composed with the first $K$ and last $K$ values of $\vect{p}^{st}(\vect{x}^t)$, respectively.
The $\text{\rm{KL}}$-divergence terms encourage the agreement of $\vect{p}^{st}_{1:K}(\vect{x}^t)$ and $\vect{p}^{st}_{K:2K}(\vect{x}^t)$ for any target instance $\vect{x}^t$. The $\widetilde{\rm{H}}$ terms share the same formulation with the entropy function. Although $\sum_{k=1}^{K} p^{st}_{k}(\vect{x}^t) \leq 1$ and $\sum_{k=K}^{2K} p^{st}_{k}(\vect{x}^t) \leq 1$, minimizing the $\widetilde{\text{\rm{H}}}$ terms encourages both the $\vect{p}^{st}_{1:K}(\vect{x}_j^t)$ and $\vect{p}^{st}_{K:2K}(\vect{x}_j^t)$ to be vectors with only one non-zero value, whose proof is almost the same as that of the entropy loss. In consideration of the terms of $\text{\rm{KL}}$-divergence and $\widetilde{\text{\rm{H}}}$, as well as the intrinsic sum-to-one-constraint, i.e., $\sum_{k=1}^{2K} p^{st}_k (\vect{x}^t) = 1$, minimizing $\text{\rm ConFUSE}_{\widehat{Q}_x^{\psi}}^{st} (\f^{s},\f^t)$ leads to $ \vect{p}^{st}_{1:K}(\vect{x}^t)$ and $\vect{p}^{st}_{K:2K}(\vect{x}^t)$ as the same vector with only one non-zero value of 0.5 for any target instance $\vect{x}^t$.

\begin{proposition}[Proposition \textcolor{red}{4}]
	\label{appendix-PropSymmNetConnectMcDalNet-SrcErr}
	Let $\mathcal{F}$ be a rich enough space of continuous and bounded scoring functions, with the sum-to-zero constraint $\sum_{k=1}^K f_k =0$. For $\f^s, \f^t \in \mathcal{F}$ and a fixed function $\psi$ that satisfies $\psi(\x_1)\neq \psi(\x_2)$ when $y_1\neq y_2$, $\exists \ \rho > 0$ such that, the minimizer $\f^{s*}$ of ${\cal{L}}_{\widehat{P}^{\psi}}^s (\f^s)$ in (\textcolor{red}{38}) also minimizes the term $\mathcal{E}^{(\rho)}_{\widehat{P}^{\psi}}(\f^s)$ in (\textcolor{red}{22}) of empirical source error defined on $\f^s$, and the minimizer $\f^{t*}$ of ${\cal{L}}_{\widehat{P}^{\psi}}^t (\f^t)$ in (\textcolor{red}{38}) also minimizes the term $\mathcal{E}^{(\rho)}_{\widehat{P}^{\psi}}(\f^t)$ in (\textcolor{red}{22}) of empirical source error defined on $\f^t$.
	\begin{proof}
		We first restate the definition of ${\cal{L}}_{\widehat{P}^{\psi}}^s (\f^s)$ and $\mathcal{E}^{(\rho)}_{\widehat{P}^{\psi}}(\f^s)$ as
		\begin{equation}
		\begin{gathered}
		\mathcal{E}^{(\rho)}_{\widehat{P}^{\psi}}(\f^s) = \mathbb{E}_{(\x,y) \sim \widehat{P}} \sum_{k=1}^{K} \Phi_\rho(\mu_k(\f^s(\psi(\x)),y)), \\
		{\cal{L}}_{\widehat{P}^{\psi}}^s (\f^s) = \mathbb{E}_{(\x,y) \sim \widehat{P}} -\log(\phi_y(\f^s(\psi(\x)))),
		\end{gathered}
		\end{equation}
		where $\phi$ is the softmax operator.
		Under the assumption that $\psi(\x_1)\neq \psi(\x_2)$ for each example with $y_1\neq y_2$, if the scoring function space is rich enough, then the minimizer $\f^{s*}$ of ${\cal{L}}_{\widehat{P}^{\psi}}^s (\f^s)$ makes $\phi_y(\f^{s*}(\psi(\x)))$ reach the maximum value for each example $(\x,y)\sim \widehat{P}$. Since the scoring function is bounded, we assume that $\|\f^{s*}(\psi(\x))\|_\infty \leq M$, and
		\begin{equation*}
		\phi_y(\f^{s*}(\psi(\x))) = \frac{\exp(f^{s*}_y(\psi(\x)))}{\exp(f^{s*}_y(\psi(\x)))+\sum_{k\neq y}\exp(f^{s*}_k(\psi(\x)))}.
		\end{equation*}
		With the sum-to-zero constraint $\sum_{k=1}^K f^{s*}_k(\psi(\x)) =0$, it is not hard to verify that $f^{s*}_y(\psi(\x)) = M$ and $f^{s*}_k(\psi(\x)) = -M/(K-1), k\neq y$. Therefore, for any $\rho \leq M/(K-1)$, we have $\sum_{k=1}^{K} \Phi_\rho(\mu_k(\f^{s*}(\psi(\x)),y)) = 0$ and thus $\mathcal{E}^{(\rho)}_{\widehat{P}^{\psi}}(\f^{s*}) = 0$. Similarly, the minimizer $\f^{t*}$ of ${\cal{L}}_{\widehat{P}^{\psi}}^t (\f^t)$ results in $\mathcal{E}^{(\rho)}_{\widehat{P}^{\psi}}(\f^{t*}) = 0$. \hfill{\QEDclosed}
		%Obviously, minimizer $\f^{s*}$ of ${\cal{L}}_{\widehat{P}^{\psi}}^s (\f^s)$ will make $\phi_y(\f^{s*}(\psi(\x)))$ approaching to $1$ for each $(\x,y)\sim \widehat{P}$, which makes $f_y^{s*}(\psi(\x)) \gg f_k^{s*}(\psi(\x))$ for those $k\neq y$. Since $\f^{s*}(\psi(\x))$ is bounded and constrained by the sum-to-zero condition, we assume that $\|\f^{s*}(\psi(\x))\|_\infty \leq M$, and therefore $f_y^{s*}(\psi(\x))$ will approach to $M$ as much as possible while each $f_k^{s*}(\psi(\x))$ of $k\neq y$ will approach to $-M/K$. This directly leads to the zero value of $\mathcal{E}^{(\rho)}_{\widehat{P}^{\psi}}(\f^{s*})$ if we choose a $\rho \ll M/K$. Similarly, minimizer $\f^{t*}$ of ${\cal{L}}_{\widehat{P}^{\psi}}^t (\f^t)$ also minimizes a term $\mathcal{E}^{(\rho)}_{\widehat{P}^{\psi}}(\f^t)$.
		
	\end{proof}%\rightline{\QEDclosed{}}
\end{proposition}

\begin{proposition}[Proposition \textcolor{red}{5}]
	\label{appendix-PropSymmNetConnectMcDalNet-FeaAligh}
	For $\psi$ of a function space of enough capacity and fixed functions $\f^s$ and $\f^t$ with the same range, the minimizer $\psi^{*}$ of $\text{\rm ConFUSE}_{\widehat{P}^{\psi}}^{st} (\f^{s},\f^t) + \lambda \text{\rm ConFUSE}_{\widehat{Q}_x^{\psi}}^{st} (\f^{s},\f^t)$ with the parameter $\lambda > 0$ in (\textcolor{red}{38}) zeroizes $\text{\rm MCSD}_{\widehat{Q}_x^{\psi}}^{(\rho)}(\f^s, \f^t) - \text{\rm MCSD}_{\widehat{P}_x^{\psi}}^{(\rho)}(\f^s, \f^t)$ in (\textcolor{red}{22}) of empirical \text{\rm MCSD} divergence defined on $\f^s$ and $\f^t$.
\end{proposition}
\begin{proof}
	The proof is trivial because the minimizer $\psi^{*}$ of $\text{\rm ConFUSE}_{\widehat{P}^{\psi}}^{st} (\f^{s},\f^t)$ results in $\f^s(\psi^*(\x)) = \f^t(\psi^*(\x))$ for each example $(\x,y) \sim \widehat{P}$, and therefore $\text{\rm MCSD}_{\widehat{P}_x^{\psi^*}}^{(\rho)}(\f^s, \f^t) = 0$; Similarly, the minimizer $\psi^{*}$ of $\text{\rm ConFUSE}_{\widehat{Q}_x^{\psi}}^{st} (\f^{s},\f^t)$ also results in $\text{\rm MCSD}_{\widehat{Q}_x^{\psi^*}}^{(\rho)}(\f^s, \f^t) = 0$ and thus $\text{\rm MCSD}_{\widehat{Q}_x^{\psi^*}}^{(\rho)}(\f^s, \f^t) - \text{\rm MCSD}_{\widehat{P}_x^{\psi^*}}^{(\rho)}(\f^s, \f^t) = 0$. \hfill{\QEDclosed}
\end{proof}%\rightline{\QEDclosed{}}

\begin{table*}[htb]
	\centering
	\caption{Accuracy (\%) of different instantiations of McDalNets on the Office-31 \cite{office_31} dataset for closed set UDA. Results are based on models adapted from a 50-layer ResNet.}
	\label{Tab:different_implementation_off31}
	\begin{tabular}{l|cccccc|c}
		\hline
		Methods                                       & A $\to$ W     & D $\to$ W  & W $\to$ D   & A $\to$ D    & D $\to$ A      & W $\to$ A     & Avg  \\
		\hline
		Source Only\cite{resnet}                      & 79.9$\pm$0.3  & 96.6$\pm$0.4   & 99.4$\pm$0.2        & 84.1$\pm$0.4  & 64.5$\pm$0.3   & 66.4$\pm$0.4   &  81.8 \\
		\hline
		McDalNets based on the following surrogates of $\widehat{\text{\rm MCSD}}$ (\textcolor{red}{14}) and $\widetilde{\text{\rm MCSD}}$ (\textcolor{red}{13})             &                     &               &                 &                &  & \\
		\quad DANN \cite{reverse_grad,dann} (\textcolor{red}{31})    & 82.2$\pm$0.2      & 98.2$\pm$0.2  &  99.8$\pm$0.2      & 84.1$\pm$0.3          & 66.3$\pm$0.4           & 66.4$\pm$0.2        & 82.8    \\
		% \quad MDD-form original (\ref{equ-mdd-form-ori}) & \multicolumn{4}{c|}{Not converge}                                                     & 66.1 (Not converge)  \\
		\quad MDD \cite{mdd} variant (\textcolor{red}{29}) & 86.5$\pm$1.2      & 98.2$\pm$0.3   &  99.8$\pm$0.2        & 87.3$\pm$0.5   & 67.9$\pm$0.3           & 67.7$\pm$0.1                        & 84.5  \\
		\hline
		McDalNets based on the following surrogates of MCSD (\textcolor{red}{7}) &                     &               &                 &             &   & \\
		\quad $L_1/\text{\rm MCD \cite{mcd}}$ (\textcolor{red}{24})  & 84.8$\pm$0.1    &98.2$\pm$0.3  & 99.8$\pm$0.2     & 86.8$\pm$0.3  & 69.8$\pm$0.1           & 68.6$\pm$0.4           & 84.7  \\
		\quad KL (\textcolor{red}{25})                                & 85.3$\pm$0.5    &98.5$\pm$0.1  & 99.8$\pm$0.2   & 86.2$\pm$0.3  & 69.6$\pm$0.6          & 68.3$\pm$0.1           & 84.6  \\
		\quad CE (\textcolor{red}{26})                     & 88.0$\pm$0.2    &98.5$\pm$0.2   & 100.0$\pm$.0   & 86.9$\pm$0.2       & 70.0$\pm$0.6  & 68.6$\pm$0.4  & 85.3   \\
		\hline
		SymmNets-V2 (\textcolor{red}{38})                  & \textbf{94.2}$\pm$0.1  & \textbf{98.8}$\pm$.0 & \textbf{100.0}$\pm$.0     & \textbf{93.5}$\pm$0.3        & \textbf{74.4}$\pm$0.1   & \textbf{73.4}$\pm$0.2                    & \textbf{89.1}   \\
		\hline
	\end{tabular}
\end{table*}

\begin{table*}[htb]
	\centering
	\caption{Accuracy (\%) of different instantiations of McDalNets on the ImageCLEF \cite{ImageCLEFDA} dataset for closed set UDA. Results are based on models adapted from a 50-layer ResNet.}
	\label{Tab:different_implementation_clef}
	\begin{tabular}{l|cccccc|c}
		\hline
		Methods                                       & I $\to$ P     & P $\to$ I  & I $\to$ C   & C $\to$ I    & C $\to$ P      & P $\to$ C     & Avg  \\
		\hline
		Source Only\cite{resnet}                      & 74.7$\pm$1.0          & 87.3$\pm$0.5       & 93.0$\pm$0.3        & 83.5$\pm$0.3         & 67.5$\pm$0.3           & 90.2$\pm$0.8          & 82.7 \\
		\hline
		McDalNets based on the following surrogates of $\widehat{\text{\rm MCSD}}$ (\textcolor{red}{14}) and $\widetilde{\text{\rm MCSD}}$ (\textcolor{red}{13})              &                     &               &                 &                &  & \\
		\quad DANN \cite{reverse_grad,dann} (\textcolor{red}{31})  & 77.3$\pm$0.1   & 90.7$\pm$0.3 &94.3$\pm$0.2    & 88.3$\pm$0.2         & 73.5$\pm$0.8           & 92.7$\pm$0.1          & 86.1    \\
		% \quad MDD-form original (\ref{equ-mdd-form-ori}) & \multicolumn{4}{c|}{Not converge}                                                     & 66.1 (Not converge)  \\
		\quad MDD \cite{mdd} variant (\textcolor{red}{29}) &77.2$\pm$0.3 & 91.8$\pm$0.2 &95.0$\pm$0.2    & 87.8$\pm$0.6         & 73.7$\pm$0.5           & 94.7$\pm$0.3          & 86.7  \\
		\hline
		McDalNets based on the following surrogates of MCSD (\textcolor{red}{7}) &                     &               &                 &             &   & \\
		\quad $L_1/\text{\rm MCD \cite{mcd}}$ (\textcolor{red}{24})  &77.8$\pm$0.2   &91.8$\pm$0.3 &94.8$\pm$0.1  &89.7$\pm$0.3 &75.2$\pm$0.5 & 93.2$\pm$0.4      & 87.0  \\
		\quad KL (\textcolor{red}{25})                                &77.7$\pm$0.2   &91.3$\pm$0.1 &95.3$\pm$0.2  &91.0$\pm$0.2 &76.0$\pm$0.3 & 94.2$\pm$0.2      & 87.6  \\
		\quad CE (\textcolor{red}{26})                     &78.2$\pm$0.1   &91.7$\pm$0.5 &95.8$\pm$0.4 &91.5$\pm$0.3  &75.3$\pm$0.1 & 94.5$\pm$0.2      & 87.8  \\
		\hline
		SymmNets-V2 (\textcolor{red}{38})                  &\textbf{79.0}$\pm$0.3 &\textbf{93.5}$\pm$0.2 &\textbf{96.9}$\pm$0.2 &\textbf{93.4}$\pm$0.3  &\textbf{79.2}$\pm$0.3 &\textbf{96.2}$\pm$0.1                 & \textbf{89.7}   \\
		\hline
	\end{tabular}
\end{table*}

\section{Experiments}  \label{appendix-Experiments}
\subsection{Datasets and Implementations}
\noindent \textbf{Office-31} The office-31 dataset \cite{office_31} is a standard benchmark dataset for domain adaptation, which contains $4,110$ images of $31$ categories shared by three distinct domains: \textit{Amazon} (\textbf{A}), \textit{Webcam} (\textbf{W}) and \textit{DSLR} (\textbf{D}). We adopt it in the closed set, partial, and open set UDA.

\noindent \textbf{ImageCLEF-DA} The ImageCLEF-DA dataset \cite{ImageCLEFDA} is a benchmark dataset for the ImageCLEF 2014 domain adaptation challenge, which contains three domains: \textit{Caltech-256} (\textbf{C}), \textit{ImageNet ILSVRC 2012} (\textbf{I}) and \textit{Pascal VOC 2012} (\textbf{P}). For each domain, there are $12$ categories and $50$ images in each class. The three domains in this dataset are of the same size, which is a good complementation of the Office-31 dataset where different domains are of different sizes. We adopt it in the closed set settng of UDA.

\noindent \textbf{Office-Home} The Office-Home dataset \cite{office_home} is a very challenging dataset for domain adaptation, which contains $15,500$ images from $65$ categories of everyday objects in the office and home scenes, shared by four significantly different domains: Artistic images (\textbf{A}), Clip Art (\textbf{C}), Product images (\textbf{P}) and Real-World images (\textbf{R}). We adopt it in the closed set and partial UDA.

\noindent \textbf{Syn2Real} The Syn2Real dataset \cite{syn2real, visda} is a challenging simulation-to-real dataset, which contains over $280$K images of $12$ categories. We adopt the training domain, which contains synthetic images generated by rendering 3D models from different angles and under different lighting conditions, and validation domain, which contains natural images, as the source domain and target domain, respectively. We adopt it in the closed set and open set UDA. For open set UDA, there are additional $33$ background categories and $69$ other categories aggregated as the unknown class of the source domain and target domain, respectively.

\noindent \textbf{Digits} The MNIST \cite{mnist}, SVHN \cite{svhn}, and USPS \cite{usps} datasets are adopted in the closed set UDA. Following \cite{adr}, we adopt the modified LeNet and evaluate on three adaptation tasks of SVHN $to$ MNIST, MNIST $\to$ USPS, and USPS $\to$ MNIST. Following \cite{adda}, we sample $2,000$ images from the MNIST and $1,800$ images from  the USPS for the adaptation between MNIST and USPS, and use the full training sets for the SVHN $\to$ MNIST task.

\noindent \textbf{DomainNet} The DomainNet dataset \cite{peng2018moment} is the largest UDA datasets to the best of our knowledge. There are $586.6$K samples of $345$ categories shared by six domains of Clipart (clp), Infograph (inf), Painting (pnt), Quichdraw (qdr), Real (rel), and Sketch (skt).  We adopt it in our experiments of closed set UDA.

\noindent \textbf{Modified LeNet Implementation} Following \cite{adr}, we adopt the modified LeNet for the Digits datasets \cite{mnist, svhn, usps}. All parameters are updated with the Adam optimizer with a learning rate of $0.0002$, a $\beta_1$ of $0.5$, a $\beta_2$ of $0.999$, and a batch size of $256$ images. We convert all training images to greyscale and scale them to $28\times 28$ pixels.

\subsection{Analysis}
\noindent \textbf{Full Results of Different Implementations of McDalNets (\textcolor{red}{22})} We present the full results of different implementations of McDalNets (\textcolor{red}{22}) on datasets of Office-31 \cite{office_31}, ImageCLEF-DA \cite{ImageCLEFDA}, Office-Home \cite{office_home}, VisDA-2017 \cite{visda}, Digits \cite{mnist,svhn,usps}, and DomainNet \cite{peng2018moment} in Table \ref{Tab:different_implementation_off31}, Table \ref{Tab:different_implementation_clef}, Table \ref{Tab:different_implementation_offhome}, Table \ref{Tab:different_implementation_visda}, Table \ref{Tab:different_implementation_Digits}, and Table \ref{Tab:different_implementation_DomainNet_full}, respectively.

\begin{table*}[htb]
	\centering
	\caption{Accuracy (\%) of different instantiations of McDalNets on the Office-Home \cite{office_home} dataset for closed set UDA. Results are based on models adapted from a 50-layer ResNet.}
	\label{Tab:different_implementation_offhome}
	\begin{tabular}{L{58.8mm}C{5.3mm}C{5.3mm}C{5.3mm}C{5.3mm}C{5.3mm}C{5.3mm}C{5.3mm}C{5.3mm}C{5.3mm}C{5.3mm}C{5.3mm}C{5.3mm}C{5.3mm}C{5.3mm}}
		\hline
		Methods                    &A$\to$C &A$\to$P &A$\to$R &C$\to$A &C$\to$P &C$\to$R &P$\to$A &P$\to$C &P$\to$R &R$\to$A &R$\to$C &R$\to$P & Avg  \\
		\hline
		Source Only\cite{resnet}                       &40.5 &66.1  &74.3 &53.2  &61.2   &63.9    &52.6     &37.5 &72.3&65.5 &43.2 & 77.0  &  58.9 \\
		\hline
		{\tiny McDalNets based on the following surrogates of $\widehat{\text{\rm MCSD}}$ (\textcolor{red}{14}) and $\widetilde{\text{\rm MCSD}}$ (\textcolor{red}{13}) }      &&&&&&&&&&&&        &               \\
		\quad DANN \cite{reverse_grad,dann} (\textcolor{red}{31})      &42.9 & 65.5& 74.3 & 54.5 & 60.6& 65.4& 54.0& 40.3& 73.1 & 66.7& 45.4& 76.9  & 60.0    \\
		\quad MDD \cite{mdd} variant (\textcolor{red}{29})   &33.2 & 64.2& 75.0 & 58.9 & 62.4& 68.3& 57.7& 43.0& 75.5 & 70.1& 46.0& 79.0  & 61.1  \\
		\hline
		{\tiny McDalNets based on the following surrogates of MCSD (\textcolor{red}{7})}       &&&&&&&&&&&&         & \\
		\quad $L_1/\text{\rm MCD \cite{mcd}}$ (\textcolor{red}{24})     & 45.4 & 67.2 & 75.2 & 58.3 & 62.9 & 68.2 & 56.7 & 42.8 & 73.9 & 67.5 & 47.9 & 78.0   & 62.0  \\
		\quad KL (\textcolor{red}{25})                                   & 46.6 & 69.2 & 75.2 & 59.9 & 65.1 & 68.2 & 60.2 & 45.6 & 73.8 & 67.3 & 50.4 & 77.7   & 63.3  \\
		\quad CE (\textcolor{red}{26})                        & 46.6 & 69.2 & 75.6 & 59.9 & 65.1 & 68.8 & 61.4 & 45.8 & 74.8 & 68.8 & 52.1 & 79.6   & 64.0   \\
		\hline
		SymmNets-V2 (\textcolor{red}{38})                &\textbf{48.1} & \textbf{74.3} & \textbf{78.7}& \textbf{64.6} & \textbf{71.8} & \textbf{74.1}& \textbf{64.4} & \textbf{50.0} & \textbf{80.2} & \textbf{74.3} & \textbf{53.1} & \textbf{83.2}  & \textbf{68.1}   \\
		\hline
	\end{tabular}
\end{table*}

\begin{table*}[htb]
	\centering
	\caption{Accuracy (\%) of different instantiations of McDalNets on the VisDA-2017 \cite{visda} dataset for closed set UDA. Results are based on models adapted from a 50-layer ResNet.}
	\label{Tab:different_implementation_visda}
	\begin{tabular}{l|cccccccccccc|c}
		\hline
		Methods                                   &    \rotatebox{45}{plane} & \rotatebox{45}{bcycle} &\rotatebox{45}{bus} & \rotatebox{45}{car} & \rotatebox{45}{horse} & \rotatebox{45}{knife} & \rotatebox{45}{mcycl} & \rotatebox{45}{person} & \rotatebox{45}{plant} & \rotatebox{45}{sktbrd} & \rotatebox{45}{train} & \rotatebox{45}{trunk} & \rotatebox{45}{Avg}  \\
		\hline
		Source Only\cite{resnet}                       &68.2 &10.9&35.3&\textbf{75.7}&53.6 &2.7 &74.1 &4.7  &61.8 &18.9  &90.5 & 4.3  &  41.8 \\
		\hline
		{\tiny McDalNets based on the following surrogates of $\widehat{\text{\rm MCSD}}$ (\textcolor{red}{14}) and $\widetilde{\text{\rm MCSD}}$ (\textcolor{red}{13}) }      &&&&&&&&&&&&        &               \\
		\quad DANN \cite{reverse_grad,dann} (\textcolor{red}{31})     &77.1 & 35.7 & 68.0 & 59.0 & 75.8 & 20.1 & 89.3 & 42.1 & 86.3 & 38.8 & 85.9 & 22.5  & 58.4    \\
		% \quad MDD-form original (\ref{equ-mdd-form-ori}) & \multicolumn{4}{c|}{Not converge}                                                     & 66.1 (Not converge)  \\
		% \quad MDD \cite{mdd} variant (\textcolor{red}{28})   &83.4 & 46.0 & \textbf{79.7} & \textbf{75.6} & 76.8 & 64.9 & 93.0 & 62.9 & 91.0 & 33.2 & 87.3 & 9.3   & $^*$67.0  \\
		\quad MDD \cite{mdd} variant (\textcolor{red}{29})   &\multicolumn{12}{c|}{did not converge}     \\
		\hline
		{\tiny McDalNets based on the following surrogates of MCSD (\textcolor{red}{7})}       &&&&&&&&&&&&         & \\
		\quad $L_1/\text{\rm MCD \cite{mcd}}$ (\textcolor{red}{24})     &84.8 & 60.0 & 75.6 & 75.5 & 82.5 & 76.5 & 93.0 & 73.1 & 92.8 & 28.2 & 90.9 & 10.4     & 70.4  \\
		\quad KL (\textcolor{red}{25})                                   &\textbf{89.3} & \textbf{62.9} & 70.6 & 70.4 & \textbf{83.5} & \textbf{83.1} & 92.5 & 68.9 & 91.5 & 6.6  & \textbf{91.0} & 18.3     & 69.0  \\
		\quad CE (\textcolor{red}{26})                        &86.5 & 56.7 & 78.0 & 72.9 & 80.8 & 81.3 & \textbf{93.7} & \textbf{76.5} & \textbf{94.1} & 20.0 & 87.6 & 16.7     & 70.5   \\
		\hline
		SymmNets-V2 (\textcolor{red}{38})                           &87.3 & 62.2 & \textbf{79.1} & 66.7 & 80.3 & 79.7 & 87.8 & 75.6 & 88.9 & \textbf{31.4} & 90.7 & \textbf{25.8}     & \textbf{71.3}   \\
		\hline
	\end{tabular}
\end{table*}

\begin{table*}[h]
	\centering
	\caption{Accuracy (\%) of different instantiations of McDalNets on the Digits \cite{mnist,svhn,usps} dataset for closed set UDA. Results are based on models adapted from a modified LeNet.}
	\label{Tab:different_implementation_Digits}
	\begin{tabular}{l|ccc|c}
		\hline
		Methods                                       & S $\to$ M     & U $\to$ M  & M $\to$ U              & Avg  \\
		\hline
		Source Only                                   & 62.7$\pm$1.1  & 77.5$\pm$2.2 & 71.2$\pm$0.7         & 70.5 \\
		\hline
		McDalNets based on the following  surrogates of $\widehat{\text{\rm MCSD}}$ (\textcolor{red}{14}) and $\widetilde{\text{\rm MCSD}}$ (\textcolor{red}{13})              &                     &               &     &           \\
		\quad DANN \cite{reverse_grad,dann} (\textcolor{red}{31})   & 74.2$\pm$1.0    & 73.0$\pm$2.9 & 70.3$\pm$1.5  & 72.5  \\
		% \quad MDD-form original (\ref{equ-mdd-form-ori}) & \multicolumn{4}{c|}{Not converge}                                                     & 66.1 (Not converge)  \\
		\quad MDD \cite{mdd} variant (\textcolor{red}{29})  &\multicolumn{3}{c|}{did not converge}  & did not converge  \\
		\hline
		McDalNets based on the following surrogates of MCSD (\textcolor{red}{7}) &               &&   & \\
		\quad $L_1/\text{\rm MCD \cite{mcd}}$ (\textcolor{red}{24})  & 90.4$\pm$0.4      &95.8$\pm$0.6 & 85.7$\pm$1.9 & 90.6 \\
		\quad KL (\textcolor{red}{25})                                & 76.6$\pm$1.3      &94.5$\pm$0.7 & 77.5$\pm$0.6 & 82.9  \\
		\quad CE (\textcolor{red}{26})                     & \textbf{97.8}$\pm$0.2      &96.6$\pm$0.6 & 90.3$\pm$0.8 & 94.9 \\
		\hline
		SymmNets-V2 (\textcolor{red}{38})                       & 96.3$\pm$1.2      &\textbf{96.8}$\pm$0.3 & \textbf{94.8}$\pm$0.6 & \textbf{96.0} \\
		\hline
	\end{tabular}
\end{table*}

\begin{table*}
	\centering
	\caption{Accuracy (\%) of different instantiations of McDalNets on the DomainNet \cite{peng2018moment} dataset for closed set UDA. Results are based on models adapted from a 50-layer ResNet. In each sub-table, the column-wise domains are selected as the source domain and the row-wise domains are selected as the target domain. The `SO' and 'Sym2' indicate the baseline of Source Only and our proposed SymmNets-V2, respectively. The `DANN' and `MDD$^*$' are the McDalNets based on the scalar-valued $\widehat{\text{\rm MCSD}}$ surrogate  (\textcolor{red}{31}) and $\widetilde{\text{\rm MCSD}}$ surrogate (\textcolor{red}{29}), respectively. The `$L_1$', `KL' and `CE' are the McDalNets based on the MCSD surrogates of  $L_1/\text{\rm MCD \cite{mcd}}$ (\textcolor{red}{24}), KL (\textcolor{red}{25}), and CE (\textcolor{red}{26}), respectively. In the `Oracle' setting, we fine-tune on labeled target data the ResNet-50 model that is pre-trained on the ImageNet dataset.}
	\label{Tab:different_implementation_DomainNet_full}
	\scriptsize %\setlength{\tabcolsep}{6.0pt}
	\begin{tabular}{L{2.5mm}|C{0.8mm}C{0.8mm}C{0.8mm}C{0.8mm}C{0.8mm}C{0.8mm}C{2.9mm}| |L{4.9mm}|C{0.8mm}C{0.8mm}C{0.8mm}C{0.8mm}C{0.8mm}C{0.8mm}C{2.9mm}||L{4.7mm}|C{0.8mm}C{0.8mm}C{0.8mm}C{0.8mm}C{0.8mm}C{0.8mm}C{3.9mm}||C{4.5mm}|C{0.8mm}C{0.8mm}C{0.8mm}C{0.8mm}C{0.8mm}C{0.8mm}C{3.9mm}}
		\hline
		SO & clp & inf & pnt & qdr & rel & skt & Avg.  & DANN & clp & inf & pnt & qdr & rel & skt & Avg. &MDD$^*$ & clp & inf & pnt & qdr & rel & skt & Avg. & $L_1$ & clp & inf & pnt & qdr & rel & skt & Avg.  \\
		\hline
		clp &$-$&17.3 & 28.9 & 8.9 & 50.4 & 38.9 & 29.2& clp& $-$ & 17.9& 32.1& 10.5 & 54.1 & 40.7 & 31.1& clp & $-$ & 18.1 & 33.1 & 11.7 & 54.3 & 40.5 & 31.5 & clp & $-$ & 18.3 & 32.5 & 11.5 & 54.1 & 41.0 & 31.5  \\
		inf & 33.9 & $-$ & 28.2 & 2.7 & 50.2 & 27.6 & 28.5 & inf & 33.7 & $-$ & 30.3 & 3.4 & 50.3 & 28.5 & 29.2 & inf & 24.3 & $-$ & 27.1 & 2.8 & 50.3 & 24.9 & 25.9 & inf & 34.5 & $-$ & 30.6 & 3.3 & 52.1 & 29.2 & 29.9  \\
		pnt & 31.8 & 14.3 & $-$ & 2.8 & 50.7 & 29.3 & 25.8 & pnt & 34.9 & 14.8 & $-$ & 4.1 & 51.2 & 32.8 & 27.6 & pnt & 32.0 & 14.3 & $-$ & 4.6 & 51.6 & 31.8 & 26.9 & pnt & 34.8 & 15.1 & $-$ & 4.5 & 51.8 & 32.6 & 27.8  \\
		qdr & 7.5 & 1.2 & 1.6 & $-$ & 5.4 & 7.6 & 4.7 & qdr & 19.6 & 2.6 & 7.7 & $-$ & 14.6 & 9.6 & 10.8 & qdr & 17.8 & 3.6 & 7.9 & $-$ & 16.7 & 13.3 & 11.9 & qdr & 19.7 & 3.3 & 8.3 & $-$ & 15.6 & 14.0 & 12.2 \\
		rel & 43.2 & 20.6 & 41.5 & 4.9 & $-$ & 33.3 & 28.7 & rel & 45.4 & 20.8 & 43.2 & 5.7 & $-$ & 35.5 & 30.1 & rel & 44.2 & 20.3 & 42.4 & 6.3 & $-$ & 36.1 & 29.8 & rel & 45.1 & 20.8 & 43.0 & 6.6 & $-$ & 36.0 & 30.3 \\
		skt & 45.8 & 14.8 & 30.1 & 11.1 & 46.5 & $-$ & 29.7 & skt & 49.8 & 18.1 & 37.3 & 11.3 & 52.3 & $-$ & 33.8 & skt & 45.2 & 17.5 & 36.3 & 12.0 & 52.4 & $-$ & 32.7 & skt &49.8 & 17.9 & 37.8 & 12.4 & 53.7 & $-$ & 34.3 \\
		Avg. & 32.4 & 13.6 & 26.1 & 6.3 & 40.8 & 27.3 & 24.4 & Avg. & 36.7 & 14.8 & 30.1 & 7.0 & 44.5 & 29.4 & 27.1  & Avg. & 32.7 & 14.8 & 29.4 & 7.5 & 45.1 & 29.3 & 26.5   & Avg. & 36.8 & 15.1 & 30.4 & 7.7 & 45.5 & 30.6 & 27.7\\
		\hline
		
		KL & clp & inf & pnt & qdr & rel & skt & Avg.  & CE & clp & inf & pnt & qdr & rel & skt & Avg. &Sym2 & clp & inf & pnt & qdr & rel & skt & Avg. &Oracle & clp & inf & pnt & qdr & rel & skt & Avg.  \\
		\hline
		clp &$-$ & 18.4 & 32.5 & 11.3 & 54.1 & 40.9 & 31.4 & clp & $-$ &18.5 & 32.6 & 11.7 & 54.4 & 41.0 & 31.6 & clp & $-$& 18.3& 33.9& 11.5& 55.4& 42.6 & 32.3& clp& 74.3 & $-$ & $-$ & $-$ & $-$ & $-$ & 74.3 \\
		inf & 34.0 & $-$& 30.5 & 3.7 & 52.0 & 29.4 & 29.9 & inf & 34.4 & $-$ & 30.7 & 3.5 & 52.3 & 29.4 & 30.1 & inf & 30.7 & $-$ & 29.0 & 3.3 & 49.0 & 27.8 & 28.0 & inf &$-$ & 40.8 &$-$ & $-$ & $-$ & $-$ & 40.8 \\
		pnt & 34.9 & 15.2 & $-$ & 4.6 & 51.7 & 32.7 & 27.8 & pnt & 34.7 & 15.2 & $-$ & 4.9 & 51.9 & 32.8 & 27.9 & pnt & 33.3 & 14.9 & $-$ & 4.4 & 50.0 & 33.6 & 27.2 & pnt & $-$ & $-$ & 69.7 & $-$ & $-$ & $-$ & 69.7 \\
		qdr & 19.8 & 3.5 & 7.7 & $-$ & 16.3 & 13.4 & 12.1 & qdr & 20.7 & 3.4 & 8.3 & $-$ & 16.9 & 14.6 & 12.8 & qdr & 22.8 & 3.2 & 7.9 & $-$ & 16.8 & 11.8 & 12.5 & qdr & $-$ & $-$ & $-$ & 70.6 & $-$ & $-$  & 70.6 \\
		rel & 45.2 & 21.0 & 42.9 & 6.6 & $-$ & 35.9 & 30.3 & rel & 45.2 & 21.3 & 43.0 & 7.0 & $-$ & 36.3 & 30.6 & rel & 48.4 & 19.5 & 44.0 & 5.3 & $-$ & 38.2 & 31.1 & real & $-$ & $-$ & $-$ & $-$ & 82.5 & $-$ & 82.5 \\
		skt & 50.1 & 18.0 & 37.5 & 12.3 & 53.5 & $-$ & 34.3 & skt & 50.0 & 18.1 & 38.0 & 12.9 & 53.3 & $-$ & 34.5 &skt & 55.2 & 18.2 & 39.5 & 12.4 & 55.2 & $-$ & 36.1 & skt & $-$ & $-$ & $-$ & $-$ & $-$ & 66.8 &  66.8  \\
		Avg. & 36.8 & 15.2 & 30.2 & 7.7 & 45.5 & 30.5 & 27.6 & Avg. & 37.0 & 15.3 & 30.5 & 8.0 & 45.8 & 30.8 & \textbf{27.9} & Avg. & 38.1 & 14.8 & 30.9 & 7.4 & 45.3 & 30.8 & \textbf{27.9} &Avg. & 74.3 & 40.8 & 69.7 & 70.6 & 82.5 & 66.8 & 67.5  \\
		\hline
	\end{tabular}
\end{table*}

\noindent \textbf{Visualization with the Class Information} We visualize the network activations from the feature extractor of ``DANN'' and ``SymmNets-V2'' on the adaptation task of $\textbf{A}\to\textbf{W}$ by t-SNE \cite{sne} with class information in Figure \ref{Fig:sne_cate}. The samples of the same class across domains are aligned intuitively with the features of SymmNets-V2.

\begin{figure*}
	\begin{minipage}[t]{0.13\linewidth}
		\vspace*{-2cm}
		DANN \cite{reverse_grad,dann}	
	\end{minipage}
	\hfill
	\subfigure[Close Set UDA]{
		\begin{minipage}[t]{0.28\linewidth}
			\centering
			\includegraphics[width=0.77\linewidth] {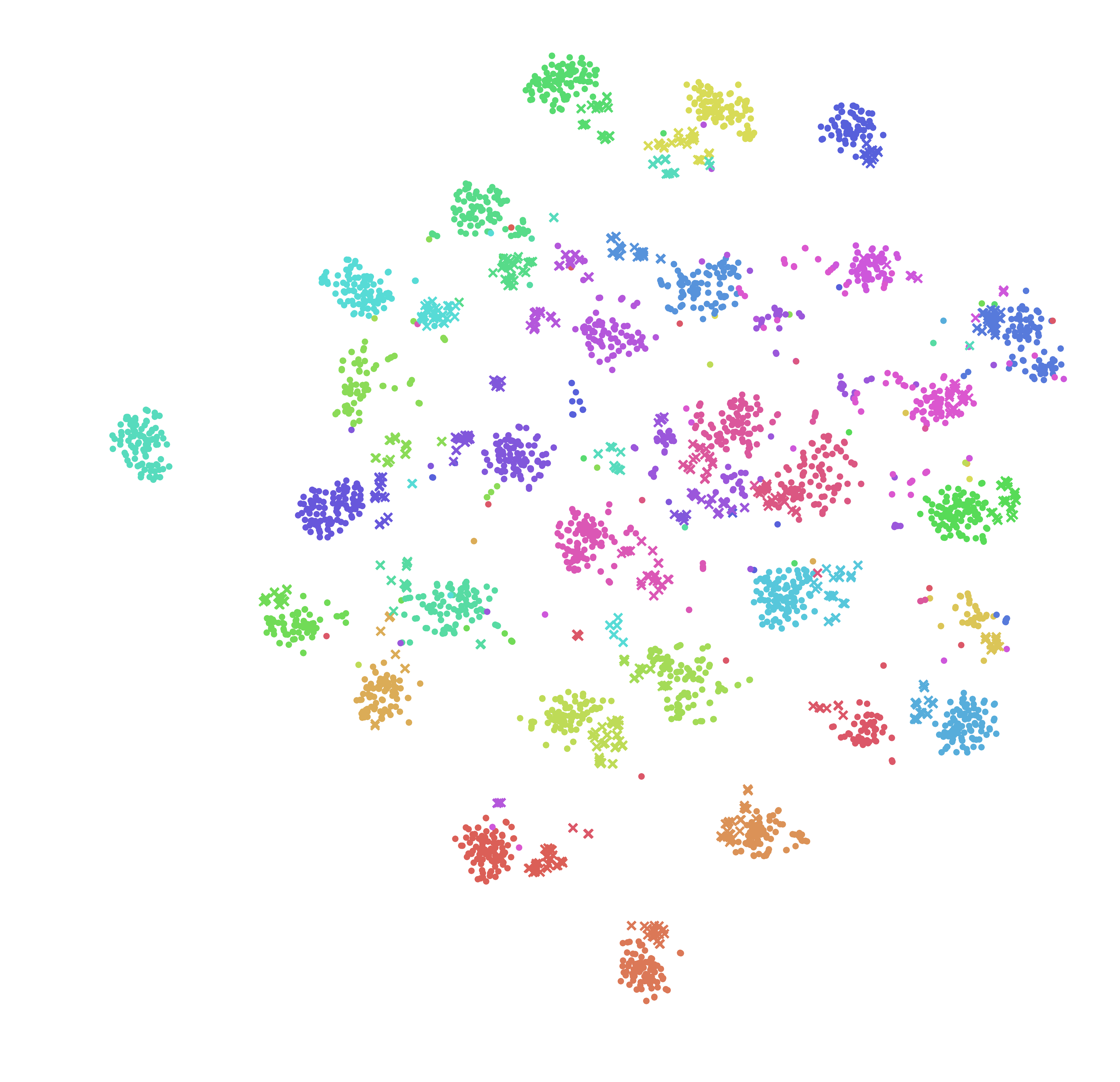}
	\end{minipage}}
	\hfill
	\subfigure[Partial UDA]{
		\begin{minipage}[t]{0.28\linewidth}
			\centering
			\includegraphics[width=0.77\linewidth] {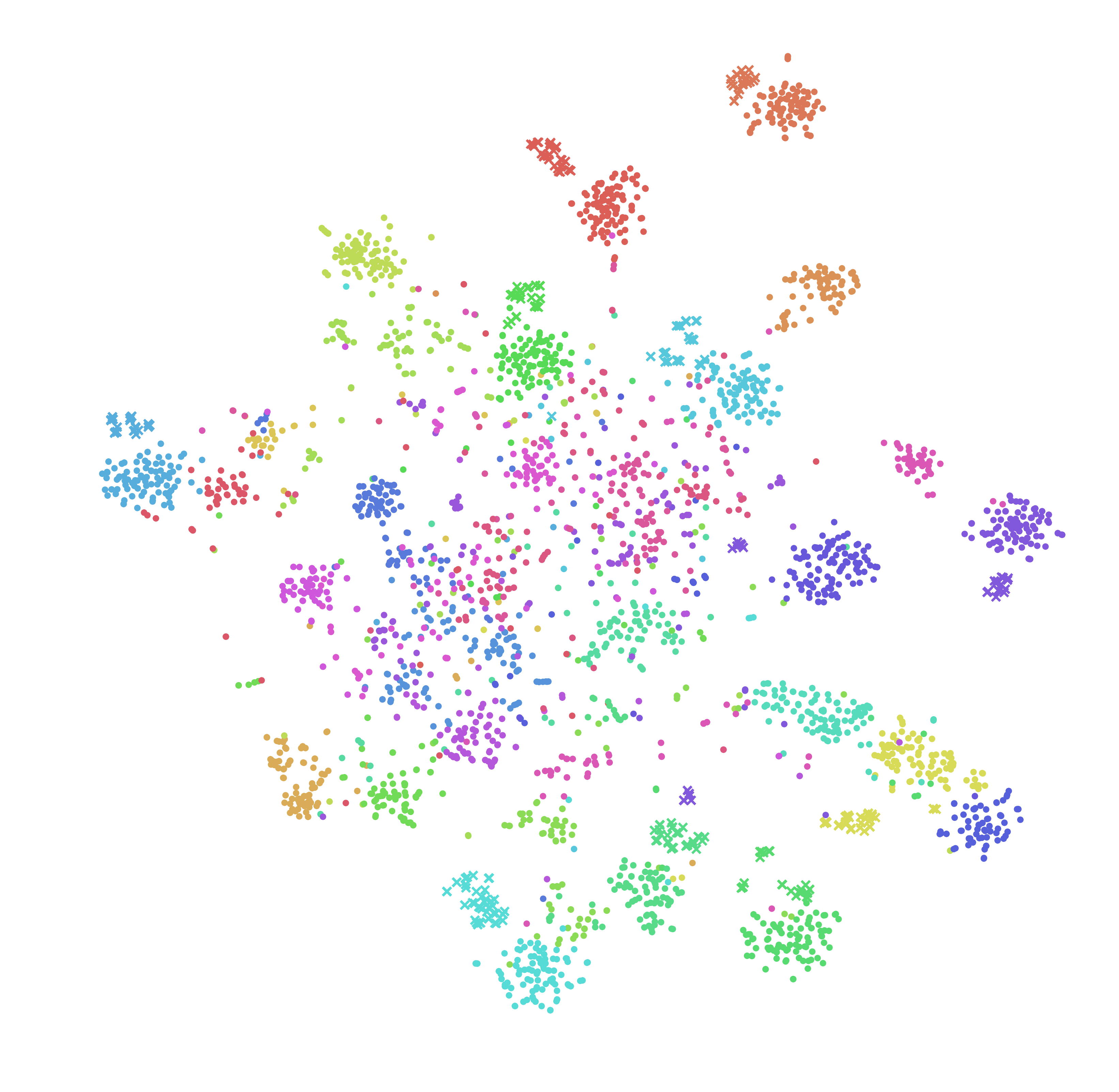}
	\end{minipage}}
	\hfill
	\subfigure[Open Set UDA]{
		\begin{minipage}[t]{0.28\linewidth}
			\centering
			\includegraphics[width=0.77\linewidth] {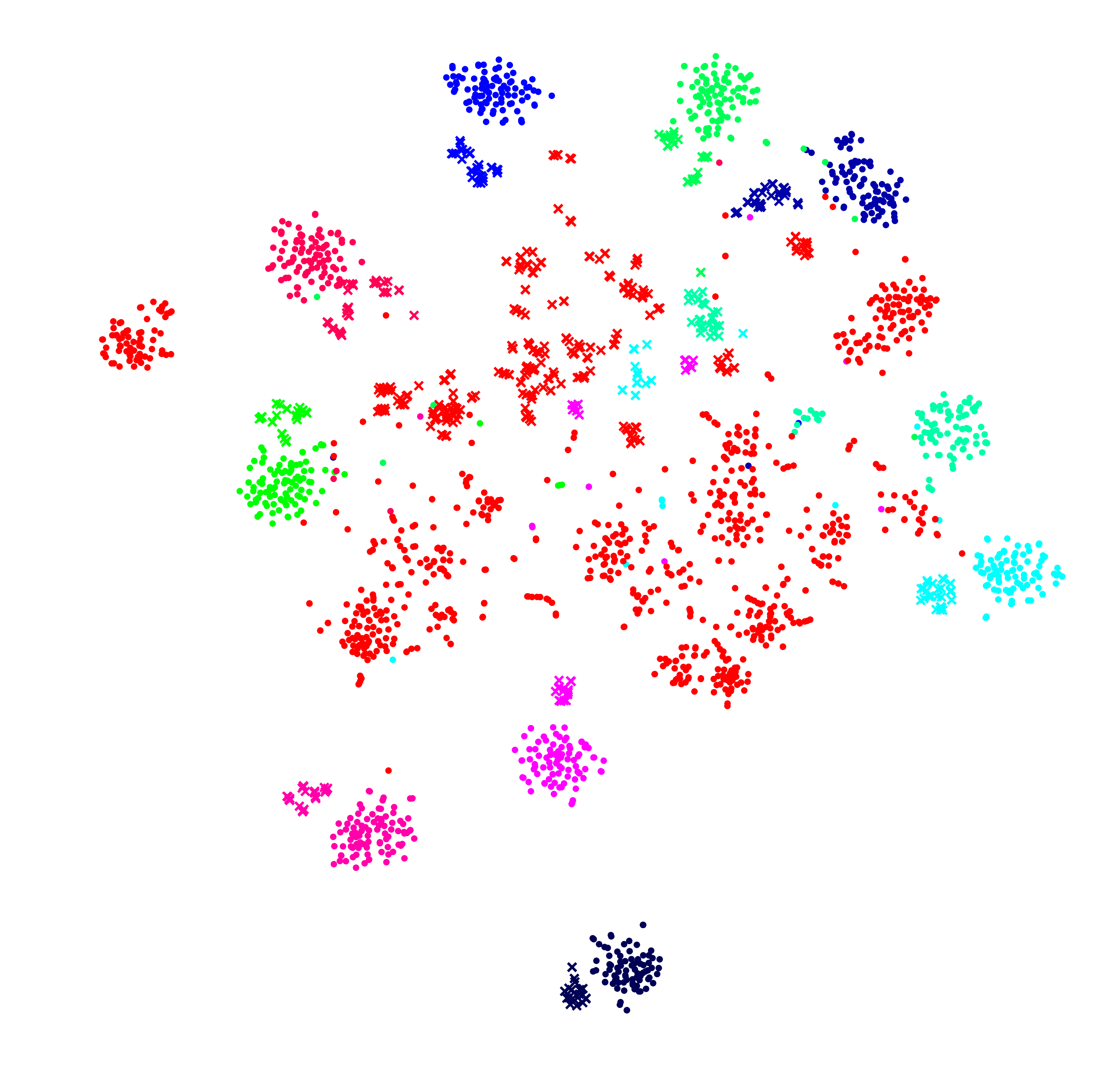}
	\end{minipage}}
	\\
	\begin{minipage}[t]{0.13\linewidth}
		\vspace*{-2cm}
		SymmNets-V2	
	\end{minipage}
	\hfill
	\subfigure[Close Set UDA]{
		\begin{minipage}[t]{0.28\linewidth}
			\centering
			\includegraphics[width=0.77\linewidth] {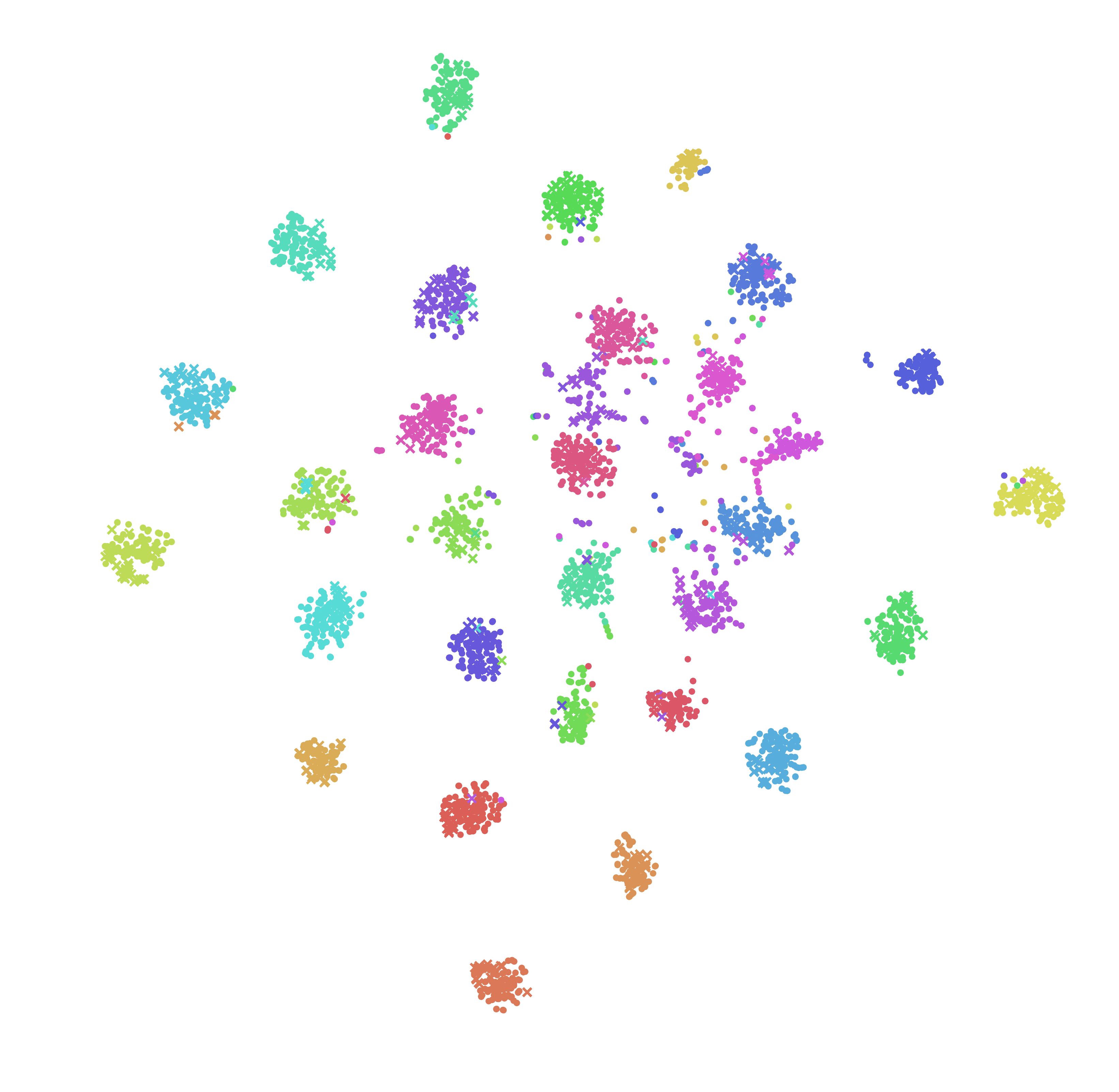}
	\end{minipage}}
	\hfill
	\subfigure[Partial UDA]{
		\begin{minipage}[t]{0.28\linewidth}
			\centering
			\includegraphics[width=0.77\linewidth] {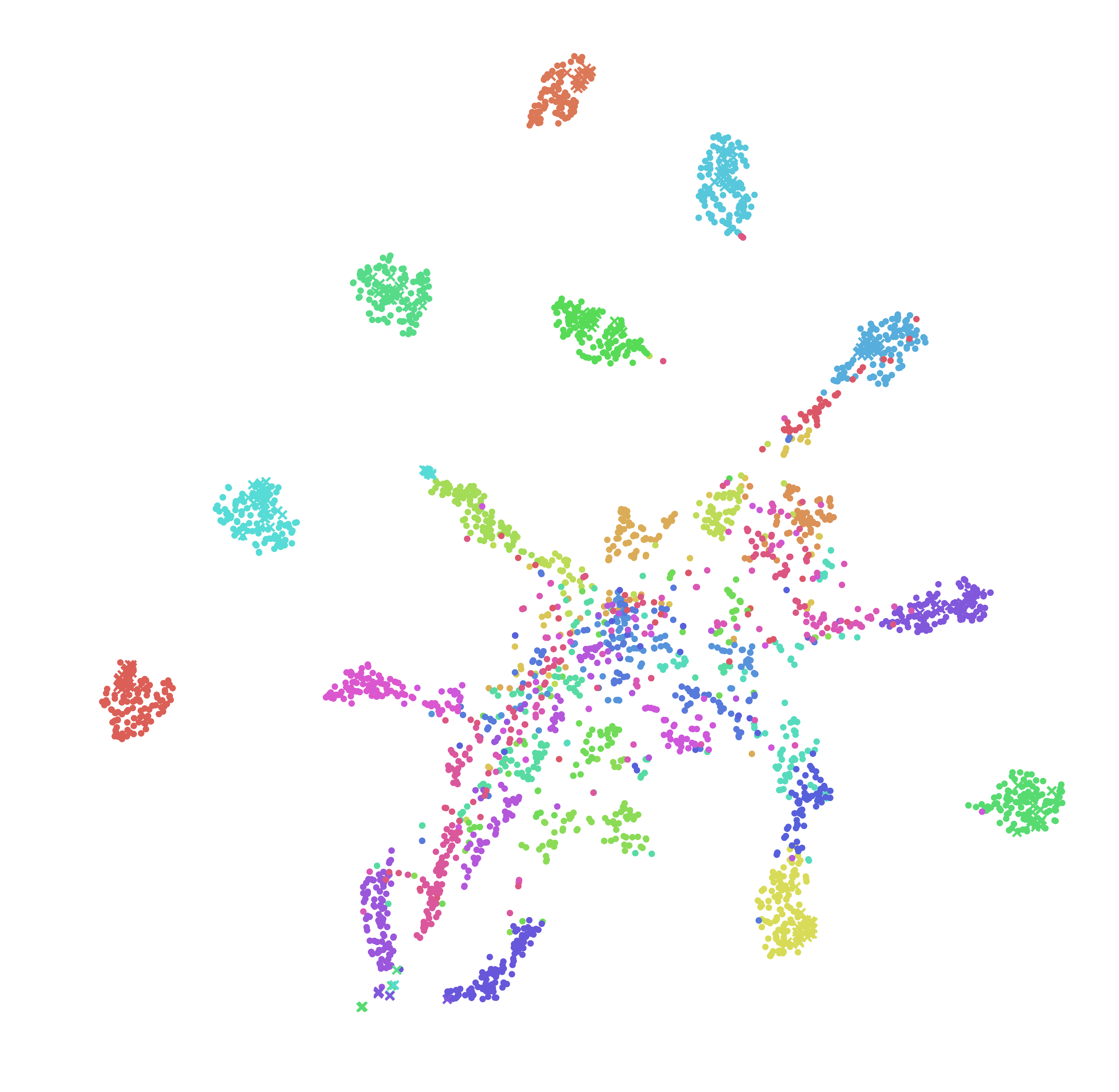}
	\end{minipage}}
	\hfill
	\subfigure[Open Set UDA]{
		\begin{minipage}[t]{0.28\linewidth}
			\centering
			\includegraphics[width=0.77\linewidth] {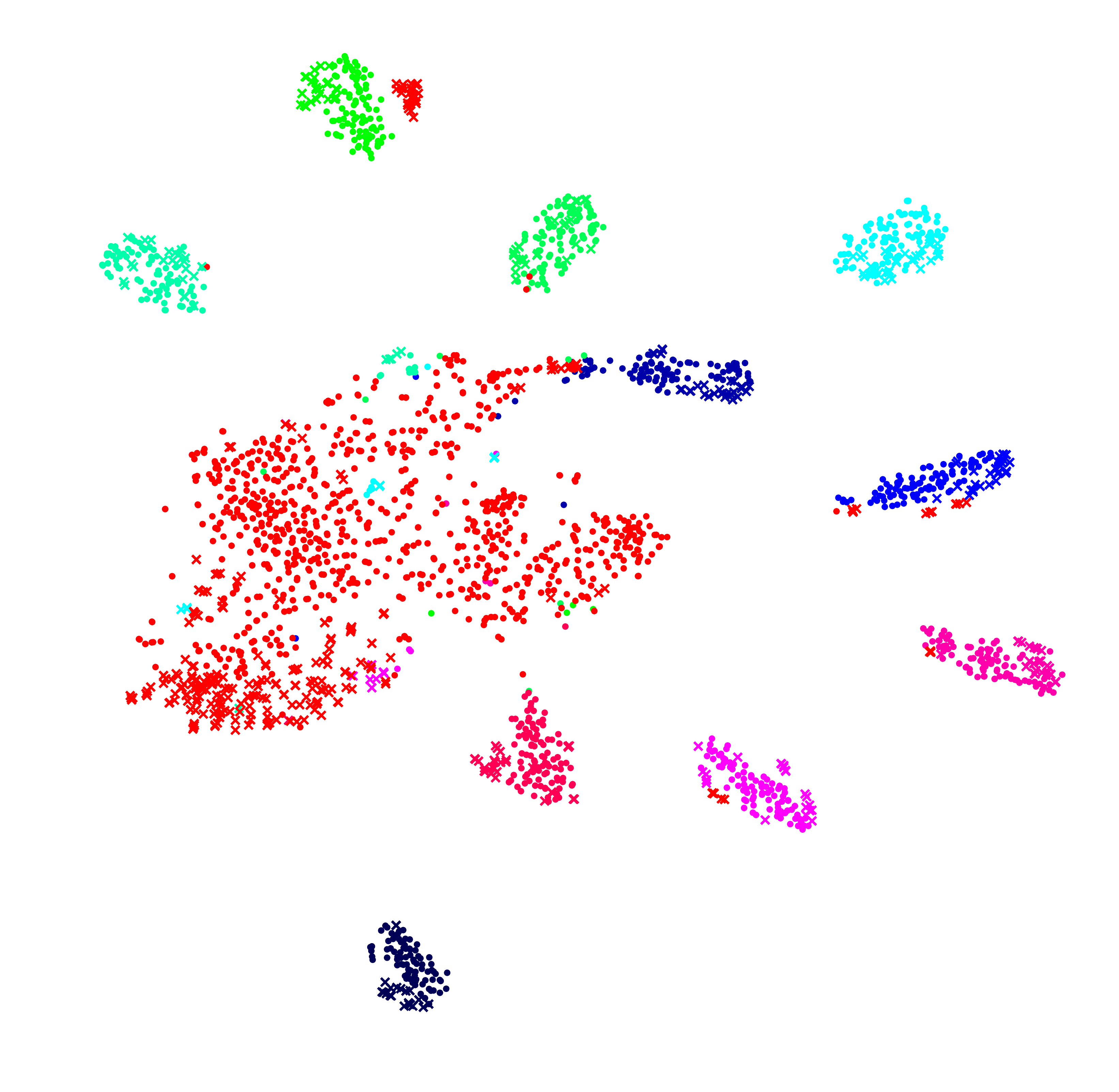}
	\end{minipage}}
	
	\caption{The t-SNE visualization of class-labeled feature representations learned by DANN (top row) and SymmNets-V2 (bottom row) under the settings of closed set, partial, and open set UDA. The point marks (``$\cdot$'') represent features of samples from the source domain $\textbf{A}$ whereas the cross marks (``x'') represent features of samples from the target domain $\textbf{W}$, where different colors represent different classes. In open set UDA, the red color indicates the unknown class. In partial UDA, we illustrate the feature representations learned by SymmNets-V2 (With active $\omega_k$), where we focus on the domain-shared classes and leave the source classes exclusive to the target domain as an indistinguishable cluster via the soft class weighting scheme, as discussed in Section \textcolor{red}{4}. }
	\label{Fig:sne_cate}
\end{figure*}

\subsection{Results}

\noindent \textbf{Results Based on the AlexNet Structure} To illustrate the generalization of our SymmNets-V2 to different network structures, we additionally implement SymmNets based on the AlexNet \cite{alexnet}. Given an AlexNet \cite{alexnet} pre-trained on the ImageNet dataset \cite{imagenet}, the feature extractor $\psi$ is the AlexNet without $fc8$ layer, and an additional bottleneck layer is added to the $fc7$ layer with a dimension of $256$ following \cite{dann}. Other settings are the same as that for the ResNet. Results for the closed set and partial UDA tasks are respectively presented in Table \ref{Tab:Office-Home_alexnet} and Table \ref{Tab:partial_office31_alexnet}, certifying the effectiveness and generalization of SymmNets-V2 on various model structures.

\begin{table*}[h!]
	\begin{center}
		
		\caption{Accuracy (\%) on the Office-Home dataset \cite{office_home} for \textit{closed set} UDA. Results are based on models adapted from a AlexNet.}
		\label{Tab:Office-Home_alexnet}
		\begin{tabular}{L{26.2mm}C{7.7mm}C{7.7mm}C{7.9mm}C{7.7mm}C{7.7mm}C{7.9mm}C{7.7mm}C{7.7mm}C{7.7mm}C{7.9mm}C{7.9mm}C{7.9mm}C{7.9mm}C{9mm}}
			\hline
			Methods                  &A$\to$C &A$\to$P &A$\to$R &C$\to$A &C$\to$P &C$\to$R &P$\to$A &P$\to$C &P$\to$R &R$\to$A &R$\to$C &R$\to$P    & Avg  \\
			\hline
			Source Only\cite{alexnet}    & 26.4     & 32.6     &41.3      & 22.1     & 41.7     & 42.1     & 20.5     & 20.3     & 51.1     & 31.0     & 27.9     & 54.9        & 34.3 \\
			DAN \cite{dan}           & 31.7     & 43.2     &55.1      & 33.8     & 48.6     & 50.8     & 30.1     & 35.1     & 57.7     & 44.6     & 39.3     & 63.7        & 44.5 \\
			DANN \cite{reverse_grad,dann}&36.4    & 45.2     &54.7      & 35.2     & 51.8     & 55.1     & 31.6     & 39.7     & 59.3     & 45.7     & 46.4     & 65.9        & 47.3 \\
			CDAN+E \cite{cada}       & \textbf{38.1} & 50.3     &60.3      & 39.7     & 56.4     & 57.8     & 35.5     & \textbf{43.1}     & 63.2     & 48.4     & \textbf{48.5}      & 71.1        & 51.0 \\
			\hline
			\textbf{SymmNets-V1 \cite{symnets}}           & 37.4     & \textbf{53.9} &60.9         &\textbf{40.0}& 56.3     &\textbf{58.5}& 34.7     & 40.1  & 64.0     & \textbf{49.6} & 46.7     & \textbf{71.6}        & \textbf{51.1} \\
			\textbf{SymmNets-V2}                          & 36.5     & 53.8          &\textbf{61.2}&\textbf{40.0}& \textbf{57.0} &58.1 & \textbf{36.2}     & 39.8  & \textbf{64.2}     & 48.8 & 46.1     & 71.2           & \textbf{51.1} \\
			\hline
			\hline
			GCAN \cite{gcan}         & 36.4     & 47.3     & 61.1     & 37.9     & 58.3     & 57.0     & 35.8     & \textbf{42.7}     & 64.5     & \textbf{50.1}     & \textbf{49.1}     & \textbf{72.5}        & 51.1 \\
			\hline
			\textbf{SymmNets-V2-SC}       & \textbf{38.6}     & \textbf{61.4} &\textbf{65.8}&\textbf{41.2}& \textbf{59.6}     &\textbf{63.4}& \textbf{37.7}     & 39.4  & \textbf{66.4}     & 49.2 & 47.1     & 71.4        & \textbf{53.4} \\
			\hline
		\end{tabular}
	\end{center}
\end{table*}

\begin{table*}[htb]
	\centering
	\caption{Accuracy (\%) on the Office-31 dataset \cite{office_31} for \textit{partial} UDA. Results are based on models adapted from a AlexNet.}
	\label{Tab:partial_office31_alexnet}
	\begin{tabular}{lccccccc}
		\hline
		Methods                          & A $\to$ W     & D $\to$ W     & W $\to$ D         & A $\to$ D     & D $\to$ A      & W $\to$ A     & Avg  \\
		\hline
		Source Only\cite{alexnet}            & 58.51         & 95.05         & 98.08             & 71.23         & 70.60          & 67.74         & 76.87 \\
		DAN \cite{dan}                   & 56.58         & 71.86         & 86.78             & 51.86         & 50.42          & 52.29         & 61.62 \\
		DANN\cite{reverse_grad,dann}       & 49.49         & 93.55         & 90.44             & 49.68         & 46.72          & 48.81         & 63.11 \\
		SAN\cite{san}                    & 80.02         & 98.64         & \textbf{100.00}      & 81.28         & 80.58          & 83.09         & 87.27 \\
		Zhang \emph{et al.} \cite{importance_weight} &76.27 &\textbf{98.98}& \textbf{100.00}    & 78.98         & \textbf{89.46} & 81.73         & 87.57 \\
		\hline
		%SymNets* ($\gamma= \textbf{1}$)   & 81.01         & 98.64         & \textbf{100.00}      & 85.56         & 80.31          & 84.94         & 88.41 \\
		%\textbf{\textcolor{red}{SymNets*}} \cite{symnets}                         & \textbf{83.05}& 97.85         & \textbf{100.00}      & \textbf{85.99}& 81.63          & \textbf{90.57}& \textbf{89.85} \\
		\textbf{SymmNets-V2}              & 76.62              & 79.30         & 99.37   & 82.83    & 71.33          & 83.19            & 82.11 \\
		\textbf{SymmNets-V2} (With active $\omega_k$)  & \textbf{82.71} & 94.90         & 98.72   & \textbf{85.35} & 83.50          & \textbf{93.00}& \textbf{89.70} \\
		\hline
	\end{tabular}
\end{table*}

\end{document}